\documentclass[twoside]{article}

\usepackage[accepted]{aistats2024}
\usepackage[round]{natbib}

\usepackage{algorithm}
\usepackage{algpseudocode}
\usepackage{amsmath}%

\usepackage{caption}
\usepackage{mathtools}
\usepackage{xfrac}
\usepackage{enumitem}
\usepackage{bbm}

\usepackage{comment}

\usepackage{hyperref}       %
\usepackage[nameinlink]{cleveref}
\usepackage{url}            %
\usepackage{booktabs}       %
\usepackage{amsfonts}       %
\usepackage{nicefrac}       %
\usepackage{microtype}      %
\usepackage{xcolor}         %

\usepackage{amsthm}
\usepackage{thmtools,thm-restate}

\newtheorem{assumption}{Assumption}
\newtheorem{lem}{Lemma}
\newtheorem{cor}{Corollary}

\newtheorem{oracle}{Oracle Assumption}

\crefname{lem}{Lemma}{Lemmas}
\crefname{assumption}{Assumption}{Assumptions}
\crefname{definition}{Definition}{Definitions}
\crefname{oracle}{Oracle Assumption}{Oracle Assumptions}
\crefname{cor}{Corollary}{Corollaries}
\crefname{thm}{Theorem}{Theorems}

\newcommand{\allPi}{\Tilde{\Pi}}
\newcommand{\allF}{\Tilde{\F}}
\newcommand{\alld}{\tilde{d}}

\newcommand{\hatR}{\hat{R}}
\newcommand{\Reg}{\text{Reg}}
\newcommand{\CReg}{\text{CReg}}
\newcommand{\hatReg}{\widehat{\text{Reg}}}
\newcommand{\hatf}{\hat{f}}
\newcommand{\hata}{\hat{a}}

\newcommand{\F}{\mathcal{F}}

\newcommand{\E}{\mathop{\mathbb{E}}}

\newcommand{\A}{\mathcal{A}}

\newcommand{\Xscript}{\mathcal{X}}
\newcommand{\ordO}{\mathcal{O}}
\newcommand{\ordOt}{\tilde{\mathcal{O}}}

\newcommand{\msafe}{m^*}
\newcommand{\msafealg}{\hat{m}}

\newcommand{\N}{\mathbb{N}}

\newcommand{\CVBandit}{\text{Mod-IGW}}
\newcommand{\ho}{\textbf{ho}}

\newcommand{\tr}{\textbf{tr}}

\newcommand{\MSOracle}{\text{EstOracle}}
\newcommand{\MisspecificationOracle}{\text{MTOracle}}
\newcommand{\EvalOracle}{\text{EvalOracle}}

\newcommand{\eventReg}{\mathcal{W}_1}
\newcommand{\eventMT}{\mathcal{W}_2}

\newcommand{\eventExpEval}{\mathcal{W}_2}

\usepackage{lipsum}

\newcommand\blfootnote[1]{%
	\begingroup
	\renewcommand\thefootnote{}\footnote{#1}%
	\addtocounter{footnote}{-1}%
	\endgroup
}

\allowdisplaybreaks

\begin{document}

\twocolumn[

\aistatstitle{Towards Costless Model Selection in Contextual Bandits: A Bias-Variance Perspective}

\aistatsauthor{ Sanath Kumar Krishnamurthy* \And Adrienne Margaret Propp* \And  Susan Athey }

\aistatsaddress{ Stanford University } ]

\begin{abstract}
Model selection in supervised learning provides costless guarantees as if the model that best balances bias and variance was known a priori. We study the feasibility of similar guarantees for cumulative regret minimization in the stochastic contextual bandit setting. Recent work \citep{marinov2021pareto} identifies instances where no algorithm can guarantee costless regret bounds. Nevertheless, we identify benign conditions where costless model selection is feasible: gradually increasing class complexity, and diminishing marginal returns for best-in-class policy value with increasing class complexity. Our algorithm is based on a novel misspecification test, and our analysis demonstrates the benefits of using model selection for reward estimation. Unlike prior work on model selection in contextual bandits, our algorithm carefully adapts to the evolving bias-variance trade-off as more data is collected. In particular, our algorithm and analysis go beyond adapting to the complexity of the simplest realizable class and instead adapt to the complexity of the simplest class whose estimation variance dominates the bias. For short horizons, this provides improved regret guarantees that depend on the complexity of simpler classes.
\end{abstract}

\section{INTRODUCTION}\label{sec:intro}

Contextual bandit algorithms are a fundamental tool for sequential decision making and have been the focus of an increasing amount of research in recent decades \citep{lattimore2020bandit}. These algorithms have been used in a wide range of applications from recommendation systems \citep{agarwal2016making} to mobile health \citep{tewari2017ads}.

We study the finite-armed, stochastic contextual bandit setting. In each round, the learner observes a feature vector, or context, drawn from a fixed distribution. The learner then selects an action and receives a reward that is a function of both the context and action. The data collected in each round is incorporated into the decision-making framework for the next round, with the goal of minimizing cumulative regret, i.e. maximizing the rewards received during the experiment.

A common approach to contextual bandits, which we call the regression-based approach, hinges on estimating the reward model. In each round, the data collected over prior rounds is used to estimate the true conditional expected reward for any context and action. When the next context is observed, the estimated reward is used to construct an action selection rule to balance two objectives: reduce uncertainty in the estimate for future rounds (exploration), and maximize the reward received in the current round (exploitation). This approach has led to the development of several contextual bandit algorithms \citep[e.g.][]{agrawal2013thompson, li2010contextual, foster2020beyond}. In general, the analyst specifies a model class $\F$ for the true reward model, and as data is gathered, the algorithm updates its selection from the class. When we assume realizability -- that is, that the true reward model lies in $\F$ -- these algorithms ensure optimal minimax guarantees on regret. \blfootnote{*Denotes significant/equal contributors.}

However, these algorithms do not specify how the model class $\F$ should be chosen, motivating recent work on model selection in contextual bandits \citep{agarwal2017corralling, foster2019model}. In a COLT 2020 open problem, \cite{foster2020open} pose a key question: given a set of $M$ nested model classes $\F_1,\F_2,\dots,\F_M$, such that at least one of these classes is realizable, can a contextual bandit algorithm achieve the best regret guarantees ensured by regression-based algorithms for a single model class? We refer to this as a costless model selection guarantee.\footnote{For brevity, we use the term costless guarantees to also capture near-costless guarantees, where we ignore terms logarithmic in number of model classes, model complexities, number of rounds, and confidence parameters.}

In this work, we present $\CVBandit$, an algorithm that achieves costless model selection under mild structural assumptions of: 1) gradually increasing class complexity, and 2) diminishing best-in-class policy improvement with increasing class complexity (diminshing marginal returns). As we discuss in \Cref{sec:prelim}, our assumptions reflect a natural setting for model selection. Further, even without such assumptions, our algorithm still achieves state-of-the-art (SOTA) guarantees (though not costless).

Our algorithm also addresses the bias-variance tradeoff inherent in contextual bandits \citep{foster2020adapting,krishnamurthy2021adapting, krishnamurthy2021tractable}, a topic that remains unexplored in the literature on model selection for contextual bandits. Existing work focuses on adapting to the complexity of the smallest realizable class, attempting to identify the single best-performing algorithm. We propose that attempting to find a single best-performing algorithm for all time horizons may not be the most effective strategy because simpler model classes provide better guarantees for shorter time horizons, while more expressive classes outperform in longer time horizons. We therefore argue that costless model selection should adapt to the simplest class where variance dominates bias. The difficulty in achieving such costless guarantees lies in detecting when the unknown bias of a class starts dominating its variance, and correcting for the potential under-exploration costs involved with delays in this detection.

We overcome this challenge to achieve our new definition of costless model selection through two main innovations:
\begin{enumerate}[topsep=0pt,itemsep=3pt,parsep=0pt]
    \item We develop a new misspecification test based on the accuracy of estimated reward models in evaluating policies from different classes. A key property of this test is that it fails (with high probability) after reward model bias dominates variance but before policy class bias dominates variance, allowing us to smoothly navigate the bias-variance trade-off. The misspecification test is of independent interest. Subsequent work has used this test to enable efficient pure exploration algorithms without assuming realizability \citep{krishnamurthy2023proportional}.
    \item We quantify the cost of potential under-exploration due to delays in detecting when reward model bias dominates variance, and develop a method called ``self-correction" to resolve any potential under-exploration.
\end{enumerate}

\subsection{Related work}

While \cite{marinov2021pareto} have already responded to the COLT 2020 open problem in the negative, we argue that this result is too pessimistic. Their specific counterexample of one very simple and one very complex class is an unfavorable setting for model selection and unnecessary to enforce in practice. Building on recent work quantifying the bias-variance tradeoff in contextual bandits \citep{foster2020adapting,krishnamurthy2021adapting,krishnamurthy2021tractable,carranza2023flexible} and literature reducing contextual bandit problems to supervised learning tasks 
\citep{langford2007epoch, dudik2011efficient, agarwal2012contextual, agarwal2014taming, foster2018practical, foster2020beyond} we show that it is indeed possible to achieve costless regret bounds under mild structural assumptions.
Existing algorithms for model selection in contextual bandits can generally be described as adopting either sequential \citep[e.g.][]{foster2019model} or parallel \citep[e.g.][]{agarwal2017corralling} search strategies. These are two alternative approaches to addressing the main challenge of model selection in contextual bandits: balancing exploration and exploitation in classes of increasing complexity. Both strategies consider bandit algorithms corresponding to models from each class $\F_i$, and try to identify the ``best" (simplest realizable) model class such that the corresponding algorithm minimizes regret.

Sequential search strategies have largely focused on model selection over a nested sequence of linear classes $\F_1\subset \F_2 \subset \dots \subset \F_M$ that are linear over a nested sequence of feature maps \citep[e.g.][]{foster2019model}. In this strategy, contextual bandits are run in sequence with increasing class complexity. For each model class $\F_i$, some share of rounds are devoted to sampling arms uniformly at random and testing for misspecification, with the ultimate goal of identifying the smallest realizable class, $\F_{i^*}$, for $i^*\in[M]$. %
To achieve costless guarantees, these strategies rely on stringent distributional assumptions for model identification known as diversity conditions -- i.e., they assume that the minimum eigenvalue of the covariance matrix for these feature maps is greater than some positive constant.\footnote{This may not be easily satisfied for feature maps with many correlated features.} This is in contrast to our algorithm which has no such requirements on the feature distribution.

Parallel search strategies use master algorithms \citep[e.g.][]{agarwal2017corralling} to run $M$ contextual bandit algorithms in parallel, one for each of the $M$ classes. The master algorithm allocates rounds to the $M$ base algorithms, learns which algorithm maximizes expected cumulative reward, and ultimately allocates most rounds to this algorithm. Since the introduction of this approach, several master algorithms have been proposed \citep[e.g.][]{arora2021corralling, pacchiano2020model,pacchiano2020regret}. However, none achieve costless model selection. %

\subsection{Preliminaries}\label{sec:prelim}

The stochastic contextual bandit setting is defined by a set of contexts $\Xscript$, a finite set of arms $\A=\{1,...,K\}$, and a distribution $D$ over contexts and arm rewards. At every time-step $t \in [T]$, nature samples a context $x_t\in\Xscript$ and reward vector $\smash{r_t \in [0,1]^K}$ from the fixed but unknown distribution $D$. Upon observing context $x_t$, the learner chooses an arm $a_t$ and receives a reward $r_t(a_t)$. Unless stated otherwise, all expectations are taken with respect to $D$.

We let $f^*:\Xscript \times \A \rightarrow [0,1]$ denote the true conditional expected reward function given contexts and actions; i.e. $f^*(x,a):=\E[r_t(a)|x_t=x]$. A model $f$ is a map from $\Xscript\times\A$ to $[0,1]$, and a model class $\F$ is a set whose elements are models. A policy $\pi$ is any function that maps contexts to a distribution over arms, and a policy class $\Pi$ is a set of policies. For deterministic policies, $\pi(x)$ denotes the arm recommended by policy $\pi$ at context $x$, and for randomized policies, $\pi(a|x)$ denotes the probability of sampling arm $a$ at context $x$.
For any model $f$, we let $\pi_f$ denote the deterministic policy induced by the model $f$, that is $\pi_f(x) := \arg\max_a f(x, a)$ for every $x$.\footnote{Where ties are broken with any fixed tie-breaking rule.} We let $\pi^*$ denote the policy that maximizes the conditional mean reward; i.e., $\pi^{*}(x) = \arg\max_a f^*(x, a)$.

We use the term exploration policy to refer to any randomized policy that our algorithm constructs for use in exploration. For any exploration policy $p$, we let $D(p)$ be the induced distribution over $\Xscript\times\A\times[0,1]$, where sampling $(x,a,r(a))\sim D(p)$ is equivalent to sampling $(x,r)\sim D$ and then sampling $a\sim p(\cdot|x)$. We let $p_t$ denote the exploration policy for round $t$.

For any model $f$ and policy $\pi$, we let $f(x,\pi(x)):=\E_{a\sim\pi(x)}[f(x,a)]$ at every context $x$, and we let $R_f(\pi)$ denote the expected instantaneous reward of the policy $\pi$ with respect to model $f$:
\begin{equation}
    \label{eq:instantaneous-reward}
    R_f(\pi) := \E_{x \sim D_{\Xscript}}[f(x, \pi(x))].
\end{equation}
Similarly, we let $\Reg_f(\pi)$ denote the expected instantaneous regret for policy $\pi$ with respect to model $f$:
\begin{equation}\label{eq:instantaneous-regret}
    \Reg_f(\pi) := \E_{x \sim D_{\Xscript}}[f(x,\pi_f(x)) - f(x, \pi(x))].
\end{equation}
When there is no possibility of confusion, we write $R(\pi)$ and $\Reg(\pi)$ to mean $R_{f^*}(\pi)$ and $\Reg_{f^*}(\pi)$ respectively. In this paper, we study contextual bandit algorithms that minimize expected cumulative regret $\CReg_T$:
\begin{equation}
    \label{eq:cum_regret}
    \CReg_T := \sum_{t=1}^T \Reg_{f^*}(p_t). %
\end{equation} 
We consider $M$ reward model classes $\F_1,\F_2,\dots,\F_M$. We let parameter $d_i$ denote a bound on the complexity of class $\F_i$. %
Without loss of generality, for all $i\in[M]$, we require $d_i\in \{2^q|q\in\N \}$.\footnote{This is without loss of generality because we can always round $d_i$ up to the nearest exponent of $2$, only increasing excess risk bounds by a constant multiplicative factor.} For notational convenience, we group the model classes $\F_1,\dots,\F_M$ in terms of their complexities. Let $M'$ be the number of unique parameters in the set $\{d_i| i\in[M] \}$, and let $\alld_i$ be the $i$-th smallest parameter in this set such that $\alld_1 \leq \alld_2 \leq \dots \leq \alld_{M'}$. For all $i\in[M']$, we then define model class $\allF_i$ and corresponding policy class $\allPi_i$:
\begin{equation}
\label{eq:allpi}
    \allF_i := \bigcup_{\{j| d_j\leq \alld_i\} } \F_j, \;\;\; \allPi_i := \{\pi_f| f\in \allF_i \}.
\end{equation}
We let $\pi^*_i$ denote the policy that maximizes the conditional mean reward among those belonging to class $\Tilde{\Pi}_i$, that is $\pi^*_i(x)=\operatorname{argmax}_{\pi\in\allPi_i}f^*(x,\pi(x))$. Similarly, we let $\Reg_{i}(\pi)$ denote the true expected instantaneous regret against the best policy in class $\Tilde{\Pi}_i$:
\begin{equation}
    \label{eq:instantaneous-releative-regret}
    \Reg_{i}(\pi) := \max_{\Tilde{\pi}\in \Tilde{\Pi}_i}R(\Tilde{\pi}) - R(\pi).
\end{equation}
We can also define the bias (misspecification error) of policy class $\allF_i$ as:
 \begin{equation}
     \beta_i:= R(\pi^*)-R(\pi^*_i).
 \end{equation}
To quantify how well class $\allF_i$ can approximate $f^*$, we use the definition of average squared misspecification error studied in \cite{krishnamurthy2021adapting}. Similar definitions of misspecification were studied in \cite{foster2020adapting} and \cite{krishnamurthy2021tractable}. We denote by $B_i$ the average squared misspecification error for the class $\allF_i$, that is:
\begin{equation}
    \label{eq:capb}
    B_i := \max_{p} \min_{f\in\allF_i} \E_{x\sim D_{\Xscript}}\E_{a\sim p(\cdot|x)} [(f(x,a)-f^*(x,a))^2],
\end{equation}
where $D_{\Xscript}$ is the marginal distribution of $D$ on the set of contexts $\Xscript$. We label model class $\allF_i$ as misspecified if $B_i>0$, and as well-specified or realizable if $B_i=0$. Note the difference in scales between our two measures of misspecification error: $B_i$ captures squared error while $\beta_i$ captures non-squared error. For notational convenience, we let $\beta_0=B_0=1$ and $\alld_0=0$.
\begin{assumption}[Realizability]
\label{ass:realizability}
We assume that there exists a class index $i\in[M']$ such that $B_i=0$, and we let $i^*$ denote the smallest class index with zero squared misspecification error.
\end{assumption}
\Cref{ass:realizability} is standard and allows $\CVBandit$ to reliably use more complex classes upon detecting misspecification. %
\Cref{ass:gradual,ass:opt-value-jump} formalize our structural conditions of gradually increasing class complexity and diminishing marginal returns from increasing class complexity. 
\begin{assumption}[Gradually Increasing Class Complexity]
\label{ass:gradual}
    For any $i\in[i^*]$, there exists a class index $j\in[M']$ such that $\alld_i<\alld_j\leq \omega \alld_i$ for some fixed but unknown constant $\omega >1$.
\end{assumption}
\begin{assumption}[Diminishing Policy Improvement with Increasing Class Complexity]
\label{ass:opt-value-jump}
Consider any $i\in[M']$ such that $B_i>0$. Let $j\in[M']$ be the largest index such that $\Tilde{d}_j \leq \omega \Tilde{d}_i$. We denote by $\Delta_i$ the improvement in the best-in-class policy value obtained by moving from policy class $i$ to policy class $j$, defined as follows:
$$ \Delta_i := \max_{\pi\in\allPi_j}\E_{x\sim D_{\Xscript}}[f^*(x,\pi(x))] - \max_{\pi\in\allPi_i}\E_{x\sim D_{\Xscript}}[f^*(x,\pi(x))]. $$
We assume that $\Delta_i$ is non-increasing in class index $i$.
\end{assumption}
Both \Cref{ass:gradual,ass:opt-value-jump} are parameterized by an unknown parameter $\omega>1$ quantifying the statistical hardness of the instance. In particular, for $\omega=\alld_{i^*}/\alld_1$, the assumptions trivially hold and so do the negative results of \cite{marinov2021pareto} -- in this case we achieve SOTA (though not costless) guarantees. However, these assumptions are often satisfied for much smaller $\omega$ -- in which case we achieve costless model selection guarantees.

\Cref{ass:gradual} can be ensured with small $\omega$ by construction, as it is always possible to add additional classes for a small cost.
\Cref{ass:opt-value-jump} with small $\omega$ is standard in the fields of statistics and machine learning; evidence of this can be found in recent work on neural scaling laws, where loss scales as a power-law with model size \citep{kaplan2020scaling,hestness2017deep}; decision trees and random forests, where it is well-known that larger complexity parameters and more terminal nodes offer diminishing returns \citep{Boehmke2019HandsOnML, rpart}; and empirical results of discrete event simulation \citep{simulation_dimret}.

\section{ORACLES}\label{sec:oracles}
There is a long line of work that reduces contextual bandit algorithms to oracle subroutines \citep[e.g.,][]{agarwal2014taming,foster2018practical,foster2020beyond}. We take a similar approach, and describe our key oracle subroutines and assumptions in this section. 

To estimate models in class $\allF_i$, we use a model selection oracle over the set $\{\F_k| d_k\leq \alld_i \}$. In \Cref{ass:model-selection}, we state our requirements for this oracle.
\begin{oracle}[Estimation Oracle]
\label{ass:model-selection}
For all $j\in[M']$, we assume access to an offline model selection oracle for estimation ($\MSOracle_j$) over classes $\{\F_k| d_k\leq \alld_j \}$ satisfying the following property: There exists a constant $C_0 \geq 1$ such that for any exploration policy $p$, any natural number $n$, and any $\zeta \in (0, 1)$, the following holds with probability at least $1-\zeta$:
\begin{equation}
\begin{split}
\E_{x\sim D_{\Xscript}}\E_{a\sim p(\cdot|x)}[ (\hatf(x, a) - f^*(x,a))^2 ] \\
\leq \min_{i\in[j]} \bigg(C_0 B_i + \xi_i(n,\zeta)\bigg) .
\end{split}
\end{equation}
Here, $\hatf$ is the output of $\MSOracle_j$ fitted on $n$ independently and identically drawn samples from $D(p)$, $B_i$ is defined in \eqref{eq:capb}, and $\xi_i$ is given by:
\begin{equation}
\label{eq:common-rate}
    \xi_i(n,\zeta) := C_1\bigg(\frac{ \alld_i\ln(nM/\zeta)}{n}\bigg)^{\rho},
\end{equation}
for some known constant $C_1 > 0$ and $\rho\in(0,1]$.\footnote{That is, the function $\xi_i(\cdot,\cdot)$ is known and can be used to compute (exploitation) parameters of our algorithm.}
\end{oracle}
We refer to $\xi_i(\cdot,\cdot)$ as the estimation rate for model class $\allF_i$ as it can be used to bound the squared prediction error of a regression oracle on model class $\allF_i$. In \Cref{app:construct-estimation-oracle}, we outline one of many approaches to construct an oracle that achieves the ``fast rates'' of \Cref{ass:model-selection}. The approach we describe there is based on validation with the holdout method. Other potential approaches include cross validation, aggregation algorithms \citep[see][ and references therein]{lecue2014optimal}, and penalized regression \citep[see relevant chapters in ][]{koltchinskii2011oracle,wainwright2019high}.

An important component of $\CVBandit$ is the comparison of a given policy's true value with its value according to estimated reward models. These tests verify the accuracy of estimated reward models, allowing us to detect misspecification. \Cref{ass:policy-val-estimation,ass:policy-val-evaluation} provide rates for estimating these quantities. 

\begin{oracle}[Direct Method Policy Estimation Rate]
\label{ass:policy-val-estimation}
For any index $i\in[M']$, any set of $M'+1$ policies $\{q_0,q_1,\dots,q_{M'}\}$, any reward model $f$, any natural number $n$, and any $\zeta \in (0, 1)$, the following holds with probability at least $1-\zeta$:
\begin{align*}
    \bigg| \frac{1}{n} \sum_{x\in S}& f(x, \pi(x)) \; - \E_{x\sim D_{\Xscript}}[f(x,\pi(x))]\bigg|\\
    &\leq \sqrt{\xi_i(n,\zeta)}, \; \forall \; \pi\in \allPi_i\cup \{q_0,q_1,\dots,q_{M'}\},
\end{align*}
where $S$ is a set of $n$ independently and identically drawn samples from the distribution $D_{\Xscript}$.
\end{oracle}
\Cref{ass:policy-val-estimation} often follows from uniform convergence arguments \citep[see][]{shalev2014understanding, koltchinskii2011oracle,wainwright2019high}. For example, for finite function classes with $\xi_i(n,\zeta)=\ordO(\ln(|\F_i|/\zeta)/n)$, \Cref{ass:policy-val-estimation} follows from Hoeffding's inequality with uniform convergence.

\begin{oracle}[Policy Evaluation Oracle]
\label{ass:policy-val-evaluation}
For any index $i\in[M']$, any set of $M'+1$ policies $\{q_0,q_1,\dots,q_{M'}\}$, any natural number $n$, any exploration policy $p$ with $p(a|x)\geq \eta$ for all $a\in\A$ and $x\in\Xscript$, and any $\zeta \in (0, 1)$, the following holds with probability at least $1-\zeta$:
\begin{align*}
    &\bigg|\widehat{R}(\pi) \; - \E_{x\sim D_{\Xscript}}[f^*(x,\pi(x))]\bigg| \leq \frac{\xi_i(n,\zeta)}{\eta} \\
    & + \sqrt{\E_{x\sim D_{\Xscript}}\bigg[\frac{1}{p(\pi(x)|x)}\bigg]\xi_i(n,\zeta)}, \forall \; \pi\in \allPi_i\cup\{q_k\}_{k=0}^{M'}, 
\end{align*}
where $\widehat{R}$ is the output of $\EvalOracle_i$ fitted on $n$ independently and identically drawn samples from $D(p)$.
\end{oracle}
When $\hatR$ is estimated via inverse propensity score estimation, \cite{agarwal2014taming} show that \Cref{ass:policy-val-evaluation} is satisfied for finite function classes with $\xi_i(n,\zeta)=\ordO(\ln(|\F_i|/\zeta)/n)$. The covering arguments in \cite{maurer2009empirical} can be used to show \Cref{ass:policy-val-evaluation} holds for general function classes. \footnote{For example, \cite{jin2023upper} show \Cref{ass:policy-val-evaluation} holds for tree based and neural network function classes with appropriate choices of 
$\xi(\cdot,\cdot)$.}

\section{ALGORITHM}\label{sec:alg}
We present our algorithm, $\CVBandit$%
, in \Cref{alg:cv-bandit}. As its name suggests, $\CVBandit$ is based on an inverse gap weighting (IGW) approach to action selection, which provides a simple analytical handle on important quantities like the expected inverse probability weight for any policy at any round, and is often used to develop optimal algorithms \citep[][]{abe1999associative,foster2020beyond,foster2020instance,simchi2020bypassing}. However, $\CVBandit$ deviates from the standard IGW approach in three major respects: 1) how the model is estimated -- specifically, using $\MSOracle$ (e.g. classical model selection for supervised learning; see \Cref{sec:oracles}); 2) how the exploitation parameter is determined -- using a novel misspecification test that selects among $M'$ candidate exploitation parameters (\Cref{sec:misspec}); and 3) how the candidate exploitation parameters scale with the number of rounds -- using a ``self-correction" strategy that accounts for any potential under-exploration in earlier rounds and updates the exploitation parameter accordingly (\Cref{sec:self_correction}). %
$\CVBandit$ proceeds in epochs indexed by $m$, with epoch $m$ spanning time-steps $t\in[\tau_{m-1}+1,\tau_m]$. %
At any such $t$, the algorithm observes context $x_t$ and samples action $a_t$ from the distribution of exploration policies $p_{m}$, defined as:
\begin{align}
   \label{eq:action_kernel}
   p_{m}(a|x):=
   \begin{cases}
    \frac{1}{K+\gamma_{m} \left(\hatf_m(x, \hat{a}) - \hatf_m(x,a) \right)}, & a\neq \hat{a},\\
   1 - \sum_{a'\neq \hat{a}} p_{m}(a'|x), & a=\hat{a}.
   \end{cases}
\end{align}
Here, $\hatf_m$ is an estimate of the reward model computed via $\MSOracle$ with data from previous epochs, and ${\hat{a} = \arg\max_a \hatf_m(x,a)}$ is the predicted best action. The exploitation parameter $\gamma_{m}$ governs the balance between exploration and exploitation. The higher the value of $\gamma_m$, the greater the probability that the greedy action $\hata$ is chosen.  The remainder of this section focuses on how the exploitation parameter $\gamma_m$ should be chosen. %

To optimize cumulative regret, we want to exploit as much as possible while still allowing for estimation of useful reward models. Following the IGW approach \citep{foster2020beyond,simchi2020bypassing}, $\gamma_{m}$ should be specified %
based on the reward model error to balance this tradeoff. A common metric for gauging reward model error is the mean squared prediction error, 
which decomposes into bias and variance. Bias ($B_i$) is unknown and can be challenging to estimate, but in early rounds when the variance ($\xi_i(\cdot,\cdot)$) dominates the bias, we can bound the squared error by 2$\xi_i(\cdot,\cdot)$ \citep[similar ideas were used to quantify the bias-variance tradeoff in contextual bandits in][]{krishnamurthy2021adapting}. Following prior IGW approaches, we can use this candidate bound on the squared prediction error to set the candidate exploitation parameter $\gamma_{m+1,i}$: 
\begin{equation}
    \label{eq:model-gamma}
    \gamma_{m+1,i} := \max\Bigg(K,\sqrt{\frac{K}{8\xi_{i}\Big(\frac{\tau_{m} - \tau_{m-1}}{2}, \frac{\delta}{6TM'^2}\Big)}}\Bigg),
\end{equation}
and corresponding exploration policies $p_{m+1,i}$ (by substituting $\gamma_{m+1}=\gamma_{m+1,i}$ in the formula for $p_{m+1}$). Parameter $\gamma_{m+1,i}$ induces sufficient exploration so long as the variance of estimating a reward model from class $i$ dominates the bias of that class.\footnote{Estimation variance decreases with more data and is eventually dominated by bias.}  We refer to the (unknown) last epoch where variance dominates bias as the ``safe epoch,'' denoted by $m^*_i$:\footnote{$\msafe_i$ is unknown because the bias $B_i$ is unknown.}
\begin{equation*}
    \label{eq:msafe-2}
    \begin{aligned}
     \msafe_i := \max \bigg\{ m \big|  \xi_i\Big(\frac{\tau_m-\tau_{m-1}}{2}, \frac{\delta}{6TM'^2} \Big) \geq C_0B_{i} \bigg\}.
    \end{aligned}
\end{equation*}
We let $\msafe_0=1$ and note that $\msafe_{i^*}$ is infinity.\footnote{That is, the exploitation parameter corresponding to the realizable class always induces sufficient exploration.} Note that since $\xi_i$ is increasing in $i$, $\gamma_{m+1,i}$ is non-increasing in $i$. Therefore, to maximize exploitation (while ensuring enough exploration to estimate a useful model), we want to use the exploitation parameter corresponding to the simplest class $i\in[M']$ such that $m\in[\msafe_i]$.

Since bias ($B_i$), and hence the safe epoch ($m^*_i$) are unknown, we don't know which of the candidate exploitation parameters to choose for a given epoch $m$. The challenge of estimating bias stems from the difficulty of estimating prediction error of the estimated reward model.\footnote{Asymptotically, prediction error of the estimated reward model converges to squared misspecification error.} Note that we can't extract measures of mean prediction error from mean squared error (or other similar measures of error that average absolute differences between outcomes and predicted values) due to unknown irreducible noise.\footnote{For example, while the mean squared error can be empirically estimated, it converges to the sum of mean squared prediction error and irreducible noise. }

To overcome this, we develop a new way of testing if these variance-based candidate error bounds actually bound the prediction error of the estimated model. In particular, we posit a shift in perspective from measuring error via the mean squared prediction error, and instead gauge the error of an estimated reward model by its accuracy in evaluating candidate policies. We do this by comparing an estimated reward model's direct method estimates (see \Cref{ass:policy-val-estimation}) with consistent policy estimates (see \Cref{ass:policy-val-evaluation}). Importantly, by leveraging the consistent estimators available for policies (which average out reward noise), we are able to capture the real prediction error of the estimated model in evaluating the candidate policies.\footnote{We are unavoidably limited in the number of policies used to perform the comparison (due to issues related to multiple hypothesis testing), and using larger policy classes for the comparison requires more exploration data.} This policy-based approach forms the foundation for our misspecification test, $\MisspecificationOracle$ (described in the next section).

\begin{algorithm}[htbp]
  \caption{$\CVBandit$ (Model Selection with Inverse Gap Weighting)}
  \label{alg:cv-bandit}
  \textbf{input:} Initial epoch length $\tau_1\geq 2$, horizon $T$, and confidence parameter $\delta$.
  \begin{algorithmic}[1] %
  \State Let $\hatf_1 \equiv 0$, $i_1=1$, $\gamma_1=1$, $\tau_0 = 0$, and $\msafealg=0$
  \For{epoch $m=1,2,\dots$}
    \State Let $p_m$ be given by \eqref{eq:action_kernel}
    \For{round $t=\tau_{m-1}+1,\dots, \tau_{m}$ }
      \State Observe context $x_t$
      \State Sample $a_t \sim p_m(\cdot|x_t)$ and observe $r_t(a_t)$
    \EndFor
    \State Let $S_m$ denote the data collected in epoch $m$. Split $S_m$ into training ($S_{m,\tr}$) and holdout ($S_{m,\ho}$) sets of roughly equal size ($|S_{m,\ho}|=\lceil |S_m|/2 \rceil$)
    \State $\hatf_{m+1}\;\,\leftarrow \MSOracle_{M'}(S_{m,\tr})$
    \State $\hatf_{m+1,i}\leftarrow \MSOracle_{i}(S_{m,\tr})\quad \forall i\geq i_m$
    \State $\hatR_{m+1} \,\leftarrow \EvalOracle_i(S_{m,\ho})$
    \State $\hatR_{m+1,f}(\pi) := \frac{1}{|S_{m,\ho}|}\sum_{(x,a,r)\in S_{m,\ho}} f(x,\pi(x))$
    \State Let $i_{m+1}=\MisspecificationOracle_{i_m}(S_{m,\ho})$
    \If{$i_{m+1} \neq i_{m}$}
        \State $\msafealg%
        \leftarrow m +  \lceil\log_2(\log_2(\gamma_{m,1}/\gamma_{m,i_{m+1}}))\rceil$
    \EndIf
    \State $\tau_{m+1}\leftarrow \tau_m + (1+\mathbbm{1}{\{m \geq \msafealg%
    \}} )(\tau_m - \tau_{m-1})$.
    \State $\gamma_{m+1} \leftarrow \gamma_{m+1,i_{m+1}} $.
  \EndFor
  \end{algorithmic}
\end{algorithm}

\subsection{Misspecification Test}\label{sec:misspec}

As discussed above, a key challenge for our exploration strategy is the specification of the exploitation parameter $\gamma_{m}$ given that the safe epochs $m^*_i$ are unknown. To address this challenge, we introduce a new misspecification test: $\MisspecificationOracle$. This section describes the test in detail. In subsequent work, this test enables efficient pure exploration in contextual bandits \cite{krishnamurthy2023proportional}.

$\MisspecificationOracle$ adopts a policy-based approach to assess estimated models. With high-probability, the test detects misspecification for class $i$ after its reward model bias dominates variance (that is, after the corresponding safe epoch)
but before its policy class bias dominates variance. This allows $\CVBandit$ to explore with the exploitation parameter corresponding to the simplest class whose variance dominates bias. There are three components to $\MisspecificationOracle$, given in \Cref{def:MTest}. Each of these is sufficient to rule out classes whose bias dominates variance.

The main policy-based misspecification test in $\MisspecificationOracle$ checks whether the estimated reward models can be used to construct sufficiently accurate direct method estimates of policy values. We test this by comparing $\hatR_{m+1}(\pi)$, the estimate of a policy value obtained via $\EvalOracle(S_m)$, and $\hatR_{m+1,f}(\pi)$, the direct method estimate of a policy value under some estimated reward model $f$, defined by:
\begin{equation}
\label{eq:implicit-policy-estimation}
    \hatR_{m+1,f}(\pi) := \frac{1}{|S_{m,\ho}|}\sum_{(x,a,r)\in S_{m,\ho}} f(x,\pi(x)).
\end{equation}
If the difference between these two estimates surpasses the threshold given in \Cref{def:MTest} for some $f\in\{\hatf_{m+1},\hatf_{m+1,i}\}$, this indicates the bias of class $i$ dominates the variance of estimating from class $i$. In other words, we are underestimating the error of the estimated reward model, and so parameter $\gamma_{m+1,i}$ does not induce sufficient exploration.

\begin{oracle}(Misspecification Test Oracle)\label{def:MTest}
In each epoch $m$, our misspecification test $\MisspecificationOracle_{i_m}$ identifies index $i_{m+1}$, which we define as the smallest index such that $i_{m+1}\geq i_m$ and, for all $i\geq i_{m+1}$, $j\geq i$, $h\leq i_m$, and $\alpha>0$, the following inequalities hold:
\begin{equation*}
\begin{aligned}
&\textbf{Policy-based misspecification test}\\
    &|\hatR_{m+1}(\pi) - \hatR_{m+1,f}(\pi)| \\
    &\quad\leq \bigg( \frac{1+\theta_{i,j}}{\alpha} + \frac{(1+\theta_{i,j})\alpha}{16} \\
    &\quad\quad + \frac{(2\theta_{i,j}^2+(1+\theta_{i,j})^2/\alpha)\gamma_m}{\gamma_{m+1,i}} + \theta_{i,j} \bigg)\frac{K}{\gamma_{m+1,i}} \\
    &\quad\quad + \frac{(1+\theta_{i,j})\gamma_m}{\alpha \gamma_{m+1,i}}\hatReg_{m+1,\hatf_m}(\pi)\\
    &\quad\quad\forall {f\in\{\hatf_{m+1},\hatf_{m+1,i}\}},\pi\in\allPi_{j}\cup\Pi_{0,m+1,i},\\
&\textbf{Reward model agreement}\\
    &\hatReg_{m+1,\hatf_{m+1}}(\pi_{\hatf_{m+1,i}})\\
    &\quad\leq 26\frac{\gamma_{m,h}}{\gamma_{m,i}}\bigg(\frac{\gamma_{\msafealg_{h-1},1}}{\gamma_{\msafealg_{h-1},i}}\bigg)^{1/2^{m-\msafealg_{h-1}}}\frac{K}{\gamma_{m+1,i}},\\
    &\hatReg_{m+1,\hatf_{m}}(\pi) \\
    &\quad\leq 4\hatReg_{m+1,f}(\pi) \\
    &\quad\quad + 34\frac{\gamma_{m,h}}{\gamma_{m,i}}\bigg(\frac{\gamma_{\msafealg_{h-1},1}}{\gamma_{\msafealg_{h-1},i}}\bigg)^{1/2^{\max(0,m-\msafealg_{h-1}-1)}}\frac{K}{\gamma_{m,i}},\\
    &\quad\quad\forall {f\in\{\hatf_{m+1},\hatf_{m+1,i}\}},\pi\in\Pi_{0,m+1,i},
\end{aligned}
\end{equation*}
where $\Pi_{0,m+1,i}=\{\pi_{\hatf_{m+1}},\pi_{\hatf_{m+1,i}},p_{m+1,1},\dots,p_{m+1,M'} \}$, $\theta_{i,j}=(\alld_j/\alld_i)^{\rho/2}$,
and $\msafealg_i:=\max\{m|i_m\leq i\}$ ($\msafealg_0:=0$). That is, $\msafealg_i$ is the latest epoch such that model class $i$ has not been labeled as misspecified. The above inequalities are derived in \Cref{lem:explicit-validation,lem:new_mtest,lem:policy-reg-bound-wrt-old-model-test}, respectively. Index $i_{m+1}$ is the output of $\MisspecificationOracle_{i_m}$.
\end{oracle}

$\MisspecificationOracle$ also includes two tests to verify that the estimated reward model exhibits agreement across possibly well-specified classes and across epochs. The first component confirms that the policy induced by $\hatf_{m+1,i}$ (the model estimated for class $i$) is a good policy according to $\hatf_{m+1}$ (the model estimated across all classes). In this case, $\hatReg_{m+1,\hatf_{m+1}}(\pi_{\hatf_{m+1,i}})$ should not exceed the threshold given in \Cref{def:MTest}, where $\hatReg_{m,f}(\pi)$ denotes the empirical regret for policy $\pi$ with respect to model $f$:
\begin{equation}\label{eq:emp_regret}
\hatReg_{m,f}(\pi)=\hat{R}_{m,f}(\pi_f)-\hat{R}_{m,f}(\pi).
\end{equation}
This helps ensure that once $\hatf_{m+1}$ believes a notably better policy lies in a larger policy class, we use candidate exploitation parameters corresponding to larger classes to ensure sufficient exploration. %

To ensure reward model agreement across epochs, the final component of $\MisspecificationOracle$ confirms that the candidate exploration policies ($\Pi_{0,m+1,i}$, defined in \Cref{def:MTest}) have sufficiently low regret under the prior epoch's estimated reward model. Thus, $\hatReg_{m+1,\hatf_m}(\pi)$ should not exceed $\hatReg_{m+1,f}(\pi)$ for $f\in\{\hatf_{m+1},\hatf_{m+1,i}\}$ for policies in $\Pi_{0,m+1,i}$ beyond the threshold given in \Cref{def:MTest}. This test helps confirm that candidate exploration policies for epoch $m+1$ were sufficiently explored in epoch $m$.

In \Cref{sec:constructing-misspecification-oracle}, we describe one approach to computationally test the inequalities in \Cref{def:MTest} via cost-sensitive classification. Note that not only do the tests in $\MisspecificationOracle$ check whether reward model bias dominates variance, but they also verify the accuracy of the estimated reward models in evaluating policies from different classes.

\subsection{Self-correction}\label{sec:self_correction}

To recap, for any epoch $m$, the goal of \Cref{def:MTest} is to verify that the prediction error of model $\hatf_m$ can be bounded by the variance of estimating from the class $\allF_{i_m}$. This verification involves testing if $\hatf_m$ can accurately evaluate policies from classes of various complexities up to this error bound. Unfortunately, this verification is loose up to a factor $(\alld_j/\alld_{i_m})^{\rho/2}$, for policy class $\allPi_j$ more complex than $\allPi_{i_m}$. The detection of misspecification indicates that in past epochs, the ability of $\hatf_m$ to evaluate policies from $\allPi_j$ may have been loose, up to a multiplicative factor of $(\alld_j/\alld_{i_m})^{\rho/2}$. %
As a result, we may have under-estimated the value of some policies up to this factor, leading to corresponding under-exploration in prior epochs. Note that under-exploration in prior rounds affects our ability to evaluate these policies well in future rounds. Hence upon detecting misspecification, we want to correct for the effects of potential under-exploration on our estimated models.\footnote{The factors described here bound the potential extent of this under-exploration.} Our analysis uncovers a self-correction mechanism to manage this challenge, which we describe in this section. 

Upon detecting misspecification, we hold the epoch length fixed for a small number of epochs. Importantly, this results in the candidate exploitation parameters being held fixed while the algorithm continues to collect data and improve the reward model via $\MSOracle$. As the estimated reward model improves, we better explore good policies that were previously not well-explored, leading to reward models that are better at estimating good policies. After this process continues for a small number of epochs, we will have sufficiently corrected for potential prior underexploration and can resume increasing our candidate exploitation parameters.

$\CVBandit$ is the first algorithm to leverage this strategy. The typical approach is to restart a bandit algorithm from scratch upon detecting misspecification \citep[e.g.,][]{foster2019model}. However, this may lead to larger than necessary cumulative regret, particularly leading to worse bounds for shorter horizons. Self-correction, in contrast, is efficient and unintrusive, allowing $\CVBandit$ to recalibrate after detecting misspecification in just a few epochs.

\section{MAIN RESULT}\label{sec:results}

We now present our main result in \Cref{thm:main-theorem}. %

\begin{restatable}[]{thm}{thmmain}
\label{thm:main-theorem}
Suppose \Cref{ass:gradual,ass:realizability,ass:opt-value-jump} hold and the oracle subroutines perform as stated in \Cref{ass:model-selection,ass:policy-val-estimation,ass:policy-val-evaluation,def:MTest}. With probability at least $1 - \delta$: for any $i,j\in[M']$ such that $i\leq j$ and $\allF_j$ not yet labelled as misspecified as of round $T$, $\CVBandit$ attains the following regret guarantee:
\begin{align}
\begin{split}
   \CReg_T \leq \ordOt\Bigg(& (\omega^2K^{1/\rho})\frac{\alld_{i-1}}{\beta_{i-1}^{2/\rho}}\\ 
   &+ \beta_jT + \bigg(\frac{\alld_j}{\alld_i}\bigg)^{\rho/2}\sqrt{K\alld_j^{\rho} T^{2-\rho}} \; \Bigg)
\end{split}
\end{align}
Here $\ordOt$ hides terms logarithmic in $T, M, 1/\delta, \alld_{i^*}$. Further, $\allF_j$ is not determined to be misspecified for at least $\Omega(\alld_j/B_j^{1/\rho})$ rounds.
\end{restatable}

To better understand \Cref{thm:main-theorem} and simplify our discussion, we focus on the implications for classes with $\rho=1$ \footnote{Estimation rates with $\rho=1$ hold for a wide variety of popular classes, e.g. finite function classes, linear classes, and classes with finite VC-sub-graph dimension \citep[see][]{koltchinskii2011oracle}.} and ignore constant factors, logarithmic factors, and $\omega$ from \Cref{ass:opt-value-jump}.
Then we achieve a cumulative regret bound of the form $(K \alld_{i-1})/\beta_{i-1}^2 + \beta_j T + \sqrt{\alld_j/\alld_i}\sqrt{K\alld_j T}$, so long as $\allF_j$ has not yet been determined to have larger bias than variance. Let us understand these terms better. 
The first term, $(K \alld_{i-1})/\beta_{i-1}^2$, bounds the time to detect misspecification in class $\allF_{i-1}$. This marks the number of rounds required for policy class bias ($\beta_{i-1}$) to dominate the corresponding variance for policy learning from class $i$ under uniform sampling ($\sqrt{K\alld_{i-1}/T}$). The second term, $\beta_j T$, accounts for the bias of class $\allPi_j$. This term would not appear in our bound had we defined cumulative regret relative to $\pi^*_j$ (the best policy in class $\allPi_j$). %
The third term, $\sqrt{\alld_j/\alld_i}\sqrt{K\alld_j T}$, is the product of two quantities. The quantity $\sqrt{K\alld_j T}$ accounts for the estimation variance for class $\allF_j$. This is also the minimax regret bound for contextual bandits working with class $\allF_j$ assuming realizability in this class. The quantity $\sqrt{\alld_j/\alld_i}$ accounts for potential under-exploration of policies in class $\allPi_j$ after self-correcting to the exploitation parameters induced by class $i$. Note that the multiplicative cost of $\sqrt{\alld_j/\alld_i}=1$ for $i=j$.
As stated earlier, the cumulative regret bound we discussed holds if $\allF_j$ was not determined to be misspecified through round $T$. Note that we only detect misspecification for class $\allF_j$ after the average squared misspecification error ($B_j$) is larger than the corresponding variance $(\alld_j/T)$. That is, $\allF_j$ is not determined to be misspecified for at least $\Omega(\alld_j/B_j)$ rounds. In this way, we always rely on the simplest class whose bias dominates the variance, achieving costless model selection guarantees under mild assumptions.

An important implication concerns how well we adapt to the realizable class $\allF_{i^*}$. Choosing $i=j=i^*$, our cumulative regret in terms of $(K,T,\alld_{i^*})$ is given by $\ordOt(\sqrt{K\alld_iT})$. This shows that despite the negative result in \cite{marinov2021pareto}, under mild assumptions, it is possible to achieve the costless regret guarantees requested in the COLT 2020 open problem \citep{foster2020open}.

Note that \Cref{ass:gradual,ass:opt-value-jump} capture instance hardness with parameter $\omega>1$. These assumptions are always satisfied for the choice $\omega=\alld_{i^*}/\alld_1$, however these assumptions are benign even with much smaller $\omega$ (as we argued in \Cref{sec:intro}). $\CVBandit$ does not need $\omega$ as an input, and automatically adapts to instance hardness. By setting $i=1$, we get regret bounds that are independent of $\omega$ (since $\alld_0\equiv 0$), giving us our worst-case guarantees. Further setting $j=i^*$, we get a regret bound of $\ordOt(\sqrt{\alld_{i^*}/\alld_1}\sqrt{KT\alld_{i^*}})$. Hence, we recover the SOTA (although not costless) guarantees from \cite{marinov2021pareto}.

\section{CONCLUSION}\label{sec:conclusion}

We study the feasibility of costless model selection in contextual bandits. First, we expanded the definition of costless model selection to not just adapt to the complexity of the simplest realizable class, but adapt to the complexity of the simplest class whose variance dominates the bias. This introduces the perspective of bias-variance trade-off to model selection in contextual bandits. Second, we identify mild assumptions under which costless model selection is feasible and can be achieved by our algorithm $\CVBandit$. If the unknown parameter $\omega$ in our assumptions is large enough, the assumptions we introduce are trivially satisfied -- in this case, we can't achieve costless model selection guarantees, but still recover near-optimal guarantees. Our analysis is enabled by two key algorithmic insights: our policy-based misspecification test and self-correction. \footnote{S.A. and S.K.K. are grateful for the generous support provided by Golub Capital Social Impact Lab and the Office of Naval Research grant N00014-19-1-2468. A.P. is grateful for the generous support of the Stanford Graduate Fellowship (SGF).}

\bibliography{ref.bib}

\begin{thebibliography}{}

\bibitem[Abe and Long, 1999]{abe1999associative}
Abe, N. and Long, P.~M. (1999).
\newblock Associative reinforcement learning using linear probabilistic
  concepts.
\newblock In {\em ICML}, pages 3--11. Citeseer.

\bibitem[Agarwal et~al., 2016]{agarwal2016making}
Agarwal, A., Bird, S., Cozowicz, M., Hoang, L., Langford, J., Lee, S., Li, J.,
  Melamed, D., Oshri, G., Ribas, O., et~al. (2016).
\newblock Making contextual decisions with low technical debt.
\newblock {\em arXiv preprint arXiv:1606.03966}.

\bibitem[Agarwal et~al., 2012]{agarwal2012contextual}
Agarwal, A., Dud{\'\i}k, M., Kale, S., Langford, J., and Schapire, R. (2012).
\newblock Contextual bandit learning with predictable rewards.
\newblock In {\em Artificial Intelligence and Statistics}, pages 19--26. PMLR.

\bibitem[Agarwal et~al., 2014]{agarwal2014taming}
Agarwal, A., Hsu, D., Kale, S., Langford, J., Li, L., and Schapire, R. (2014).
\newblock Taming the monster: A fast and simple algorithm for contextual
  bandits.
\newblock In {\em International Conference on Machine Learning}, pages
  1638--1646.

\bibitem[Agarwal et~al., 2017]{agarwal2017corralling}
Agarwal, A., Luo, H., Neyshabur, B., and Schapire, R.~E. (2017).
\newblock Corralling a band of bandit algorithms.
\newblock In {\em Conference on Learning Theory}, pages 12--38. PMLR.

\bibitem[Agrawal and Goyal, 2013]{agrawal2013thompson}
Agrawal, S. and Goyal, N. (2013).
\newblock Thompson sampling for contextual bandits with linear payoffs.
\newblock In {\em International Conference on Machine Learning}, pages
  127--135.

\bibitem[Arora et~al., 2021]{arora2021corralling}
Arora, R., Marinov, T.~V., and Mohri, M. (2021).
\newblock Corralling stochastic bandit algorithms.
\newblock In {\em International Conference on Artificial Intelligence and
  Statistics}, pages 2116--2124. PMLR.

\bibitem[Boehmke and Greenwell, 2019]{Boehmke2019HandsOnML}
Boehmke, B.~C. and Greenwell, B.~M. (2019).
\newblock Hands-on machine learning with r.

\bibitem[Carranza et~al., 2023]{carranza2023flexible}
Carranza, A.~G., Krishnamurthy, S.~K., and Athey, S. (2023).
\newblock Flexible and efficient contextual bandits with heterogeneous
  treatment effect oracles.
\newblock In {\em International Conference on Artificial Intelligence and
  Statistics}, pages 7190--7212. PMLR.

\bibitem[Dudik et~al., 2011]{dudik2011efficient}
Dudik, M., Hsu, D., Kale, S., Karampatziakis, N., Langford, J., Reyzin, L., and
  Zhang, T. (2011).
\newblock Efficient optimal learning for contextual bandits.
\newblock {\em arXiv preprint arXiv:1106.2369}.

\bibitem[Foster et~al., 2018]{foster2018practical}
Foster, D.~J., Agarwal, A., Dud{\'\i}k, M., Luo, H., and Schapire, R.~E.
  (2018).
\newblock Practical contextual bandits with regression oracles.
\newblock {\em arXiv preprint arXiv:1803.01088}.

\bibitem[Foster et~al., 2020a]{foster2020adapting}
Foster, D.~J., Gentile, C., Mohri, M., and Zimmert, J. (2020a).
\newblock Adapting to misspecification in contextual bandits.
\newblock {\em Advances in Neural Information Processing Systems}, 33.

\bibitem[Foster et~al., 2019]{foster2019model}
Foster, D.~J., Krishnamurthy, A., and Luo, H. (2019).
\newblock Model selection for contextual bandits.
\newblock In {\em Advances in Neural Information Processing Systems}, pages
  14741--14752.

\bibitem[Foster et~al., 2020b]{foster2020open}
Foster, D.~J., Krishnamurthy, A., and Luo, H. (2020b).
\newblock Open problem: Model selection for contextual bandits.
\newblock In {\em Conference on Learning Theory}, pages 3842--3846. PMLR.

\bibitem[Foster and Rakhlin, 2020]{foster2020beyond}
Foster, D.~J. and Rakhlin, A. (2020).
\newblock Beyond ucb: Optimal and efficient contextual bandits with regression
  oracles.
\newblock {\em arXiv preprint arXiv:2002.04926}.

\bibitem[Foster et~al., 2020c]{foster2020instance}
Foster, D.~J., Rakhlin, A., Simchi-Levi, D., and Xu, Y. (2020c).
\newblock Instance-dependent complexity of contextual bandits and reinforcement
  learning: A disagreement-based perspective.
\newblock {\em arXiv preprint arXiv:2010.03104}.

\bibitem[Hestness et~al., 2017]{hestness2017deep}
Hestness, J., Narang, S., Ardalani, N., Diamos, G., Jun, H., Kianinejad, H.,
  Patwary, M. M.~A., Yang, Y., and Zhou, Y. (2017).
\newblock Deep learning scaling is predictable, empirically.

\bibitem[Jin, 2023]{jin2023upper}
Jin, Y. (2023).
\newblock Upper bounds on the natarajan dimensions of some function classes.
\newblock In {\em 2023 IEEE International Symposium on Information Theory
  (ISIT)}, pages 1020--1025. IEEE.

\bibitem[Kaplan et~al., 2020]{kaplan2020scaling}
Kaplan, J., McCandlish, S., Henighan, T., Brown, T.~B., Chess, B., Child, R.,
  Gray, S., Radford, A., Wu, J., and Amodei, D. (2020).
\newblock Scaling laws for neural language models.

\bibitem[Koltchinskii, 2011]{koltchinskii2011oracle}
Koltchinskii, V. (2011).
\newblock {\em Oracle Inequalities in Empirical Risk Minimization and Sparse
  Recovery Problems: Ecole d’Et{\'e} de Probabilit{\'e}s de Saint-Flour
  XXXVIII-2008}, volume 2033.
\newblock Springer Science \& Business Media.

\bibitem[Krishnamurthy et~al., 2017]{krishnamurthy2017active}
Krishnamurthy, A., Agarwal, A., Huang, T.-K., Daum{\'e}~III, H., and Langford,
  J. (2017).
\newblock Active learning for cost-sensitive classification.
\newblock In {\em International Conference on Machine Learning}, pages
  1915--1924. PMLR.

\bibitem[Krishnamurthy et~al., 2021a]{krishnamurthy2021adapting}
Krishnamurthy, S.~K., Hadad, V., and Athey, S. (2021a).
\newblock Adapting to misspecification in contextual bandits with offline
  regression oracles.
\newblock {\em arXiv preprint arXiv:2102.13240}.

\bibitem[Krishnamurthy et~al., 2021b]{krishnamurthy2021tractable}
Krishnamurthy, S.~K., Hadad, V., and Athey, S. (2021b).
\newblock Tractable contextual bandits beyond realizability.
\newblock In {\em International Conference on Artificial Intelligence and
  Statistics}, pages 1423--1431. PMLR.

\bibitem[Krishnamurthy et~al., 2023]{krishnamurthy2023proportional}
Krishnamurthy, S.~K., Zhan, R., Athey, S., and Brunskill, E. (2023).
\newblock Proportional response: Contextual bandits for simple and cumulative
  regret minimization.
\newblock {\em arXiv preprint arXiv:2307.02108}.

\bibitem[Kuhn, 2022]{rpart}
Kuhn, M. (2022).
\newblock Machine learning with caret in r.

\bibitem[Langford and Zhang, 2007]{langford2007epoch}
Langford, J. and Zhang, T. (2007).
\newblock The epoch-greedy algorithm for contextual multi-armed bandits.
\newblock In {\em Proceedings of the 20th International Conference on Neural
  Information Processing Systems}, pages 817--824. Citeseer.

\bibitem[Lattimore and Szepesv{\'a}ri, 2020]{lattimore2020bandit}
Lattimore, T. and Szepesv{\'a}ri, C. (2020).
\newblock {\em Bandit algorithms}.
\newblock Cambridge University Press.

\bibitem[Lecu{\'e} et~al., 2014]{lecue2014optimal}
Lecu{\'e}, G., Rigollet, P., et~al. (2014).
\newblock Optimal learning with q-aggregation.
\newblock {\em Annals of Statistics}, 42(1):211--224.

\bibitem[Li et~al., 2010]{li2010contextual}
Li, L., Chu, W., Langford, J., and Schapire, R.~E. (2010).
\newblock A contextual-bandit approach to personalized news article
  recommendation.
\newblock In {\em Proceedings of the 19th international conference on World
  wide web}, pages 661--670. ACM.

\bibitem[Marinov and Zimmert, 2021]{marinov2021pareto}
Marinov, T.~V. and Zimmert, J. (2021).
\newblock The pareto frontier of model selection for general contextual
  bandits.
\newblock {\em Advances in Neural Information Processing Systems}, 34.

\bibitem[Maurer and Pontil, 2009]{maurer2009empirical}
Maurer, A. and Pontil, M. (2009).
\newblock Empirical bernstein bounds and sample variance penalization.
\newblock {\em arXiv preprint arXiv:0907.3740}.

\bibitem[Mitchell et~al., 2009]{mitchell2009general}
Mitchell, C., van~de Geer, S., et~al. (2009).
\newblock General oracle inequalities for model selection.
\newblock {\em Electronic Journal of Statistics}, 3:176--204.

\bibitem[Pacchiano et~al., 2020a]{pacchiano2020regret}
Pacchiano, A., Dann, C., Gentile, C., and Bartlett, P. (2020a).
\newblock Regret bound balancing and elimination for model selection in bandits
  and rl.
\newblock {\em arXiv preprint arXiv:2012.13045}.

\bibitem[Pacchiano et~al., 2020b]{pacchiano2020model}
Pacchiano, A., Phan, M., Abbasi-Yadkori, Y., Rao, A., Zimmert, J., Lattimore,
  T., and Szepesvari, C. (2020b).
\newblock Model selection in contextual stochastic bandit problems.
\newblock {\em arXiv preprint arXiv:2003.01704}.

\bibitem[Qian and Murphy, 2011]{qian2011performance}
Qian, M. and Murphy, S.~A. (2011).
\newblock Performance guarantees for individualized treatment rules.
\newblock {\em Annals of statistics}, 39(2):1180.

\bibitem[Robinson, 2023]{simulation_dimret}
Robinson, S. (2023).
\newblock Exploring the relationship between simulation model accuracy and
  complexity.
\newblock {\em Journal of the Operational Research Society}, 74(9):1992--2011.

\bibitem[Shalev-Shwartz and Ben-David, 2014]{shalev2014understanding}
Shalev-Shwartz, S. and Ben-David, S. (2014).
\newblock {\em Understanding machine learning: From theory to algorithms}.
\newblock Cambridge university press.

\bibitem[Simchi-Levi and Xu, 2020]{simchi2020bypassing}
Simchi-Levi, D. and Xu, Y. (2020).
\newblock Bypassing the monster: A faster and simpler optimal algorithm for
  contextual bandits under realizability.
\newblock {\em Available at SSRN}.

\bibitem[Tewari and Murphy, 2017]{tewari2017ads}
Tewari, A. and Murphy, S.~A. (2017).
\newblock From ads to interventions: Contextual bandits in mobile health.
\newblock In {\em Mobile Health}, pages 495--517. Springer.

\bibitem[Wainwright, 2019]{wainwright2019high}
Wainwright, M.~J. (2019).
\newblock {\em High-dimensional statistics: A non-asymptotic viewpoint},
  volume~48.
\newblock Cambridge University Press.

\end{thebibliography}

\onecolumn
\appendix

\section{ADDITIONAL PRELIMINARIES}
\label{app:additional-preliminaries}

We now work towards proving \Cref{thm:main-theorem}. We start with providing a proof outline, set up additional notation in \Cref{app:additional-notation}, and aggregate commonly used notation in \Cref{tab:notation}. 

\textbf{Proof outline:} \Cref{app:properties-action-selection-kernel} provides basic well-known properties of the IGW exploration strategy, after which we get into the meat of our proof. \Cref{app:direct-method-guarantees} provides bounds on accuracy of the estimated reward model in evaluating policies via the direct method within various safe epochs (this argument is similar to the one provided in \cite{krishnamurthy2021adapting}). \Cref{app:main-policy-based-test} designs and analyzes our main policy-based misspecification test, in order to test and verify the direct method implications of the estimated reward model.\footnote{Since safe epochs, which depend on model class bias, are unknown, we must test and verify these direct method bounds via our policy-based misspecification test.} \Cref{app:reward-model-agreement} adds additional tests to ensure sufficient reward model agreement across epochs and classes. \Cref{app:inductive-argument-with-verified-bounds} then bounds the true regret of various policies with regret according to estimated models via an inductive argument. \Cref{app:bounding-time-to-misspecification} upper bounds the time to detect misspecification for various classes under our assumptions. Finally, \Cref{app:final-regret-guarantees} uses these results to prove \Cref{thm:main-theorem}. Additional details of interest are discussed in \Cref{app:additional-details}.

\begin{table}[h]
    \centering
    \begin{tabular}{ll}
    \toprule
    Symbol & Description \\
    \midrule
        $i_{m}$ & simplest possibly-well-specified model class\\
        $h,i,j$ & class indices such that $h\leq i\leq j$\\
        $\tilde{d}_i$ & complexity of model class $\mathcal{\Tilde{F}}_i$\\
        $\xi_i$ & estimation rate for model class $\allF_i$, defined in \eqref{eq:common-rate}\\
        $\gamma_{m,i}$ & exploration parameter for model class $\Tilde{\mathcal{F}}_i$ in epoch $m$, defined in \eqref{eq:model-gamma}\\
        $m$ & epoch index\\
        $m_i^*$ & ``safe epoch" for model class $i$, up to which sufficient exploration is guaranteed, defined in \eqref{eq:msafe-2}\\
        $\msafealg_i$ & implicit estimate of safe epoch for model class $i$, defined in \Cref{def:MTest},\\ &where $\msafealg_i=0$ for $i\leq0$\\
        $\tau_m$ & final round of epoch $m$\\
        $p_m$ & exploration policy for epoch $m$ \\
        $V(p,\pi)$ & expected inverse probability weight, defined in \eqref{eq:decisional-divergence}\\
        $f$ & model mapping from contexts and actions to rewards\\
        $f^*$ & true conditional expectation reward function\\
        $\hatf_m\in\mathcal{\Tilde{F}}$ & estimated reward model according to EstOracle$_{M}(S_{m-1,tr})$, fitted over classes $\mathcal{\Tilde{F}}_1,...,\mathcal{\Tilde{F}}_M$\\
        $\hatf_{m,i}\in\mathcal{\Tilde{F}}$ & estimated reward model according to EstOracle$_{i}(S_{m-1,tr})$, fitted over classes $\mathcal{\Tilde{F}}_1,...,\mathcal{\Tilde{F}}_i$\\
        $\pi_f$ & optimal policy with respect to reward model $f$\\
        $R_f(\pi)$ & expected reward of policy $\pi$ with respect to reward model $f$\\
        $\hat{R}_{m+1}(\pi)$ & estimated reward of policy $\pi$ according to EvalOracle($S_m$)\\
        $\hat{R}_{m+1,f}(\pi)$ & implicit estimated reward of policy $\pi$, defined in \eqref{eq:implicit-policy-estimation}\\
        $\Reg_f(\pi)$ & regret of policy $\pi$ with respect to reward model $f$: $\Reg_f(\pi)=R_f(\pi_f)-R_f(\pi)$\\
        $\Reg_i(\pi)$ & true expected instantaneous regret against the best policy in the class $\Tilde{\Pi}_i$ \\
        $\hatReg_{m,f}(\pi)$ & empirical regret of policy $\pi$ with respect to model $f$, defined in \eqref{eq:emp_regret} \\
        $S_m$ & data collected in epoch $m$\\
        MTOracle & Misspecification test\\
        EstOracle$_i$ & Estimation oracle over model classes in $[i]$\\
        EvalOracle & Policy evaluation oracle\\
        Mod-IGW & Model selection with inverse gap weighting\\
        \bottomrule
    \end{tabular}
    \caption{Table of notations}
    \label{tab:notation}
\end{table}

\subsection{Additional Notation}
\label{app:additional-notation}

The most commonly used notations in this paper are collected in \Cref{tab:notation}. Let $\Gamma_{t}$ denote the set of observed data points up to and including time $t$. That is
\begin{align}
    \label{eq:history}
    \Gamma_{t} := \{(x_s, a_s, r_s(a_s))\}_{s=1}^{t}
\end{align}

For any randomized policy $p$ and any policy $\pi$, we let $V(p,\pi)$ denote the expected inverse probability weight of covering $\pi$ under $p$:
\begin{equation}
    \label{eq:decisional-divergence}
    V(p,\pi):=\E_{x\sim D_{\Xscript},a\sim\pi(x)}\bigg[\frac{\pi(a|x)}{p(a|x)}\bigg].
\end{equation}
The variance term for several policy evaluation estimators like IPW depends on this expected inverse probability weight \citep[see e.g.][]{agarwal2014taming}. We also let let $m(t)$ denote the epoch containing round $t$, so that $m(t):=\min \{m| t \leq \tau_m \}$.

\subsection{Properties of the IGW Exploration Policy}
\label{app:properties-action-selection-kernel}

We now state helpful properties of the exploration policy, and only include the proofs for completeness. Similar properties are explicitly stated and proven in \cite{simchi2020bypassing}, but also show up in the analysis for \cite{foster2020beyond} (see section B.1 of their paper). Arguably, these properties characterize the key features of inverse gap weighting algorithms. \Cref{lem:QmRegEst} and \Cref{lem:boundV} bound the estimated instantaneous regret and the expected inverse probability weight for the exploration policy constructed by inverse gap weighting.

\begin{restatable}[]{lem}{lemQmRegEst}
\label{lem:QmRegEst}
For any epoch $m\geq 1$, we have:
$$ \Reg_{\hatf_m}(p_m) \leq \frac{K}{\gamma_m}. $$
\end{restatable}
\begin{proof}
Note that:
\begin{align*}
    \Reg_{\hatf_m}(p_m) &= \E_{x\sim D_{\Xscript}}\Big[ \sum_{a\in\A} p_m(a|x) \Big(\hatf_m(x,\pi_{\hatf_m}(x)) - \hatf_m(x,a)\Big) \Big]\\
    & = \E_{x\sim D_{\Xscript}}\Bigg[ \sum_{a\in\A} \frac{\Big(\hatf_m(x,\pi_{\hatf_m}(x)) - \hatf_m(x,a)\Big)}{K+\gamma_m\Big(\hatf_m(x,\pi_{\hatf_m}(x)) - \hatf_m(x,a)\Big)} \Bigg] \leq \frac{K}{\gamma_m}.
\end{align*}
\end{proof}

\begin{restatable}[]{lem}{lemboundV}
\label{lem:boundV}
For all policy $\pi$ and epochs $m \geq 1$, we have:
\begin{align*}
    &V(p_m,\pi) \leq K + \gamma_m \E_{x\sim D_{\Xscript}}\Big[\big(\hatf_{m}(x,\pi_{\hatf_{m}}(x))- \hatf_{m}(x,\pi(x))\big) \Big].
\end{align*}
\end{restatable}
\begin{proof}
	Consider any policy $\pi$ and epoch $m\geq 1$. For any context $x\in\Xscript$ and action $a\in \A \setminus \{\pi_{\hatf_{m}}(x)\}$, from our choice for $p_m$, we get:
	$$
	\frac{1}{p_m(a|x)}= K+\gamma_m(\hatf_m(x,\pi_{\hatf_{m}}(x)) - \hatf_m(x,a)).
	$$
	For the action $a= \pi_{\hatf_{m}}(x)$, we have:
	\begin{align*}
	\frac{1}{p_m(a|x)} = \frac{1}{1- \sum_{a'\neq a}\frac{1}{K+\gamma_m \big(\hatf_{m}(x,\pi_{\hatf_{m}}(x))- \hatf_{m}(x,a')\big) }} \leq K
	\end{align*}
	In particular, putting the above inequality together, we get:
	\begin{align*}
	\frac{\pi(a|x)}{p_m(a|x)} \leq \frac{1}{p_m(a|x)} \leq K + \gamma_m \Big[\big(\hatf_{m}(x,\pi_{\hatf_{m}}(x))- \hatf_{m}(x,a)\big) \Big].
	\end{align*}
	The lemma now follows by taking expectation over $x\sim D_{\Xscript}, a\sim \pi(x)$.
\end{proof}

\section{DIRECT METHOD GUARANTEES FOR ESTIMATED MODELS}
\label{app:direct-method-guarantees}

We judge our estimated reward model by its ability to evaluate policies via the direct method. In this section we prove direct method bounds that should hold with high-probability up to various safe epochs.

\subsection{High Probability Events For Regression}
\label{app:high-probability}

In this section, we define an event $\eventReg$ that holds with high probability under \Cref{ass:model-selection}. At a high level, $\eventReg$ defines the event where the prediction guarantees of $\MSOracle$ hold. That is, this event bounds the expected squared error difference between the true model ($f^*$) and the estimated model ($\hatf_{m+1}$). %
\begin{equation}
    \label{eq:w-event-Reg}
    \begin{aligned}
      \eventReg := \Bigg\{ &\forall m \in [m^*_i], j\in [i,M'], \\ 
      &\E_{x\sim D_{\Xscript}}\E_{a\sim p_m(\cdot|x)}[(\hatf_{m+1,j}(x,a)-f^*(x,a))^2] \leq 2\xi_i\Big(\frac{\tau_m-\tau_{m-1}}{2}, \delta/(6TM'^2) \Big) \Bigg\}.
    \end{aligned}
\end{equation}
In \Cref{lem:high-prob-eventReg}, we use standard union bound arguments to show that the event $\eventReg$ holds with high probability.

\begin{lem}
\label{lem:high-prob-eventReg}
Suppose \Cref{ass:model-selection} holds. Then the event $\eventReg$ holds with probability at least $1-\delta/2$.
\end{lem}
\begin{proof}
Consider any epoch $m$. Note that, conditional on $\Gamma_{\tau_{m-1}}$ the number of samples in epoch $m$ are fixed and these samples are i.i.d. from the distribution $D(p_m)$.\footnote{$D(p_m)$ depends on $\Gamma_{\tau_{m-1}}$ because $p_m$ is constructed using the data in $\Gamma_{\tau_{m-1}}$.} Hence with probability $1-\delta/(6TM')$, from \Cref{ass:model-selection}, for all $i\in[M']$ and $j\in[i,M']$ such that $m\in[\msafe_i]$ we have:
\begin{equation}
    \label{eq:mse_event}
    \begin{aligned}
    \E_{x\sim D_{\Xscript}}\E_{a\sim p_m(\cdot|x)}[(\hatf_{m+1,j}(x,a)-f^*(x,a))^2] &\leq \min_{i'\in[j]} (C_0B_{i'} + \xi_{i'}(\frac{\tau_m-\tau_{m-1}}{2},\delta/(6TM'^2)))\\
    &\leq C_0B_{i} + \xi_{i}(\frac{\tau_m-\tau_{m-1}}{2},\delta/(6TM'^2))\\
    &\leq 2\xi_{i}(\frac{\tau_m-\tau_{m-1}}{2},\delta/(6TM'^2)),
    \end{aligned}
\end{equation}
where the last inequality follows from the definition of $\msafe_i$ and the fact that $m\leq\msafe_i$. Therefore, the probability that $\eventReg$ does not hold can be bounded by:
$$ \sum_{m=1}^{m(T)}\frac{\delta}{6TM'} \leq \frac{\delta}{6M'} \leq \frac{\delta}{2}. $$
\end{proof}

\subsection{Direct Method for Policy Optimization}
\label{app:direct-method}

Given any estimated model $\hatf$, $R_{\hatf}(\pi)$ gives us an implicit estimate for any policy $\pi$. Moreover, as discussed earlier, $\pi_{\hatf}$ is the policy that maximizes these implicitly estimated rewards. This approach to policy optimization is known as the direct method for policy optimization. Several papers have analyzed the direct method for policy evaluation \citep[e.g.][]{qian2011performance,simchi2020bypassing,krishnamurthy2021adapting}. 

In \Cref{lem:reg-est-accuracy} we state a guarantee on the direct method via a model selection oracle for estimation. The proof is essentially the same as the proof of prior guarantees on the direct method.

\begin{restatable}{lem}{lemImpPolEval}
\label{lem:reg-est-accuracy}
Suppose the event $\eventReg$ defined in \eqref{eq:w-event-Reg} holds. Then, for all policies $\pi$, class indices $i\in[M']$ and $j\in [i,M']$, $\alpha>0$, and epochs $m\in [\msafe_i]$, we have:
\begin{align*}
    |R_{\hatf_{m+1,j}}(\pi)-R(\pi)| \leq \bigg(\frac{1}{\alpha} + \frac{\alpha}{16} \bigg)\frac{K}{\gamma_{m+1,i}} + \frac{\gamma_m}{\alpha \gamma_{m+1,i}}\Reg_{\hatf_m}(\pi).
\end{align*}
\end{restatable}
\begin{proof}
For any policy $\pi$, class indices $i\in[M']$ and $j\in [i,M']$, $\alpha>0$, and epochs $m\in [\msafe_i]$, note that:
\begin{align*}
    &|R_{\hat{f}_{m+1,j}}(q)-R(q)|\\ 
    =&\big| \E_{x\sim D_{\Xscript}, a\sim q}[\hatf_{m+1,j}(x,a)-f^*(x,a)]\big|\\
  \stackrel{(i)}{=}&\bigg| \E_{x\sim D_{\Xscript}, a\sim p_m}\Big[\frac{q(a|x)}{p_m(a|x)}\big(\hatf_{m+1,j}(x,a)-f^*(x,a)\big)\Big]\bigg|\\
  \stackrel{(ii)}{\leq} & \E_{x\sim D_{\Xscript}, a\sim p_m}\Big[\frac{q(a|x)}{p_m(a|x)}\big|\hatf_{m+1,j}(x,a)-f^*(x,a)\big|\Big]\\
    =& \E_{x\sim D_{\Xscript}, a\sim p_m}\Big[\sqrt{\Big(\frac{q(a|x)}{p_m(a|x)}\Big)^2\big|\hatf_{m+1,j}(x,a)-f^*(x,a)\big|^2}\Big]\\
  \stackrel{(iii)}{\leq}&\sqrt{\E_{x\sim D_{\Xscript}, a\sim p_m}\Big[\Big(
  \frac{q(a|x)}{p_m(a|x)}\Big)^2\Big]}\sqrt{\E_{x\sim D_{\Xscript}, a\sim p_m}\big[ (\hatf_{m+1,j}(x,a)-f^*(x,a))^2\big] }\\
  \stackrel{(iv)}{=}&\sqrt{\E_{x\sim D_{\Xscript}, a\sim q}\Big[
  \frac{q(a|x)}{p_m(a|x)}\Big]} \sqrt{\E_{x\sim D_{\Xscript}, a\sim p_m}\big[ (\hatf_{m+1,j}(x,a)-f^*(x,a))^2\big] }\\
    \stackrel{(v)}{\leq} & \sqrt{V(p_m, q)} \sqrt{2\xi_i\Big(\frac{\tau_m-\tau_{m-1}}{2}, \delta/(6TM'^2) \Big)} = \frac{\sqrt{V(p_m,q)}\sqrt{K}}{2\gamma_{m+1,i}}\\
    \stackrel{(vi)}{\leq} & \frac{V(p_m,q)}{\alpha\gamma_{m+1,i}} + \frac{\alpha K}{16\gamma_{m+1,i}}\\
    \stackrel{(vii)}{\leq}&  \frac{K+\gamma_m\Reg_{\hatf_m}(q)}{\alpha\gamma_{m+1,i}} + \frac{\alpha K}{16\gamma_{m+1,i}} = \bigg(\frac{1}{\alpha} + \frac{\alpha}{16} \bigg)\frac{K}{\gamma_{m+1,i}} + \frac{\gamma_m}{\alpha \gamma_{m+1,i}}\Reg_{\hatf_m}(q),
\end{align*}
where (i) and (iv) follow from change of measure arguments, (ii) follows from Jenson's inequality, (iii) follows from Cauchy-Schwartz inequality, (v) follows from $\eventReg$, (vi) follows from AM-GM inequality, and (vii) follows from \Cref{lem:boundV}.
\end{proof}

The accuracy of the direct method for policy evaluation only depends on the prediction error of the underlying estimator. We therefore note that when the underlying estimator is constructed by a model selection oracle for estimation, the prediction error will decrease more rapidly in terms of sample size for small datasets. This allows us to accordingly increase the corresponding exploitation parameters more rapidly for earlier rounds.

\section{POLICY-BASED MISSPECIFICATION TEST}
\label{app:main-policy-based-test}
In this section, we establish the foundation for the main misspecification test. By the definition of $\msafealg_i$, none of the tests corresponding to class $i$ fail until this epoch.

\textbf{Proof outline:} \Cref{sec:eventW2} provides a high-probability event for policy evaluation that holds under \Cref{ass:policy-val-estimation,ass:policy-val-evaluation}. Moving forward, all our analysis relies on the high-probability events defined so far ($\eventReg,\eventExpEval$). \Cref{app:policy-evaluation-with-oracles} provides refined policy evaluation guarantees. \Cref{sec:misspec_test1} develops the policy-based misspecification test ($\MisspecificationOracle$), and provides validated guarantees for the direct method estimates.

\subsection{High Probability Events For Explicit Policy Evaluation}\label{sec:eventW2}

In this section, we define an event $\eventExpEval$ that holds with high-probability under \Cref{ass:policy-val-estimation} and \Cref{ass:policy-val-evaluation}. At a high-level, $\eventExpEval$ defines the event where the evaluation guarantees of consistent (e.g. IPS/DR) and direct method policy estimates hold.
\begin{equation}
    \label{eq:w-event-ExpEval}
    \begin{aligned}
        &\eventExpEval := \Bigg\{ \forall m, \; \forall i, j \in [M'], \; \forall \; \pi\in\Tilde{\Pi}_i\cup\{\pi_{\hatf_{m+1}},p_{m+1,1},\dots,p_{m+1,M'} \},\\  
        & \; |\hatR_{m+1}(\pi) - R(\pi)| \leq \sqrt{V(p_m,\pi)\xi_i\Big(\frac{\tau_m-\tau_{m-1}}{2}, \frac{\delta}{6TM'^2} \Big)} + 2\gamma_m\xi_i\Big(\frac{\tau_m-\tau_{m-1}}{2}, \frac{\delta}{6TM'^2} \Big), \\
        & \; |\hatR_{m+1,\hatf_{m+1}}(\pi) - R_{\hatf_{m+1}}(\pi)| \leq \sqrt{\xi_i\Big(\frac{\tau_m-\tau_{m-1}}{2}, \frac{\delta}{6TM'^2} \Big)}, \\
        & \; |\hatR_{m+1,\hatf_{m+1,j}}(\pi) - R_{\hatf_{m+1,j}}(\pi)| \leq \sqrt{\xi_i\Big(\frac{\tau_m-\tau_{m-1}}{2}, \frac{\delta}{6TM'^2}\Big)} \Bigg\}.
    \end{aligned}
\end{equation}
In \Cref{lem:high-prob-eventPolicyVal}, we use standard union bound arguments to show that the event $\eventExpEval$ holds with high-probability. 

\begin{lem}
\label{lem:high-prob-eventPolicyVal}
Suppose \Cref{ass:policy-val-estimation} and \Cref{ass:policy-val-evaluation} hold. The event $\eventExpEval$ holds with probability at least $1-\delta/2$.
\end{lem}
\begin{proof}
Consider any epoch $m$. Note that, conditional on $\Gamma_{\tau_{m-1}}$ the number of samples in epoch $m$ are fixed and these samples are i.i.d. from the distribution $D(p_m)$. Consider any pair of model indices $i,j$. Hence with probability $1-3\frac{\delta}{6TM'^2}$, from \Cref{ass:policy-val-estimation} and \Cref{ass:policy-val-evaluation}, for all policies $\pi\in\Tilde{\Pi}_i\cup\{\pi_{\hatf_{m+1}},p_{m+1,1},\dots,p_{m+1,M'} \}$ we have:\footnote{Note that $K \leq \gamma_m$, hence $p_m(\cdot|\cdot)\geq 1/(K+\gamma_m) \geq 1/(2\gamma_m)$.}
\begin{align*}
    & \; |\hatR_{m+1}(\pi) - R(\pi)| \leq \sqrt{V(p_m,\pi)\xi_i\Big(\frac{\tau_m-\tau_{m-1}}{2}, \frac{\delta}{6TM'^2} \Big)} + 2\gamma_m\xi_i\Big(\frac{\tau_m-\tau_{m-1}}{2}, \frac{\delta}{6TM'^2} \Big), \\
    & \; |\hatR_{m+1,\hatf_{m+1}}(\pi) - R_{\hatf_{m+1}}(\pi)| \leq \sqrt{\xi_i\Big(\frac{\tau_m-\tau_{m-1}}{2}, \frac{\delta}{6TM'^2} \Big)}, \\
    & \; |\hatR_{m+1,\hatf_{m+1,j}}(\pi) - R_{\hatf_{m+1,j}}(\pi)| \leq \sqrt{\xi_i\Big(\frac{\tau_m-\tau_{m-1}}{2}, \frac{\delta}{6TM'^2}\Big)}.
\end{align*}
Hence, $\eventExpEval$ holds with probability at least:
$$ 1 - \sum_{i=1}^M\sum_{j=1}^M\sum_{m=1}^{m(T)} \frac{3\delta}{6TM'^2} \geq 1 - \delta/2.$$
\end{proof}

\subsection{Policy Evaluation}
\label{app:policy-evaluation-with-oracles}
In this section, we bound the error of $\hatR_{m+1}(\pi)$, the estimate of a policy value obtained via $\EvalOracle$.

\begin{restatable}{lem}{lemExpPolEval}
\label{lem:policy-eval-accuracy}
Suppose the event $\eventExpEval$ defined in \eqref{eq:w-event-ExpEval} holds. Then, for all class indices $i\in[M']$, policies $\pi\in\Tilde{\Pi}_i\cup\{\pi_{\hatf_{m+1}},p_{m+1,1},\dots,p_{m+1,M'} \}$, $\alpha>0$, and epochs $m\geq 1$, we have:
\begin{align*}
    |\hatR_{m+1}(\pi) - R(\pi)| \leq \bigg(\frac{1}{\alpha} + \frac{\alpha}{16} + \frac{2\gamma_m}{\gamma_{m+1,i}} \bigg)\frac{K}{\gamma_{m+1,i}} + \frac{\gamma_m}{\alpha \gamma_{m+1,i}}\Reg_{\hatf_m}(\pi).
\end{align*}
\end{restatable}
\begin{proof}
For any class index $i$, policy $\pi\in \Tilde{\Pi}_i\cup\{\pi_{\hatf_{m+1}},p_{m+1,1},\dots,p_{m+1,M'} \}$, and epoch $m\geq 1$, we have:
\begin{align*}
    &|\hatR_{m+1}(\pi) - R(\pi)|\\
    \stackrel{(i)}{\leq} & \sqrt{V(p_m,\pi)\xi_i\Big(\frac{\tau_m-\tau_{m-1}}{2}, \delta/(6TM'^2) \Big)} + 2\gamma_m\xi_i\Big(\frac{\tau_m-\tau_{m-1}}{2}, \delta/(6TM'^2) \Big)  \\
    \stackrel{(ii)}{\leq} & \frac{\sqrt{V(p_m,\pi)}\sqrt{K}}{2\gamma_{m+1,i}} + \frac{2\gamma_m}{\gamma_{m+1,i}}\frac{K}{\gamma_{m+1,i}} \\
    \stackrel{(iii)}{\leq} & \frac{V(p_m,\pi)}{\alpha\gamma_{m+1,i}} + \frac{\alpha K}{16\gamma_{m+1,i}} + \frac{2\gamma_m}{\gamma_{m+1,i}}\frac{K}{\gamma_{m+1,i}}\\
    \stackrel{(iv)}{\leq} & \bigg(\frac{1}{\alpha} + \frac{\alpha}{16} + \frac{2\gamma_m}{\gamma_{m+1,i}} \bigg)\frac{K}{\gamma_{m+1,i}} + \frac{\gamma_m}{\alpha \gamma_{m+1,i}}\Reg_{\hatf_m}(\pi),
\end{align*}
where (i) follows from $\eventExpEval$, (ii) follows from the definition of $\gamma_{m+1,i}$, (iii) follows from the AM-GM inequality for any $\alpha>0$, and (iv) follows from \Cref{lem:boundV}.
\end{proof}

\subsection{Validating Direct Method Estimates}\label{sec:misspec_test1}
In this section, we design the main policy-based misspecification test and provide the implied guarantees when some conditions of the test hold. In \Cref{lem:explicit-validation} we develop the empirical test that must hold through epoch $\msafe_i$. The implications of this test are captured in \Cref{lem:combinedprevious}, which provides guarantees through $\msafealg_i$ (by definition, the test corresponding to class $i$ is satisfied until $\msafealg_i$).

\begin{restatable}{lem}{explicitvalidation}
\label{lem:explicit-validation}
Suppose $\eventReg$ and $\eventExpEval$ hold. Consider any pair of class indices $i,j\in[M']$ such that $j\geq i$, any epoch $m\in[\msafe_i]$, $\alpha>0$, and model $f\in\{\hatf_{m+1}, \hatf_{m+1,i}\}$. Let $\theta_{i,j}:=\frac{\gamma_{m+1,i}}{\gamma_{m+1,j}}$. Then for any policy $\pi\in\allPi_{j}\cup\{\pi_{\hatf_{m+1}},p_{m+1,1},\dots,p_{m+1,M'} \}$, we have:
\begin{equation}
\begin{aligned}
    &|\hatR_{m+1}(\pi) - \hatR_{m+1,f}(\pi)| \\
    &\quad\leq \bigg( \frac{1+\theta_{i,j}}{\alpha} + \frac{(1+\theta_{i,j})\alpha}{16} + \frac{(2\theta_{i,j}^2+(1+\theta_{i,j})^2/\alpha)\gamma_m}{\gamma_{m+1,i}} + \theta_{i,j} \bigg)\frac{K}{\gamma_{m+1,i}} \\
    &\quad\quad + \frac{(1+\theta_{i,j})\gamma_m}{\alpha \gamma_{m+1,i}}\hatReg_{m+1,\hatf_m}(\pi). 
\end{aligned}
\end{equation}
\end{restatable}
\begin{proof}
Consider any pair of class indices $i,j\in[M']$ such that $j\geq i$, any epoch $m\in[\msafe_i]$, $\alpha>0$, and model $f\in\{\hatf_{m+1}, \hatf_{m+1,i}\}$.  For any policy $\pi\in\Tilde{\Pi}_{j}\cup\{\pi_{\hatf_{m+1}},p_{m+1,1},\dots,p_{m+1,M'} \}$, we have:
\begin{align*}
    &|\hatR_{m+1}(\pi) - \hatR_{m+1,f}(\pi)|\\
    \stackrel{(i)}{\leq} & |\hatR_{m+1}(\pi) - R(\pi)| + |R(\pi) - R_f(\pi)| + |R_f(\pi) - \hatR_{m+1,f}(\pi)|\\
    \stackrel{(ii)}{\leq} & \bigg( \frac{1}{\alpha} + \frac{\alpha}{16} + \frac{2\gamma_m}{\gamma_{m+1,j}} \bigg)\frac{K}{\gamma_{m+1,j}} + \frac{\gamma_m}{\alpha \gamma_{m+1,j}}\Reg_{\hatf_m}(\pi) \\ 
    & + \bigg( \frac{1}{\alpha} + \frac{\alpha}{16}\bigg)\frac{K}{\gamma_{m+1,i}} + \frac{\gamma_m}{\alpha \gamma_{m+1,i}}\Reg_{\hatf_m}(\pi)+ \frac{K}{\gamma_{m+1,j}}\\
    \stackrel{(iii)}{\leq} & \bigg( \frac{1+\theta_{i,j}}{\alpha} + \frac{(1+\theta_{i,j})\alpha}{16} + \frac{2\theta_{i,j}^2\gamma_m}{\gamma_{m+1,i}} + \theta_{i,j}  \bigg)\frac{K}{\gamma_{m+1,i}} + \frac{(1+\theta_{i,j})\gamma_m}{\alpha \gamma_{m+1,i}}\Reg_{\hatf_m}(\pi),
\end{align*}
where (i) is an application of triangle inequality, and (ii) follows from \Cref{lem:reg-est-accuracy}, \Cref{lem:policy-eval-accuracy}, and events $\eventReg$ and $\eventExpEval$. Then (iii) follows from applying the definition of parameter $\theta_{i,j}$. \Cref{lem:explicit-validation} now follows from noting that:
\begin{align*}
    & \frac{(1+\theta_{i,j})\gamma_m}{\alpha \gamma_{m+1,i}}\Reg_{\hatf_m}(\pi)\\
    &\quad \stackrel{(i)}{\leq}  \frac{(1+\theta_{i,j})\gamma_m}{\alpha \gamma_{m+1,i}}\bigg(\hatReg_{m+1,\hatf_m}(\pi) + |R_{\hatf_m}(\pi_{\hatf_m}) - \hatR_{m+1,\hatf_m}(\pi_{\hatf_m})| + |R_{\hatf_m}(\pi) - \hatR_{m+1,\hatf_m}(\pi)| \bigg)\\
    &\quad \stackrel{(ii)}{\leq} \frac{(1+\theta_{i,j})\gamma_m}{\alpha \gamma_{m+1,i}}\bigg(\hatReg_{m+1,\hatf_m}(\pi) + (1+\theta_{i,j})\frac{K}{\gamma_{m+1,i}} \bigg),
\end{align*}
where (i) is an application of triangle inequality, and (ii) follows from $\eventExpEval$.

\end{proof}

\begin{lem}
\label{lem:verified-bounds}
Suppose $\eventReg$ and $\eventExpEval$ hold. Consider any pair of class indices $i,j\in[M']$ such that $j\geq i$,  any epoch $m\in[\msafealg_{i}]$, $\alpha>0$, and model $f\in\{\hatf_{m+1}, \hatf_{m+1,i}\}$. Then for any policy $\pi\in\Tilde{\Pi}_{j}\cup\{\pi_{\hatf_{m+1}},p_{m+1,1},\dots,p_{m+1,M'} \}$, we have:
\begin{align*}
    &|R_{f}(\pi) - R(\pi)| - |\hatR_{m+1,f}(\pi) - \hatR_{m+1}(\pi)| \\
    &\leq \bigg( \frac{1}{\alpha} + \frac{\alpha}{16} + \frac{2\gamma_m}{\gamma_{m+1,j}} + 1 \bigg)\frac{K}{\gamma_{m+1,j}} + \frac{\gamma_m}{\alpha \gamma_{m+1,j}}\Reg_{\hatf_m}(\pi). 
\end{align*}
\end{lem}
\begin{proof}
Consider any pair of class indices $i,j\in[M']$ such that $j\geq i$, any epoch $m\in[\msafealg_{i}]$, $\alpha>0$, and model $f\in\{\hatf_{m+1}, \hatf_{m+1,i}\}$. Then for any policy $\pi\in\Tilde{\Pi}_{j}\cup\{\pi_{\hatf_{m+1}},p_{m+1,1},\dots,p_{m+1,M'} \}$, we have:
\begin{align*}
    &|R_{f}(\pi) - R(\pi)| - |\hatR_{m+1,f}(\pi) - \hatR_{m+1}(\pi)|\\
    &\quad\stackrel{(i)}{\leq}  |\hatR_{m+1}(\pi) - R(\pi)| + |R_f(\pi) - \hatR_{m+1,f}(\pi)|\\
    &\quad\stackrel{(ii)}{\leq}  \bigg( \frac{1}{\alpha} + \frac{\alpha}{16} + \frac{2\gamma_m}{\gamma_{m+1,j}} + 1 \bigg)\frac{K}{\gamma_{m+1,j}} + \frac{\gamma_m}{\alpha \gamma_{m+1,j}}\Reg_{\hatf_m}(\pi),
\end{align*}
where (i) is an application of triangle inequality, and (ii) follows from \Cref{lem:policy-eval-accuracy} and  $\eventExpEval$ .
\end{proof}

\begin{lem}
\label{lem:combinedprevious}
Suppose $\eventReg$ and $\eventExpEval$ hold. Consider any pair of class indices $i,j\in[M']$ such that $j\geq i$, any epoch $m\in[\msafealg_{i}]$, $\alpha>0$, and model $f\in\{\hatf_{m+1}, \hatf_{m+1,i}\}$. Let $\theta_{i,j}:=\frac{\gamma_{m+1,i}}{\gamma_{m+1,j}}$. Then for any policy $\pi\in\Tilde{\Pi}_{j}\cup\{\pi_{\hatf_{m+1}},p_{m+1,1},\dots,p_{m+1,M'} \}$, we have:
\begin{align*}
    |R_{f}(\pi) - R(\pi)| \leq & \bigg( \frac{1+2\theta_{i,j}}{\alpha} + \frac{(1+2\theta_{i,j})\alpha}{16} + \frac{(2\theta_{i,j}^2+2(1+\theta_{i,j})^2/\alpha+2\theta_{i,j})\gamma_m}{\gamma_{m+1,i}} + 2\theta_{i,j} \bigg)\frac{K}{\gamma_{m+1,i}}\\ 
    &+ \frac{(1+2\theta_{i,j})\gamma_m}{\alpha \gamma_{m+1,i}}Reg_{\hatf_m}(\pi).
\end{align*}
\end{lem}
\begin{proof}
Consider any pair of class indices $i,j\in[M']$ such that $j\geq i$, any epoch $m\in[\msafealg_{i}]$, $\alpha>0$, and model $f\in\{\hatf_{m+1}, \hatf_{m+1,i}\}$. Then for any policy $\pi\in\Tilde{\Pi}_{j}\cup\{\pi_{\hatf_{m+1}},p_{m+1,1},\dots,p_{m+1,M'} \}$, we have from \Cref{lem:verified-bounds}:
\begin{align*}
    &|R_{f}(\pi) - R(\pi)| - |\hatR_{m+1,f}(\pi) - \hatR_{m+1}(\pi)| \\
    &\leq \bigg( \frac{1}{\alpha} + \frac{\alpha}{16} + \frac{2\gamma_m}{\gamma_{m+1,j}} + 1 \bigg)\frac{K}{\gamma_{m+1,j}} + \frac{\gamma_m}{\alpha \gamma_{m+1,j}}\Reg_{\hatf_m}(\pi). 
\end{align*}
From \Cref{lem:explicit-validation}, we know that for any class index $i\in[M]$, epoch $m\in[\msafe_i]$, model $f\in\{\hatf_{m+1}, \hatf_{m+1,i}\}$, and policy $\pi\in\Tilde{\Pi}_{j}\cup\{\pi_{\hatf_{m+1}},p_{m+1,1},\dots,p_{m+1,M'} \}$, we have:
\begin{align*}
    &|\hatR_{m+1}(\pi) - \hatR_{m+1,f}(\pi)| \\
    \leq &  \bigg( \frac{1+\theta_{i,j}}{\alpha} + \frac{(1+\theta_{i,j})\alpha}{16} + \frac{(2\theta_{i,j}^2+(1+\theta_{i,j})^2/\alpha)\gamma_m}{\gamma_{m+1,i}} + \theta_{i,j} \bigg)\frac{K}{\gamma_{m+1,i}} + \frac{(1+\theta_{i,j})\gamma_m}{\alpha \gamma_{m+1,i}}\hatReg_{m+1,\hatf_m}(\pi). 
\end{align*}
Combining the above results, we have:
\begin{align*}
    &|R_{f}(\pi) - R(\pi)| \\
    &\leq \bigg( \frac{1+\theta_{i,j}}{\alpha} + \frac{(1+\theta_{i,j})\alpha}{16} + \frac{(2\theta_{i,j}^2+(1+\theta_{i,j})^2/\alpha)\gamma_m}{\gamma_{m+1,i}} + \theta_{i,j} \bigg)\frac{K}{\gamma_{m+1,i}}+ \frac{(1+\theta_{i,j})\gamma_m}{\alpha \gamma_{m+1,i}}\hatReg_{m+1,\hatf_m}(\pi)\\
    &\quad  
    +\bigg( \frac{1}{\alpha} + \frac{\alpha}{16} + \frac{2\gamma_m}{\gamma_{m+1,j}} + 1 \bigg)\frac{K}{\gamma_{m+1,j}} + \frac{\gamma_m}{\alpha \gamma_{m+1,j}}\Reg_{\hatf_m}(\pi)\\
    &\stackrel{(i)}{\leq}\bigg( \frac{1+2\theta_{i,j}}{\alpha} + \frac{(1+2\theta_{i,j})\alpha}{16} + \frac{(2\theta_{i,j}^2+(1+\theta_{i,j})^2/\alpha+2\theta_{i,j})\gamma_m}{\gamma_{m+1,i}} + 2\theta_{i,j} \bigg)\frac{K}{\gamma_{m+1,i}}\\
    & \quad +\frac{(1+\theta_{i,j})\gamma_m}{\alpha\gamma_{m+1,i}}\hatReg_{m+1,\hatf_m}(\pi) + \frac{\theta_{i,j}\gamma_m}{\alpha\gamma_{m+1,i}}\Reg_{\hatf_m}(\pi),
\end{align*}
where (i) follows from plugging in $\frac{1}{\gamma_{m+1,j}}\leq\frac{\theta_{i,j}}{\gamma_{m+1,i}}$.
We can then combine the last two terms using the same approach used in the proof of \Cref{lem:explicit-validation}:
\begin{align*}
    &\frac{(1+\theta_{i,j})\gamma_m}{\alpha \gamma_{m+1,i}}\hatReg_{m+1,\hatf_m}(\pi)\\
    &\quad\stackrel{(i)}{\leq} \frac{(1+\theta_{i,j})\gamma_m}{\alpha \gamma_{m+1,i}}\bigg(\Reg_{\hatf_m}(\pi) + |R_{\hatf_m}(\pi_{\hatf_m}) - \hatR_{m+1,\hatf_m}(\pi_{\hatf_m})| + |R_{\hatf_m}(\pi) - \hatR_{m+1,\hatf_m}(\pi)| \bigg)\\
    &\quad \stackrel{(ii)}{\leq}  \frac{(1+\theta_{i,j})\gamma_m}{\alpha \gamma_{m+1,i}}\bigg(\Reg_{\hatf_m}(\pi) + (1+\theta_{i,j})\frac{K}{\gamma_{m+1,i}} \bigg).
\end{align*}
Here, (i) is an application of triangle inequality, and (ii) follows from $\eventExpEval$.
Applying this to our expression above gives the final form for \Cref{lem:combinedprevious}:
\begin{align*}
    &|R_{f}(\pi) - R(\pi)| \\
    &\leq\bigg( \frac{1+2\theta_{i,j}}{\alpha} + \frac{(1+2\theta_{i,j})\alpha}{16} + \frac{(2\theta_{i,j}^2+2(1+\theta_{i,j})^2/\alpha+2\theta_{i,j})\gamma_m}{\gamma_{m+1,i}} + 2\theta_{i,j} \bigg)\frac{K}{\gamma_{m+1,i}}\\
    &\quad +\frac{(1+2\theta_{i,j})\gamma_m}{\alpha \gamma_{m+1,i}}\Reg_{\hatf_m}(\pi).
\end{align*}
\end{proof}

\section{VERIFYING REWARD MODEL AGREEMENT}\label{app:reward-model-agreement}

In this section, we design the remainder of the misspecification test. In particular, we ensure agreement in reward models estimated across classes and epochs. We first prove an inductive result that allows us to relate the true regret to the regret according to estimated models. Then, in \Cref{sec:classgap}, we develop an empirical test for reward model agreement and prove verified guarantees that hold as long as the test doesn't fail. By the definition of $\msafealg_i$, none of the tests corresponding to class $i$ fail until this epoch.

\textbf{Proof outline:} \Cref{lem:inductive-step-0.1-mstar} proves our key inductive step that holds within safe epochs. \Cref{lem:policy-reg-bound,lem:policy-reg-bound-intermediate} use this inductive step to relate the true regret to the regret according to estimated models (within safe epochs). \Cref{lem:new_mtest} provides the expected reward model agreement across classes and describes the corresponding test. \Cref{lem:new_mtest_implication} provides implied guarantees as long this test holds. \Cref{lem:policy-reg-bound-wrt-old-model,lem:policy-reg-bound-wrt-old-model-test} provide the expected reward model agreement across epochs and describe the corresponding test. \Cref{lem:policy-reg-bound-wrt-old-model-test-implication} provides implied guarantees as long this test holds.

\begin{lem}
\label{lem:inductive-step-0.1-mstar}
Suppose the event $\eventReg$ defined in \eqref{eq:w-event-Reg} holds. Consider any class index $i\in[M']$ and consider any epoch $m\in[\msafe_i]$. Suppose there exists a constant ($\eta>0$) such that for all policies $\pi$, we have:
\begin{align*}
    \Reg(\pi) &\leq 2\Reg_{\hatf_{m}}(\pi) + \frac{\eta K}{\gamma_{m,i}}\\
    \Reg_{\hatf_{m}}(\pi) &\leq 2\Reg(\pi) + \frac{\eta K}{\gamma_{m,i}}.
\end{align*}
We then have that:
\begin{align*}
   \Reg(\pi) &\leq 2\Reg_{f}(\pi) + \frac{\eta' K}{\gamma_{m+1,i}},\quad \forall f\in\{\hatf_{m+1},\hatf_{m+1,i} \} \\
    \Reg_{\hatf_{m+1}}(\pi) &\leq 2\Reg(\pi) + \frac{\eta' K}{\gamma_{m+1,i}}.
\end{align*}
where $\eta'= 2\max\Big(\frac{\gamma_m}{\gamma_{m,i}},\sqrt{1+\frac{\gamma_m}{\gamma_{m,i}}\eta}\Big)$.
\end{lem}
\begin{proof}
Let $\alpha$ be any positive constant, and let $\alpha'=\gamma_{m}/\gamma_{m,i}$. Note that for any $f\in\{\hatf_{m+1},\hatf_{m+1,i}\}$, we have:
\begin{equation}
\label{eq:inductive-bound-for-pi-mstar}
\begin{aligned}
    \Reg(\pi) - \Reg_{f}(\pi)
    =& \Big(R(\pi^*) - R(\pi)\Big) - \Big(R_{f}(\pi_{f}) - R_{f}(\pi)\Big) \\
    \stackrel{(i)}{\leq} & \Big(R(\pi^*) - R(\pi)\Big) -  \Big(R_{f}(\pi^*) - R_{f}(\pi)\Big) \\
    \stackrel{(ii)}{\leq} & |R(\pi^*) - R_{f}(\pi^*)| + |R(\pi) - R_{f}(\pi))|\\
    \stackrel{(iii)}{\leq} & \bigg( \frac{2}{\alpha} + \frac{\alpha}{8} \bigg)\frac{K}{\gamma_{m+1,i}} + \frac{\gamma_m}{\alpha \gamma_{m+1,i}}\bigg(\Reg_{\hat{f}_m}(\pi) + \Reg_{\hat{f}_m}(\pi^*) \bigg),
\end{aligned}
\end{equation}
where (i) follows from the definition of $\pi_{f}$ for $f\in\{\hatf_{m+1},\hatf_{m+1,i}\}$, (ii) follows from the triangle inequality, and (iii) follows from \Cref{lem:reg-est-accuracy}. Now note that:
\begin{equation}
\label{eq:inductive-hyp-for-pi-mstar}
    \begin{aligned}
    &\frac{\gamma_m}{\alpha\gamma_{m+1,i}}\Reg_{\hatf_m}(\pi)  \stackrel{(i)}{\leq} \frac{\gamma_m}{\alpha\gamma_{m+1,i}}\Big( 2\Reg(\pi) + \frac{\eta K}{\gamma_{m,i}}\Big)  \leq \frac{2\alpha'}{\alpha}\Reg(\pi) + \frac{\alpha'\eta K}{\alpha\gamma_{m+1,i}},
    \end{aligned}
\end{equation}
where (i) follows from the conditions stated in \Cref{lem:inductive-step-0.1-mstar}. Similarly note that:
\begin{align}
\label{eq:inductive-hyp-for-opt-pi-mstar}
    \frac{\gamma_m}{\alpha\gamma_{m+1,i}}\Reg_{\hatf_m}(\pi^*)  \stackrel{(i)}{\leq} \frac{\gamma_m}{\alpha\gamma_{m+1,i}}\Big( 2\Reg(\pi^*) + \frac{\eta K}{\gamma_{m,i}} \Big)  \stackrel{(ii)}{=} \frac{\alpha'\eta K}{\alpha\gamma_{m+1,i}},
\end{align}
where (i) follows from the conditions stated in \Cref{lem:inductive-step-0.1-mstar}, and (ii) follows from the fact that $\Reg(\pi^*)=0$. Now from combining \eqref{eq:inductive-bound-for-pi-mstar}, \eqref{eq:inductive-hyp-for-pi-mstar}, and \eqref{eq:inductive-hyp-for-opt-pi-mstar}, we get:
\begin{equation}
\label{eq:IH-implication-mstar}
    \begin{aligned}
    \Reg(\pi) - \Reg_{f}(\pi) &\leq \bigg( \frac{2}{\alpha} + \frac{\alpha}{8} + \frac{2\alpha'\eta}{\alpha} \bigg)\frac{K}{\gamma_{m+1,i}} + \frac{2\alpha'}{\alpha}\Reg(\pi)\\
    \frac{\alpha-2\alpha'}{\alpha}\Reg(\pi) &\leq \Reg_{f}(\pi) + \bigg( \frac{2}{\alpha} + \frac{\alpha}{8} + \frac{2\alpha'\eta}{\alpha} \bigg)\frac{K}{\gamma_{m+1,i}}\\
    \Reg(\pi) &\leq \frac{\alpha}{\alpha-2\alpha'}\Reg_{f}(\pi) + \frac{\alpha}{\alpha-2\alpha'}\bigg( \frac{2}{\alpha} + \frac{\alpha}{8} + \frac{2\alpha'\eta}{\alpha} \bigg)\frac{K}{\gamma_{m+1,i}}.
    \end{aligned}
\end{equation}
Similar to \eqref{eq:inductive-bound-for-pi-mstar}, we get:
\begin{equation}
\label{eq:inductive-bound-for-pi-est-mstar}
\begin{aligned}
    \Reg_{\hatf_{m+1}}(\pi) - \Reg(\pi)
    =& \Big(R_{\hatf_{m+1}}(\pi_{\hatf_{m+1}}) - R_{\hatf_{m+1}}(\pi)\Big) - \Big(R(\pi^*) - R(\pi)\Big) \\
    \stackrel{(i)}{\leq} & \Big(R_{\hatf_{m+1}}(\pi_{\hatf_{m+1}}) - R_{\hatf_{m+1}}(\pi)\Big) - \Big(R(\pi_{\hatf_{m+1}}) - R(\pi)\Big) \\
    \stackrel{(ii)}{\leq} & |R(\pi_{\hatf_{m+1}}) - R_{\hatf_{m+1}}(\pi_{\hatf_{m+1}})| + |R(\pi) - R_{\hatf_{m+1}}(\pi))| \\
    \stackrel{(iii)}{\leq} & \bigg( \frac{2}{\alpha} + \frac{\alpha}{8} \bigg)\frac{K}{\gamma_{m+1,i}} + \frac{\gamma_m}{\alpha \gamma_{m+1,i}}\bigg(\Reg_{\hat{f}_{m}}(\pi) + \Reg_{\hat{f}_{m}}(\pi_{\hatf_{m+1}}) \bigg),
\end{aligned}
\end{equation}
where (i) follows from the definition of $\pi^*$, (ii) follows from the triangle inequality, and (iii) follows from \Cref{lem:reg-est-accuracy}. Similar to \eqref{eq:inductive-hyp-for-opt-pi-mstar}, we get:
\begin{equation}
\label{eq:inductive-hyp-for-mplusone-est-mstar}
    \begin{aligned}
        \frac{\gamma_m}{\alpha\gamma_{m+1,i}}\Reg_{\hatf_m}(\pi_{\hatf_{m+1}}) &\stackrel{(i)}{\leq} \frac{\gamma_m}{\alpha\gamma_{m+1,i}}\Big( 2\Reg(\pi_{\hatf_{m+1}}) + \frac{\eta K}{\gamma_{m,i}}  \Big) \\
        & \leq \frac{2\alpha'}{\alpha}\Reg(\pi_{\hat{f}_{m+1}}) + \frac{\alpha'\eta K}{\alpha\gamma_{m+1,i}}\\
        & \stackrel{(ii)}{\leq} \frac{2\alpha'}{\alpha}\Big( \frac{\alpha}{\alpha-2\alpha'}\bigg( \frac{2}{\alpha} + \frac{\alpha}{8} + \frac{2\alpha'\eta}{\alpha} \bigg)\frac{K}{\gamma_{m+1,i}}\Big) + \frac{\alpha'\eta K}{\alpha\gamma_{m+1,i}},
    \end{aligned}
\end{equation}
where (i) follows from the conditions stated in \Cref{lem:inductive-step-0.1-mstar}, and (ii) follows from \eqref{eq:IH-implication-mstar}. Combining  \eqref{eq:inductive-hyp-for-pi-mstar}, \eqref{eq:IH-implication-mstar}, \eqref{eq:inductive-bound-for-pi-est-mstar}, and \eqref{eq:inductive-hyp-for-mplusone-est-mstar}, we get:
\begin{equation}
\label{eq:IH-implication-est-mstar}
    \begin{aligned}
    \Reg_{\hatf_{m+1}}(\pi) - \Reg(\pi) &\leq \bigg( \frac{2}{\alpha} + \frac{\alpha}{8} + \frac{2\alpha'\eta}{\alpha} \bigg)\frac{K}{\gamma_{m+1,i}}\bigg(1+\frac{2\alpha'}{\alpha - 2\alpha'} \bigg) + \frac{2\alpha'\Reg(\pi)}{\alpha}.\\
    \Reg_{\hatf_{m+1}}(\pi) &\leq \frac{\alpha + 2\alpha'}{\alpha}\Reg(\pi) + \frac{\alpha}{\alpha-2\alpha'}\bigg( \frac{2}{\alpha} + \frac{\alpha}{8} + \frac{2\alpha'\eta}{\alpha} \bigg)\frac{K}{\gamma_{m+1,i}}.
    \end{aligned}
\end{equation}
If $\alpha \geq 4\alpha'$, we have that:
\begin{align}
\label{eq:bound-multiplicative-factor-on-Reg-mstar}
    \frac{\alpha + 2\alpha'}{\alpha} \leq 2, \text{ and }\; \frac{\alpha}{\alpha-2\alpha'} \leq 2.
\end{align}
Further, if it is also true that $\alpha\geq 4\sqrt{1+\alpha'\eta}$, we get:
\begin{equation}
\label{eq:bound-additive-factor-on-Reg-mstar}
    \begin{aligned}
        \frac{\alpha}{\alpha-2\alpha'}\bigg( \frac{2}{\alpha} + \frac{\alpha}{8} + \frac{2\alpha'\eta}{\alpha} \bigg) & \leq 2\bigg( \frac{2}{\alpha} + \frac{\alpha}{8} + \frac{2\alpha'\eta}{\alpha} \bigg)\\
        & \leq 2\bigg(\frac{\alpha}{4}\bigg).
    \end{aligned}
\end{equation}
We therefore choose $\alpha=4\max(\alpha',\sqrt{1+\alpha'\eta})$.
We finally get the required result by combining \eqref{eq:IH-implication-mstar}, \eqref{eq:IH-implication-est-mstar}, \eqref{eq:bound-multiplicative-factor-on-Reg-mstar}, and \eqref{eq:bound-additive-factor-on-Reg-mstar}.
\end{proof}

\begin{restatable}[]{lem}{lemboundReg}
\label{lem:policy-reg-bound}
Suppose the event $\eventReg$ holds. Consider any class index $i\in[M']$. For all policies $\pi$ and epochs $m\leq \msafe_i+1$ we have:
\begin{align*}
    \Reg(\pi) &\leq 2\Reg_{f}(\pi) + \frac{\eta_{i,m} K}{\gamma_{m,i}}, \quad\forall f\in\{\hatf_m,\hatf_{m,i} \}, \\
    \Reg_{\hatf_{m}}(\pi) &\leq 2\Reg(\pi) + \frac{\eta_{i,m} K}{\gamma_{m,i}},
\end{align*}
where $\eta_{i,m} = 2 + 4(\gamma_{m-1,1}/\gamma_{m-1,i})$.
\end{restatable}
\begin{proof}
We will prove this by induction. The base case follows from the fact that for all policies $\pi$, we have:
\begin{align*}
    \Reg(\pi) \leq 1 \leq \eta_1 K/\gamma_{1,i}\\
    \Reg_{\hatf_{1}}(\pi) \leq 1 \leq \eta_1 K/\gamma_{1,i}.
\end{align*}
For the inductive step, fix some $m\leq \msafe_i$. Assume for all policies $\pi$, we have:
\begin{align*}
    \Reg(\pi) &\leq 2\Reg_{f}(\pi) + \frac{\eta_{i,m} K}{\gamma_{m,i}},\quad \forall f\in\{\hatf_m,\hatf_{m,i} \},\\
    \Reg_{\hatf_{m}}(\pi) &\leq 2\Reg(\pi) + \frac{\eta_{i,m} K}{\gamma_{m,i}}.
\end{align*}
Therefore, from \Cref{lem:inductive-step-0.1-mstar} we have:
\begin{align*}
    \Reg(\pi) &\leq 2\Reg_{f}(\pi) + \frac{\eta'_{i,m+1} K}{\gamma_{m+1,i}}, \quad\forall f\in\{\hatf_{m+1},\hatf_{m+1,i}\}, \\
    \Reg_{\hatf_{m+1}}(\pi) &\leq 2\Reg(\pi) + \frac{\eta'_{i,m+1} K}{\gamma_{m+1,i}},
\end{align*}
where $\eta'_{i,m+1}= 2\max\Big(\frac{\gamma_m}{\gamma_{m,i}},\sqrt{1+\frac{\gamma_m}{\gamma_{m,i}}\eta_{i,m}}\Big)$. Then we have:
\begin{align*}
    \eta'_{i,m+1}&=2\max\bigg(\frac{\gamma_m}{\gamma_{m,i}},\sqrt{1+\frac{\gamma_m}{\gamma_{m,i}}\eta_{i,m}}\bigg) \\
    &\leq 2\max\bigg(\frac{\gamma_{m,1}}{\gamma_{m,i}},\sqrt{1+\frac{\gamma_{m,1}}{\gamma_{m,i}}\eta_{i,m}}\bigg)\\
    & \leq 2\max\bigg(\frac{\gamma_{m,1}}{\gamma_{m,i}},\sqrt{1 + 2^2\frac{\gamma_{m,1}^2}{\gamma_{m,i}^2} + 2\frac{\gamma_{m,1}}{\gamma_{m,i}} }\bigg) \\
    & \leq \max\bigg(2\frac{\gamma_{m,1}}{\gamma_{m,i}},\eta_{i,m+1} \bigg) \\
    &= \eta_{i,m+1}.
\end{align*}
This completes the inductive argument.
\end{proof}

\begin{lem}
\label{lem:policy-reg-bound-intermediate}
Suppose the event $\eventReg$ holds. Consider any two class indices $i,h\in[M']$ such that $h\leq i$. For all policies $\pi$ and epochs $m\in [\msafealg_{h-1}+1, \msafe_i+1]$, we have:
\begin{align*}
    \Reg(\pi) &\leq 2\Reg_{f}(\pi) + 8\frac{\gamma_{m-1,h}}{\gamma_{m-1,i}}\bigg(\frac{\gamma_{\msafealg_{h-1},1}}{\gamma_{\msafealg_{h-1},i}}\bigg)^{1/2^{m-\msafealg_{h-1}-1}}\frac{K}{\gamma_{m,i}}, \quad\forall f\in\{\hatf_{m},\hatf_{m,i}\}, \\
    \Reg_{\hatf_{m}}(\pi) &\leq 2\Reg(\pi) + 8\frac{\gamma_{m-1,h}}{\gamma_{m-1,i}}\bigg(\frac{\gamma_{\msafealg_{h-1},1}}{\gamma_{\msafealg_{h-1},i}}\bigg)^{1/2^{m-\msafealg_{h-1}-1}}\frac{K}{\gamma_{m,i}}.
\end{align*}
\end{lem}
\begin{proof}
Consider any two class indices $i,h\in[M']$ such that $h\leq i$. We will prove the required bound by induction. The bound for the base case $m=\msafealg_{h-1}+1$ follows from \Cref{lem:policy-reg-bound}. Suppose the bound in \Cref{lem:policy-reg-bound-intermediate} holds for class indices $i,h$ and for some epoch $m\in [\msafealg_{h-1}+1, \msafe_i]$. From \Cref{lem:inductive-step-0.1-mstar}, we have:
\begin{align*}
    \Reg(\pi) &\leq 2\Reg_{f}(\pi) + \frac{\eta' K}{\gamma_{m+1,i}}, \quad\forall f\in\{\hatf_{m+1},\hatf_{m+1,i}\}, \\
    \Reg_{\hatf_{m+1}}(\pi) &\leq 2\Reg(\pi) + \frac{\eta' K}{\gamma_{m+1,i}},
\end{align*}
along with:
\begin{align*}
    \eta'&=2\max\Bigg(\frac{\gamma_m}{\gamma_{m,i}},\sqrt{1+\frac{\gamma_m}{\gamma_{m,i}}8\frac{\gamma_{m-1,h}}{\gamma_{m-1,i}}\bigg(\frac{\gamma_{\msafealg_{h-1},1}}{\gamma_{\msafealg_{h-1},i}}\bigg)^{1/2^{m-\msafealg_{h-1}-1}}}\Bigg) \\
    & \stackrel{(i)}{\leq} 2\max\Bigg(\frac{\gamma_{m,h}}{\gamma_{m,i}},\sqrt{1+\frac{\gamma_{m,h}}{\gamma_{m,i}}8\frac{\gamma_{m,h}}{\gamma_{m,i}}\bigg(\frac{\gamma_{\msafealg_{h-1},1}}{\gamma_{\msafealg_{h-1},i}}\bigg)^{1/2^{m-\msafealg_{h-1}-1}}}\Bigg) \\
    & \stackrel{(ii)}{\leq} 8\frac{\gamma_{m,h}}{\gamma_{m,i}}\bigg(\frac{\gamma_{\msafealg_{h-1},1}}{\gamma_{\msafealg_{h-1},i}}\bigg)^{1/2^{m-\msafealg_{h-1}}},%
\end{align*}
where (i) follows from the fact that $\gamma_{m}\leq \gamma_{m,h}$ (since $m\geq \msafealg_{h-1}+1$), and \newline  $\gamma_{m-1,h}/\gamma_{m-1,i}\leq \gamma_{m,h}/\gamma_{m,i}$ for $h\leq i$. Then (ii) follows from the fact that $1+z\leq 2z$ for $z\geq 1$. Hence we have shown the bound in \Cref{lem:policy-reg-bound-intermediate} holds for class indices $i,h$ and epoch $m+1$. This completes our inductive argument.
\end{proof}

\subsection{Verifying Reward Model Agreement Across Classes}\label{sec:classgap}
In \Cref{lem:new_mtest}, we develop a bound on $\hatReg_{m+1,\hat{f}_{m+1}}(\pi_{\hatf{_m+1,i}}),$ which indicates whether the policy induced by the model predicted for class $i$ is considered to be a good policy by the model we have estimated. When this bound is exceeded, it suggests that we should use the exploitation parameter corresponding to larger classes. The implications of this test are captured in \Cref{lem:new_mtest_implication}, which provides a bound on $\Reg_{m+1,\hat{f}_{m+1}}(\pi_{\hatf{_m+1,i}})$ through $\msafealg_i,$ provided the test is satisfied.

\begin{restatable}{lem}{mtesttwo}
\label{lem:new_mtest}
Suppose the events $\eventReg$ and $\eventExpEval$ hold. Consider $h\leq i$ and $m\in[\msafealg_{h-1},\msafe_i]$, we then have:
\begin{equation}
    \begin{aligned}
        \hatReg_{m+1,\hatf_{m+1}}(\pi_{\hatf_{m+1,i}})\leq 26\frac{\gamma_{m,h}}{\gamma_{m,i}}\bigg(\frac{\gamma_{\msafealg_{h-1},1}}{\gamma_{\msafealg_{h-1},i}}\bigg)^{1/2^{m-\msafealg_{h-1}}}\frac{K}{\gamma_{m+1,i}}.
    \end{aligned}
\end{equation} 
\end{restatable}
\begin{proof}
From \Cref{lem:policy-reg-bound-intermediate}, we have the following for any policy $\pi$:
\begin{equation*}
    \begin{aligned}
        \Reg(\pi) &\leq 2\Reg_{f}(\pi) + 8\frac{\gamma_{m,h}}{\gamma_{m,i}}\bigg(\frac{\gamma_{\msafealg_{h-1},1}}{\gamma_{\msafealg_{h-1},i}}\bigg)^{1/2^{m-\msafealg_{h-1}}}\frac{K}{\gamma_{m+1,i}}, \quad\forall f\in\{\hatf_{m+1},\hatf_{m+1,i} \}, \\
        \Reg_{\hatf_{m+1}}(\pi) &\leq 2\Reg(\pi) + 8\frac{\gamma_{m,h}}{\gamma_{m,i}}\bigg(\frac{\gamma_{\msafealg_{h-1},1}}{\gamma_{\msafealg_{h-1},i}}\bigg)^{1/2^{m-\msafealg_{h-1}}}\frac{K}{\gamma_{m+1,i}}.
    \end{aligned}
\end{equation*}
Combining the above and plugging in $\pi=\pi_{\hat{f}_{m+1,i}}$, we have:
\begin{equation}
\begin{aligned}
\label{eq:new_test_bound}
\Reg_{\hatf_{m+1}}(\pi_{\hatf_{m+1,i}}) &\leq 2\Reg(\pi_{\hatf_{m+1,i}}) + 8\frac{\gamma_{m,h}}{\gamma_{m,i}}\bigg(\frac{\gamma_{\msafealg_{h-1},1}}{\gamma_{\msafealg_{h-1},i}}\bigg)^{1/2^{m-\msafealg_{h-1}}}\frac{K}{\gamma_{m+1,i}}\\
&\leq 4\Reg_{\hatf_{m+1,i}}(\pi_{\hatf_{m+1,i}}) +24\frac{\gamma_{m,h}}{\gamma_{m,i}}\bigg(\frac{\gamma_{\msafealg_{h-1},1}}{\gamma_{\msafealg_{h-1},i}}\bigg)^{1/2^{m-\msafealg_{h-1}}}\frac{K}{\gamma_{m+1,i}}\\
&\leq 24\frac{\gamma_{m,h}}{\gamma_{m,i}}\bigg(\frac{\gamma_{\msafealg_{h-1},1}}{\gamma_{\msafealg_{h-1},i}}\bigg)^{1/2^{m-\msafealg_{h-1}}}\frac{K}{\gamma_{m+1,i}},
\end{aligned}
\end{equation}
where the last inequality follows from $\Reg_{\hatf_{m+1,i}}(\pi_{\hatf_{m+1,i}})=0$. We then have:
\begin{align*}
&\hatReg_{m+1,\hatf_{m+1}}(\pi_{\hatf_{m+1,i}})\\
&=\hatR_{m+1,\hatf_{m+1}}(\pi_{\hatf_{m+1}})-\hatR_{m+1,\hatf_{m+1}}(\pi_{\hatf_{m+1,i}})\\
&=\big(\hatR_{m+1,\hatf_{m+1}}(\pi_{\hatf_{m+1}}) - R_{\hatf_{m+1}}(\pi_{\hatf_{m+1}})\big)\\
&\quad+ \big(R_{\hatf_{m+1}}(\pi_{\hatf_{m+1,i}})-\hatR_{m+1,\hatf_{m+1}}(\pi_{\hatf_{m+1,i}})\big) + \Reg_{\hatf_{m+1}}(\pi_{\hatf_{m+1,i}})\\
&\stackrel{(i)}{\leq} \frac{2K}{\gamma_{m+1,i}}+\Reg_{\hatf_{m+1}}(\pi_{\hatf_{m+1,i}})\\
&\stackrel{(ii)}{\leq} 26\frac{\gamma_{m,h}}{\gamma_{m,i}}\bigg(\frac{\gamma_{\msafealg_{h-1},1}}{\gamma_{\msafealg_{h-1},i}}\bigg)^{1/2^{m-\msafealg_{h-1}}}\frac{K}{\gamma_{m+1,i}},
\end{align*}
where (i) follows from $\eventExpEval$ and (ii) follows from \eqref{eq:new_test_bound}.
\end{proof}

\begin{lem}
\label{lem:new_mtest_implication}
Suppose the event $\eventExpEval$ holds. Consider $h\leq i$ and $m\in[\msafealg_{h-1},\msafealg_i]$, we then have:
\begin{equation*}
    \begin{aligned}
        \Reg_{\hatf_{m+1}}(\pi_{\hatf_{m+1,i}})\leq 28\frac{\gamma_{m,h}}{\gamma_{m,i}}\bigg(\frac{\gamma_{\msafealg_{h-1},1}}{\gamma_{\msafealg_{h-1},i}}\bigg)^{1/2^{m-\msafealg_{h-1}}}\frac{K}{\gamma_{m+1,i}}.
    \end{aligned}
\end{equation*}
\end{lem}
\begin{proof}
Since the test in \Cref{lem:new_mtest} holds up to $\msafealg_i$, we have the following guarantee:%
\begin{equation}
\begin{aligned}
\label{eq:bound_reg_fhat_pifhati}
\Reg_{\hatf_{m+1}}(\pi_{\hatf_{m+1,i}})&=R_{\hatf_{m+1}}(\pi_{\hatf_{m+1}})-R_{\hatf_{m+1}}(\pi_{\hatf_{m+1,i}})\\
&=\big(R_{\hatf_{m+1}}(\pi_{\hatf_{m+1}}) - \hatR_{m+1,\hatf_{m+1}}(\pi_{\hatf_{m+1}})\big)\\
&\quad+ \big(\hatR_{m+1,\hatf_{m+1}}(\pi_{\hatf_{m+1,i}})-R_{\hatf_{m+1}}(\pi_{\hatf_{m+1,i}})\big) + \hatReg_{m+1,\hatf_{m+1}}(\pi_{\hatf_{m+1,i}})\\
&\leq \frac{2K}{\gamma_{m+1,i}}+\hatReg_{m+1,\hatf_{m+1}}(\pi_{\hatf_{m+1,i}})\\
&\leq 28\frac{\gamma_{m,h}}{\gamma_{m,i}}\bigg(\frac{\gamma_{\msafealg_{h-1},1}}{\gamma_{\msafealg_{h-1},i}}\bigg)^{1/2^{m-\msafealg_{h-1}}}\frac{K}{\gamma_{m+1,i}}.
\end{aligned}
\end{equation}
\end{proof}

\subsection{Verifying Reward Model Agreement Across Epochs}

The goal of this section is to verify that potential new exploration policies had sufficiently low regret according to models in previous epochs. This helps ensure that these new exploration policies were well-explored in previous epochs and we can rely on our estimates for these policies. \Cref{lem:policy-reg-bound-wrt-old-model} provides the expected reward model agreement across epochs by bounding $\Reg_{\hatf_m}(\pi)$ in terms of regret according to $\hatf_{m+1}$ and $\hatf_{m+1,i}$.
\Cref{lem:policy-reg-bound-wrt-old-model-test} describes the corresponding test. \Cref{lem:policy-reg-bound-wrt-old-model-test-implication} provides implied guarantees as long this test holds.

\begin{lem}
\label{lem:policy-reg-bound-wrt-old-model}
Suppose the event $\eventReg$ holds. Consider any two class indices $i,h\in[M']$ such that $h\leq i$. For all policies $\pi$ and epochs $m\in [\msafealg_{h-1}, \msafe_i]$, we have:
\begin{align*}
    &\Reg_{\hatf_{m}}(\pi) \leq 4\Reg_{f}(\pi) + 24\frac{\gamma_{m,h}}{\gamma_{m,i}}\bigg(\frac{\gamma_{\msafealg_{h-1},1}}{\gamma_{\msafealg_{h-1},i}}\bigg)^{1/2^{\max(0,m-\msafealg_{h-1}-1)}}\frac{K}{\gamma_{m,i}}, \quad\forall f\in\{\hatf_{m+1},\hatf_{m+1,i}\}.
\end{align*}
\end{lem}
\begin{proof}
Consider $m\in [\msafealg_{h-1}, \msafe_i]$, policy $\pi$, and $f\in\{\hatf_{m+1},\hatf_{m+1,i}\}$.
    \begin{equation}
    \begin{aligned}
    \Reg_{\hatf_{m}}(\pi)
    &\stackrel{(i)}{\leq} 2\Reg(\pi) + 8\frac{\gamma_{m,h}}{\gamma_{m,i}}\bigg(\frac{\gamma_{\msafealg_{h-1},1}}{\gamma_{\msafealg_{h-1},i}}\bigg)^{1/2^{\max(0,m-\msafealg_{h-1}-1)}}\frac{K}{\gamma_{m,i}}\\
    &\stackrel{(ii)}{\leq} 4\Reg_{f}(\pi) +24\frac{\gamma_{m,h}}{\gamma_{m,i}}\bigg(\frac{\gamma_{\msafealg_{h-1},1}}{\gamma_{\msafealg_{h-1},i}}\bigg)^{1/2^{\max(0,m-\msafealg_{h-1}-1)}}\frac{K}{\gamma_{m,i}}.
    \end{aligned}
    \end{equation}
    To show (i) we consider two cases. For the case $m=\msafealg_{h-1}$, (i) follows from \Cref{lem:policy-reg-bound} and the fact that $\frac{\gamma_{m-1,1}}{\gamma_{m-1,i}}\leq \frac{\gamma_{m,1}}{\gamma_{m,i}}$. For the case $m>\msafealg_{h-1}$, (i) follows from \Cref{lem:policy-reg-bound-intermediate} and the fact that $\frac{\gamma_{m-1,h}}{\gamma_{m-1,i}}\leq \frac{\gamma_{m,h}}{\gamma_{m,i}}$. Then (ii) follows from \Cref{lem:policy-reg-bound-intermediate} and the fact that $\frac{\gamma_{\msafealg_{h-1},1}}{\gamma_{\msafealg_{h-1},i}}\geq 1, \frac{K}{\gamma_{m,i}} \geq \frac{K}{\gamma_{m+1,i}}$.
\end{proof}

\begin{restatable}{lem}{mtestthree}
\label{lem:policy-reg-bound-wrt-old-model-test}
Suppose the event $\eventReg$ holds. Consider any two class indices $i,h\in[M']$ such that $h\leq i$. For all policies $\Pi_{0,m+1,i}=\pi\in \{\pi_{\hatf_{m+1}},\pi_{\hatf_{m+1,i}},p_{m+1,1},\dots,p_{m+1,M'} \}$, epochs $m\in [\msafealg_{h-1}, \msafe_i]$, and models $f\in\{\hatf_{m+1},\hatf_{m+1,i}\}$, we have:
\begin{equation}
\begin{aligned}
    &\hatReg_{m+1,\hatf_{m}}(\pi) \leq 4\hatReg_{m+1,f}(\pi) + 34\frac{\gamma_{m,h}}{\gamma_{m,i}}\bigg(\frac{\gamma_{\msafealg_{h-1},1}}{\gamma_{\msafealg_{h-1},i}}\bigg)^{1/2^{\max(0,m-\msafealg_{h-1}-1)}}\frac{K}{\gamma_{m,i}}.
\end{aligned}
\end{equation}
\end{restatable}
\begin{proof}
    Consider any policy $\pi\in \{\pi_{\hatf_{m+1}},\pi_{\hatf_{m+1},i},p_{m+1,1},\dots,p_{m+1,M'} \}$, model $f\in\{\hatf_{m+1},\hatf_{m+1,i}\}$, and epoch $m\in [\msafealg_{h-1}, \msafe_i]$.
    \begin{align*}
    &\hatReg_{m+1,\hatf_{m}}(\pi) - 4\hatReg_{m+1,f}(\pi)\\
    &=(\hatR_{m+1,\hatf_{m}}(\pi_{\hatf_{m}})-\hatR_{m+1,\hatf_{m}}(\pi)) - 4(\hatR_{m+1,f}(\pi_{f})-\hatR_{m+1,f}(\pi))\\
    &=\big(\hatR_{m+1,\hatf_{m}}(\pi_{\hatf_{m}}) - R_{\hatf_{m}}(\pi_{\hatf_{m}})\big)+ \big(R_{\hatf_{m}}(\pi)-\hatR_{m+1,\hatf_{m}}(\pi)\big) + \Reg_{\hatf_{m}}(\pi)\\
    &\quad+ 4\big(R_{f}(\pi_{f})-\hatR_{m+1,f}(\pi_{f})\big) + 4\big(\hatR_{m+1,f}(\pi)-R_{f}(\pi) \big) - 4\Reg_{f}(\pi)\\
    &\stackrel{(i)}{\leq} \frac{10K}{\gamma_{m+1,i}}+\Reg_{\hatf_{m}}(\pi)-4\Reg_{f}(\pi)\\
    &\stackrel{(ii)}{\leq} 34\frac{\gamma_{m,h}}{\gamma_{m,i}}\bigg(\frac{\gamma_{\msafealg_{h-1},1}}{\gamma_{\msafealg_{h-1},i}}\bigg)^{1/2^{\max(0,m-\msafealg_{h-1}-1)}}\frac{K}{\gamma_{m,i}},
    \end{align*}
    where (i) follows from $\eventExpEval$ and (ii) follows from \Cref{lem:policy-reg-bound-wrt-old-model}.
\end{proof}

\begin{lem}
\label{lem:policy-reg-bound-wrt-old-model-test-implication}
Suppose the event $\eventReg$ holds. Consider any two class indices $i,h\in[M']$ such that $h\leq i$. For all policies $\pi\in \{\pi_{\hatf_{m+1}},\pi_{\hatf_{m+1,i}},p_{m+1,1},\dots,p_{m+1,M'} \}$, epochs $m\in [\msafealg_{h-1}, \msafealg_i]$, and models $f\in\{\hatf_{m+1},\hatf_{m+1,i}\}$, we have:
\begin{align*}
    &\Reg_{\hatf_{m}}(\pi) \leq 4\Reg_{f}(\pi) + 44\frac{\gamma_{m,h}}{\gamma_{m,i}}\bigg(\frac{\gamma_{\msafealg_{h-1},1}}{\gamma_{\msafealg_{h-1},i}}\bigg)^{1/2^{\max(0,m-\msafealg_{h-1}-1)}}\frac{K}{\gamma_{m,i}}.
\end{align*}
\end{lem}
\begin{proof}
    Consider any policy $\pi\in \{\pi_{\hatf_{m+1}},\pi_{\hatf_{m+1},i},p_{m+1,1},\dots,p_{m+1,M'} \}$, model $f\in\{\hatf_{m+1},\hatf_{m+1,i}\}$, and epoch $m\in [\msafealg_{h-1}, \msafealg_i]$.
    \begin{align*}
    &\Reg_{\hatf_{m}}(\pi) - 4\Reg_{f}(\pi)\\
    &=(R_{\hatf_{m}}(\pi_{\hatf_{m}})-R_{\hatf_{m}}(\pi)) - 4(R_{f}(\pi_{f})-R_{f}(\pi))\\
    &=\big(R_{\hatf_{m}}(\pi_{\hatf_{m}}) - \hatR_{m+1,\hatf_{m}}(\pi_{\hatf_{m}})\big)+ \big(\hatR_{m+1,\hatf_{m}}(\pi)-R_{\hatf_{m}}(\pi)\big) + \hatReg_{m+1,\hatf_{m}}(\pi)\\
    &\quad+ 4\big(\hatR_{m+1,f}(\pi_{f})-R_{f}(\pi_{f})\big) + 4\big(R_{f}(\pi)-\hatR_{m+1,f}(\pi) \big) - 4\hatReg_{m+1,f}(\pi)\\
    &\stackrel{(i)}{\leq} \frac{10K}{\gamma_{m+1,i}}+\hatReg_{m+1,\hatf_{m}}(\pi)-4\hatReg_{m+1,f}(\pi)\\
    &\stackrel{(ii)}{\leq} 44\frac{\gamma_{m,h}}{\gamma_{m,i}}\bigg(\frac{\gamma_{\msafealg_{h-1},1}}{\gamma_{\msafealg_{h-1},i}}\bigg)^{1/2^{\max(0,m-\msafealg_{h-1}-1)}}\frac{K}{\gamma_{m,i}},
    \end{align*}
    where (i) follows from $\eventExpEval$ and (ii) follows from \Cref{lem:policy-reg-bound-wrt-old-model-test}.
\end{proof}

\section{INDUCTIVE ARGUMENT BASED ON TESTED GUARANTEES}\label{app:inductive-argument-with-verified-bounds}

In \Cref{app:main-policy-based-test,app:reward-model-agreement}, we developed several verified guarantees on our estimated reward models. In this section, we rely on these guarantees to relate the true regret (with respect to the best policies in different classes) to the regret according to estimated models. Our proof follows by induction and demonstrates the benefits of the self-correction step (holding candidate exploitation parameters fixed by not increasing epoch lengths for a few epochs) in our algorithm.

\textbf{Proof outline:} \Cref{lem:inductive-step-0.1-mhat} is our main inductive step. \Cref{lem:policy-reg-bound-mhat,lem:policy-reg-bound-intermediate-mhat} apply this step in order to relate the true regret (with respect to the best policies in different classes) to the regret according to estimated models. \Cref{lem:policy-reg-bound-intermediate-mhat} in particular demonstrates how holding candidate exploitation parameters fixed for some epochs helps with correcting for the effects of under-exploration on our estimated reward models.

\subsection{Inductive Step Based on Tested Guarantees}\label{sec:inductive-step-based-on-tested-guarantees}

\Cref{lem:inductive-step-0.1-mhat} is our main inductive step that utilizes tested guarantees. We state and prove it in this section.

\begin{lem}
\label{lem:inductive-step-0.1-mhat}
Suppose the event $\eventReg$ defined in \eqref{eq:w-event-Reg} holds. Consider any class indices $i,j\in[M']$ such that $j\geq i$, and consider any epoch $m\in[\msafealg_i]$. Let $\theta_{i,j}:=\frac{\gamma_{m+1,i}}{\gamma_{m+1,j}}$. Suppose there exist constants ($\eta,\Tilde{\eta},\Tilde{\eta}'>0$) such that we have the following for any $f\in\{\hatf_{m+1},\hatf_{m+1,i}\}$:
\begin{equation}
\label{eq:final-inductive-step-hypothesis}
    \begin{aligned}
        &\Reg_{j}(\pi) \leq 2\Reg_{\hatf_{m}}(\pi) + \frac{\eta K}{\gamma_{m,i}}, \quad\forall \pi\in\allPi_j \\
        &\Reg_{\hatf_{m}}(\pi) \leq 2\Reg_{j}(\pi) + \frac{\eta K}{\gamma_{m,i}}, \quad\forall \pi\in\allPi_j\\
        &\Reg_{\hatf_{m+1}}(\pi_{\hat{f}_{m+1,i}}) \leq \frac{\Tilde{\eta} K}{\gamma_{m+1,i}},\\
        & \Reg_{\hatf_{m}}(\pi) \leq 4\Reg_{f}(\pi) + \frac{\Tilde{\eta}' K}{\gamma_{m,i}}, \quad \forall\pi\in \{\pi_{\hatf_{m+1}},\pi_{\hatf_{m+1},i},p_{m+1,1},\dots,p_{m+1,M'} \}.
    \end{aligned}
\end{equation}
We then have that for all polices $\pi\in \allPi_j\cup \{\pi_{\hatf_{m+1}},p_{m+1,1},\dots,p_{m+1,M'} \}$ and models $f\in\{\hatf_{m+1}, \hatf_{m+1,i} \}$:
\begin{equation}
\label{eq:final-inductive-step-result}
    \begin{aligned}
        \Reg_{j}(\pi) &\leq 2\Reg_{f}(\pi) + \frac{\eta' K}{\gamma_{m+1,i}}\\
        \Reg_{f}(\pi) &\leq 2\Reg_{j}(\pi) + \frac{\eta' K}{\gamma_{m+1,i}},
    \end{aligned}
\end{equation}
where $\eta'= 5\alpha'(1+2\theta_{i,j})^2 + 2\Tilde{\eta} + 2(1+2\theta_{i,j})\sqrt{\alpha'(\eta+\Tilde{\eta}')}$ and $\alpha'=\gamma_m/\gamma_{m,i}$.
\end{lem}

\begin{proof}
To begin our proof, we first define a few quantities. Let $\alpha'=\gamma_{m}/\gamma_{m,i}$, and let $\alpha\geq 8\alpha'(1+2\theta_{i,j})$ be a positive constant (will be fixed later in the proof). Let $f$ be any model in $\{\hat{f}_{m+1},\hat{f}_{m+1,i}\}$. Let $\allPi_{0,m+1}=\{\pi_{\hatf_{m+1}},p_{m+1,1},\dots,p_{m+1,M'} \}$. Our proof is broken into two parts.

\textbf{Part 1:} The first part of the proof works towards the first inequality in \eqref{eq:final-inductive-step-result}. We start with bounding the difference between $\Reg_j(\pi)$ and $\Reg_f(\pi)$ for all policies $\pi\in \allPi_j \cup \allPi_{0,m+1}$. 
\begin{equation}
\label{eq:inductive-bound-for-pi-mhat}
    \begin{aligned}
        \pi\in \allPi_j \cup \allPi_{0,m+1},\;\;& \Reg_{j}(\pi) - \Reg_{f}(\pi)\\ 
        =& \Big(R(\pi_j^*) - R(\pi)\Big) - \Big(R_{f}(\pi_{f}) - R_{f}(\pi)\Big) \\
        \stackrel{(i)}{\leq} & \Big(R(\pi_j^*) - R(\pi)\Big) -  \Big(R_{f}(\pi^*_j) - R_{f}(\pi)\Big) \\
        \stackrel{(ii)}{\leq} & |R(\pi^*_j) - R_{f}(\pi^*_j)| + |R(\pi) - R_{f}(\pi))|\\
        \stackrel{(iii)}{\leq} & 2\Bar{C} + \frac{(1+2\theta_{i,j})}{\alpha}\frac{\gamma_m}{\gamma_{m+1,i}}\bigg(\Reg_{\hat{f}_m}(\pi^*_j) + \Reg_{\hat{f}_m}(\pi) \bigg),\\
        \stackrel{(iv)}{\leq} & 2\Bar{C} + \frac{(1+2\theta_{i,j})\alpha'}{\alpha}\frac{\eta K}{\gamma_{m+1,i}} + \frac{(1+2\theta_{i,j})}{\alpha}\frac{\gamma_m}{\gamma_{m+1,i}} \Reg_{\hat{f}_m}(\pi),
    \end{aligned}
\end{equation}
where (i) follows from the definition of $\pi_{f}$, (ii) follows from the triangle inequality, (iii) follows from \Cref{lem:combinedprevious}, and (iv) follows from \eqref{eq:final-inductive-step-hypothesis} and $\Reg_j(\pi^*_j)=0$. For brevity, in (iii), we have defined the quantity:
$$\Bar{C}=\bigg( \frac{1+2\theta_{i,j}}{\alpha} + \frac{(1+2\theta_{i,j})\alpha}{16} + \frac{(2\theta_{i,j}^2+2(1+\theta_{i,j})^2/\alpha + 2\theta_{i,j})\gamma_m}{\gamma_{m+1,i}} + 2\theta_{i,j} \bigg)\frac{K}{\gamma_{m+1,i}},$$
which is the first term from the result of \Cref{lem:combinedprevious}.

\textbf{Part 1 (Case 1: $\pi\in\allPi_j$):} This case only considers policies $\pi\in\allPi_j$, we refine \eqref{eq:inductive-bound-for-pi-mhat} for such policies using \eqref{eq:final-inductive-step-hypothesis}. Combining \eqref{eq:inductive-bound-for-pi-mhat} and \eqref{eq:final-inductive-step-hypothesis}, we get:
\begin{equation}
\label{eq:IH-implication-mhat}
    \begin{aligned}
    &\Reg_{j}(\pi) - \Reg_{f}(\pi) \leq 2\Bar{C}  + \frac{(1+2\theta_{i,j})\alpha'}{\alpha}\frac{\eta K}{\gamma_{m+1,i}} + \frac{(1+2\theta_{i,j})}{\alpha}\frac{\gamma_m}{\gamma_{m+1,i}}\bigg(2\Reg_{j}(\pi) + \frac{\eta K}{\gamma_{m,i}} \bigg)\\
    \implies &\frac{\alpha-2\alpha'(1+2\theta_{i,j})}{\alpha}  \Reg_{j}(\pi) \leq \Reg_{f}(\pi)+ 2\Bar{C} + \frac{2\alpha'\eta K(1+2\theta_{i,j})}{\alpha\gamma_{m+1,i}} \\
    \implies &\Reg_{j}(\pi) \leq \frac{\alpha}{\alpha-2\alpha'(1+2\theta_{i,j})}\bigg(\Reg_{f}(\pi) + 2\Bar{C} + \frac{2\alpha'\eta K(1+2\theta_{i,j})}{\alpha\gamma_{m+1,i}}\bigg)\\
    \implies &\Reg_{j}(\pi) \leq 2\Reg_{f}(\pi) + 4\Bar{C} + \frac{4\alpha'\eta K(1+2\theta_{i,j})}{\alpha\gamma_{m+1,i}},
    \end{aligned}
\end{equation}
where the last implication follows from the fact that $\alpha\geq 4\alpha'(1+2\theta_{i,j})$ and hence $\frac{\alpha}{\alpha-2\alpha'(1+2\theta_{i,j})}\leq 2$. 

\textbf{Part 1 (Case 2: $\pi\in\Pi_{0,m+1}$):} This case only considers policies $\pi\in\Pi_{0,m+1}$, we refine \eqref{eq:inductive-bound-for-pi-mhat} for such policies using \eqref{eq:final-inductive-step-hypothesis}. Combining \eqref{eq:inductive-bound-for-pi-mhat} and \eqref{eq:final-inductive-step-hypothesis}, we get:
\begin{equation}
\label{eq:IH-implication-mhat-case2}
    \begin{aligned}
    &\Reg_{j}(\pi) - \Reg_{f}(\pi) \leq 2\Bar{C}  + \frac{(1+2\theta_{i,j})\alpha'}{\alpha}\frac{\eta K}{\gamma_{m+1,i}} + \frac{(1+2\theta_{i,j})}{\alpha}\frac{\gamma_m}{\gamma_{m+1,i}}\bigg(4\Reg_{f}(\pi) + \frac{\Tilde{\eta}' K}{\gamma_{m,i}} \bigg)\\
    \implies &\Reg_{j}(\pi) \leq \bigg(1 + \frac{4(1+2\theta_{i,j})}{\alpha}\frac{\gamma_m}{\gamma_{m+1,i}} \bigg)\Reg_{f}(\pi)+ 2\Bar{C} + \frac{\alpha'(\eta + \Tilde{\eta}' )K(1+2\theta_{i,j})}{\alpha\gamma_{m+1,i}}\\
    \implies &\Reg_{j}(\pi) \leq 2\Reg_{f}(\pi)+ 2\Bar{C} + \frac{\alpha'(\eta + \Tilde{\eta}' )K(1+2\theta_{i,j})}{\alpha\gamma_{m+1,i}},
    \end{aligned}
\end{equation}
where the last implication follows from the fact that $\alpha\geq 4\alpha'(1+2\theta_{i,j})$.

\textbf{Part 2:} The second part of the proof works towards the second inequality in \eqref{eq:final-inductive-step-result}. We start with bounding the difference between $\Reg_f(\pi)$ and $\Reg_j(\pi)$ for all policies $\pi\in \allPi_j \cup \allPi_{0,m+1}$. 

\begin{equation}
\label{eq:inductive-bound-for-pi-est-mhat}
\begin{aligned}
    \pi\in \allPi_j \cup \allPi_{0,m+1},\;\;&\Reg_{f}(\pi)- \Reg_{j}(\pi) \\
    &= \Big(R_{f}(\pi_{f}) - R_{f}(\pi)\Big) - \Big(R(\pi^*_j) - R(\pi)\Big) \\
    &\stackrel{(i)}{\leq}  \Big(R_{f}(\pi_{f}) - R_{f}(\pi)\Big) - \Big(R(\pi_{\hatf_{m+1,i}}) - R(\pi)\Big) \\
    &\stackrel{(ii)}{\leq} \Reg_f(\pi_{\hatf_{m+1,i}}) +  \Big|R_{f}(\pi_{\hatf_{m+1,i}}) - R(\pi_{\hatf_{m+1,i}})\Big| + \Big|R(\pi) - R_{f}(\pi)\Big|,\\
    &\stackrel{(iii)}{\leq} \frac{\Tilde{\eta}K}{\gamma_{m+1,i}} +  \Big|R_{f}(\pi_{\hatf_{m+1,i}}) - R(\pi_{\hatf_{m+1,i}})\Big| + \Big|R(\pi) - R_{f}(\pi)\Big|\\
    &\stackrel{(iv)}{\leq} \frac{\Tilde{\eta}K}{\gamma_{m+1,i}} +  2\Bar{C} + \frac{(1+2\theta_{i,j})}{\alpha}\frac{\gamma_m}{\gamma_{m+1,i}}\bigg(\Reg_{\hat{f}_m}(\pi_{\hatf_{m+1,i}}) + \Reg_{\hat{f}_m}(\pi) \bigg)\\
    &\stackrel{(v)}{\leq} \frac{\Tilde{\eta}K}{\gamma_{m+1,i}} +  2\Bar{C} + \frac{(1+2\theta_{i,j})\alpha'}{\alpha}\frac{\Tilde{\eta}'K}{\gamma_{m+1,i}} + \frac{(1+2\theta_{i,j})}{\alpha}\frac{\gamma_m}{\gamma_{m+1,i}} \Reg_{\hat{f}_m}(\pi),
\end{aligned}
\end{equation}
where (i) follows from the definition of $\pi^*_j$ and $\pi_{\hat{f}_{m+1,i}}\in\allPi_j$, (ii) follows from triangle inequality, (iii) follows from \eqref{eq:final-inductive-step-hypothesis}, (iv) follows from \Cref{lem:combinedprevious}, and (v) follows from \eqref{eq:final-inductive-step-hypothesis} and $\Reg_{\hatf_{m+1,i}}(\pi_{\hatf_{m+1,i}})=0$. 

\textbf{Part 2 (Case 1 $\pi\in\allPi_j$):} This case only considers policies $\pi\in\allPi_j$, we refine \eqref{eq:inductive-bound-for-pi-est-mhat} for such policies using \eqref{eq:final-inductive-step-hypothesis}. Combining \eqref{eq:inductive-bound-for-pi-est-mhat} and \eqref{eq:final-inductive-step-hypothesis}, we get:
\begin{equation}
\label{eq:inductive-bound-for-pi-est-mhat-case1}
    \begin{aligned}
    &\Reg_f(\pi)\\
    &\leq \frac{\Tilde{\eta}K}{\gamma_{m+1,i}} +  2\Bar{C} + \frac{(1+2\theta_{i,j})\alpha'}{\alpha}\frac{\Tilde{\eta}'K}{\gamma_{m+1,i}} + \frac{(1+2\theta_{i,j})}{\alpha}\frac{\gamma_m}{\gamma_{m+1,i}} \bigg(2\Reg_{j}(\pi)+\frac{\eta K}{\gamma_{m,i}} \bigg) + \Reg_j(\pi) \\
    &\leq \bigg(1+ \frac{2(1+2\theta_{i,j})}{\alpha}\frac{\gamma_m}{\gamma_{m+1,i}} \bigg)\Reg_j(\pi) + 2\Bar{C} + \frac{K}{\gamma_{m+1,i}}\bigg(\Tilde{\eta} + \frac{(1+2\theta_{i,j})\alpha'}{\alpha}(\eta + \Tilde{\eta}') \bigg)\\
    &\leq 2\Reg_j(\pi) + 2\Bar{C} + \frac{K}{\gamma_{m+1,i}}\bigg(\Tilde{\eta} + \frac{(1+2\theta_{i,j})\alpha'}{\alpha}(\eta + \Tilde{\eta}') \bigg),
    \end{aligned}
\end{equation}
where the last inequality follows from the fact that $\alpha\geq 4\alpha'(1+2\theta_{i,j})$.

\textbf{Part 2 (Case 2 $\pi\in\allPi_{0,m+1}$):} This case only considers policies $\pi\in\allPi_{0,m+1}$, we refine \eqref{eq:inductive-bound-for-pi-est-mhat} for such policies using \eqref{eq:final-inductive-step-hypothesis}. Combining \eqref{eq:inductive-bound-for-pi-est-mhat} and \eqref{eq:final-inductive-step-hypothesis}, we get:
\begin{equation}
\label{eq:inductive-bound-for-pi-est-mhat-case2}
    \begin{aligned}
    &\Reg_f(\pi)-\Reg_j(\pi)\leq \frac{\Tilde{\eta}K}{\gamma_{m+1,i}} +  2\Bar{C} + \frac{(1+2\theta_{i,j})\alpha'}{\alpha}\frac{\Tilde{\eta}'K}{\gamma_{m+1,i}} + \frac{(1+2\theta_{i,j})}{\alpha}\frac{\gamma_m}{\gamma_{m+1,i}} \bigg(4\Reg_{f}(\pi)+\frac{\Tilde{\eta}' K}{\gamma_{m,i}} \bigg)  \\
    &\implies \bigg(1- \frac{4(1+2\theta_{i,j})}{\alpha}\frac{\gamma_m}{\gamma_{m+1,i}} \bigg)\Reg_f(\pi) \leq \Reg_j(\pi) + 2\Bar{C} + \frac{K}{\gamma_{m+1,i}}\bigg(\Tilde{\eta} + \frac{2(1+2\theta_{i,j})\alpha'\Tilde{\eta}'}{\alpha} \bigg)\\
    &\implies \Reg_f(\pi) \leq 2\Reg_j(\pi) + 4\Bar{C} + \frac{K}{\gamma_{m+1,i}}\bigg(2\Tilde{\eta} + \frac{4(1+2\theta_{i,j})\alpha'\Tilde{\eta}'}{\alpha} \bigg),
    \end{aligned}
\end{equation}
where the last implication follows from the fact that $\alpha\geq 8\alpha'(1+2\theta_{i,j})$.

Note that, if \eqref{eq:condition-eta'} holds and if $\alpha\geq 8\alpha'(1+2\theta_{i,j})$, then \eqref{eq:IH-implication-mhat},\eqref{eq:IH-implication-mhat-case2},
\eqref{eq:inductive-bound-for-pi-est-mhat-case1} \eqref{eq:inductive-bound-for-pi-est-mhat-case2} imply \eqref{eq:final-inductive-step-result} holds for all policies in $\allPi_j\cup\Pi_{0,m+1}$:
\begin{equation}
\label{eq:condition-eta'}
    \begin{aligned}
        \eta' \geq & 4\bigg( \frac{1+2\theta_{i,j}}{\alpha} + \frac{(1+2\theta_{i,j})\alpha}{16} + \frac{(2\theta_{i,j}^2+2(1+\theta_{i,j})^2/\alpha + 2\theta_{i,j})\gamma_m}{\gamma_{m+1,i}} + 2\theta_{i,j} \bigg) \\
        &+ 2\Tilde{\eta} + \frac{4(1+2\theta_{i,j})\alpha'(\eta+\Tilde{\eta}')}{\alpha}. 
    \end{aligned}
\end{equation}
We now fix our choice of $\alpha$ in a way that ensures \eqref{eq:condition-eta'} and $\alpha\geq 8\alpha'(1+2\theta_{i,j})$ does infact hold. We choose $\alpha = 8\alpha'(1+2\theta_{i,j})+ 4\sqrt{\alpha'(\eta+\Tilde{\eta}')} $.%
Clearly, $\alpha\geq 8\alpha'(1+2\theta_{i,j})$ holds. We will now show that \eqref{eq:condition-eta'} holds.
\begin{equation}
    \begin{aligned}
        &4\bigg( \frac{1+2\theta_{i,j}}{\alpha} + \frac{(1+2\theta_{i,j})\alpha}{16} + \frac{(2\theta_{i,j}^2+2(1+\theta_{i,j})^2/\alpha + 2\theta_{i,j})\gamma_m}{\gamma_{m+1,i}} + 2\theta_{i,j} \bigg)\\ 
        &+ 2\Tilde{\eta} + \frac{4(1+2\theta_{i,j})\alpha'(\eta+\Tilde{\eta}')}{\alpha}\\
        \stackrel{(i)}{\leq} & 4\bigg( \frac{1}{8} + \frac{(1+2\theta_{i,j})\alpha}{16} + 2\theta_{i,j}^2\alpha' + \frac{1+\theta_{i,j}}{4} + 2\theta_{i,j}\alpha'  \bigg) + 2\Tilde{\eta} + (1+2\theta_{i,j})\sqrt{\alpha'(\eta+\Tilde{\eta}')}\\
        \stackrel{(ii)}{=} & 4\bigg( \frac{1}{8} + \frac{(1+2\theta_{i,j})^2}{2} + 2\theta_{i,j}^2\alpha' + \frac{1+\theta_{i,j}}{4} + 2\theta_{i,j}\alpha'  \bigg) + 2\Tilde{\eta} + 2(1+2\theta_{i,j})\sqrt{\alpha'(\eta+\Tilde{\eta}')}\\
        \leq & 5\alpha'(1+2\theta_{i,j})^2 + 2\Tilde{\eta} + 2(1+2\theta_{i,j})\sqrt{\alpha'(\eta+\Tilde{\eta}')} = \eta',
    \end{aligned}
\end{equation}
where (i) follows from $\alpha \geq \max\{8\alpha'(1+2\theta_{i,j}), 4\sqrt{\alpha'(\eta+\Tilde{\eta}')} \}$, (ii) follows from our choice of $\alpha$, and the last inequality follows from simple algebraic manipulations. This completes our proof.

\end{proof}

\subsection{Under-Exploration and Self Correction}
\label{app:self-correction}

\Cref{lem:policy-reg-bound-mhat,lem:policy-reg-bound-intermediate-mhat} apply the inductive step established in \Cref{sec:inductive-step-based-on-tested-guarantees} to relate the true regret (with respect to the best policies in different classes) to the regret according to estimated models. \Cref{lem:policy-reg-bound-intermediate-mhat} in particular demonstrates the self-correction property of $\CVBandit$.

\begin{restatable}[]{lem}{lemboundRegmHat}
\label{lem:policy-reg-bound-mhat}
Suppose the event $\eventReg$ holds. Consider any class indices $i,j\in[M']$ such that $j\geq i$. Let $\theta_{i,j}:=\frac{\gamma_{m+1,i}}{\gamma_{m+1,j}}$. For any policy $\pi\in\Tilde{\Pi}_{j}\cup \Pi_{0,m}$, model $f_m\in\{\hatf_m,\hatf_{m,i}\}$, and epoch $m\leq \msafealg_i+1$, we have:
\begin{align*}
    \Reg_{j}(\pi) \leq 2\Reg_{f_m}(\pi) + \frac{\eta_{i,m} K}{\gamma_{m,i}}\\
    \Reg_{f_m}(\pi) \leq 2\Reg_{j}(\pi) + \frac{\eta_{i,m} K}{\gamma_{m,i}},
\end{align*}
where $\eta_{i,m} = 100(1+2\theta_{i,j})^2(\gamma_{m-1,1}/\gamma_{m-1,i}) $.
\end{restatable}
\begin{proof}
We will prove this by induction. The base case follows from the fact that for any policy $\pi\in\Tilde{\Pi}_{j}\cup \Pi_{0,1}$ and model $f_1\in \{\hatf_1,\hatf_{1,i}\}$, we have:
\begin{equation}
    \begin{aligned}
    \Reg_{j}(\pi) \leq 1 \leq -2 + \eta_1 K/\gamma_{1,i}\\
    \Reg_{f_{1}}(\pi) \leq 1 \leq -2+ \eta_1 K/\gamma_{1,i}.
    \end{aligned}
\end{equation}
For the inductive step, fix some $m\leq \msafealg_i$. Assume for any policy $\pi\in\Tilde{\Pi}_{j}\cup \Pi_{0,m}$  and model $f_m\in \{\hatf_m,\hatf_{m,i}\}$, we have:
\begin{equation}
\label{eq:inductive-assumption-for-induction-with-tests}
    \begin{aligned}
    \Reg_{j}(\pi) \leq 2\Reg_{f_{m}}(\pi) + \frac{\eta_{i,m} K}{\gamma_{m,i}}\\
    \Reg_{f_{m}}(\pi) \leq 2\Reg_{j}(\pi) + \frac{\eta_{i,m} K}{\gamma_{m,i}}.
    \end{aligned}
\end{equation}
Let $f_{m+1}\in \{\hatf_{m+1},\hatf_{m+1,i}\}$. From \Cref{lem:new_mtest_implication} and \Cref{lem:policy-reg-bound-wrt-old-model-test-implication}, we have that \eqref{eq:test-implications-for-first-induction} holds:
\begin{equation}
\label{eq:test-implications-for-first-induction}
    \begin{aligned}
        &\Reg_{\hatf_{m+1}}(\pi_{\hat{f}_{m+1,i}}) \leq \frac{\Tilde{\eta} K}{\gamma_{m+1,i}},\\
        & \Reg_{\hatf_{m}}(\pi) \leq 4\Reg_{f_{m+1}}(\pi) + \frac{\Tilde{\eta}' K}{\gamma_{m,i}}, \quad \forall\pi\in \{\pi_{\hatf_{m+1}},\pi_{\hatf_{m+1},i},p_{m+1,1},\dots,p_{m+1,M'} \},\\
        & \Tilde{\eta} = 28\frac{\gamma_{m,1}}{\gamma_{m,i}},\;\; \Tilde{\eta}' = 44\frac{\gamma_{m,1}}{\gamma_{m,i}}.
    \end{aligned}
\end{equation}
Now from \eqref{eq:inductive-assumption-for-induction-with-tests}, \eqref{eq:test-implications-for-first-induction}, and \Cref{lem:inductive-step-0.1-mhat}, we have that \eqref{eq:inductive-result-for-induction-with-tests} holds:
\begin{equation}
\label{eq:inductive-result-for-induction-with-tests}
    \begin{aligned}
    \forall \pi\in\allPi_j\cup\Pi_{0,m+1},\;\; &\Reg_{j}(\pi) \leq 2\Reg_{f_{m+1}}(\pi) + \frac{\eta'_{i,m+1} K}{\gamma_{m+1,i}},\\
    &\Reg_{f_{m+1}}(\pi) \leq 2\Reg_{j}(\pi) + \frac{\eta'_{i,m+1} K}{\gamma_{m+1,i}},
    \end{aligned}
\end{equation}
where $\eta'_{i,m+1}= 5(\gamma_{m}/\gamma_{m,i})(1+2\theta_{i,j})^2 + 2\Tilde{\eta} + 2(1+2\theta_{i,j})\sqrt{(\gamma_{m}/\gamma_{m,i})(\eta_{i,m}+\Tilde{\eta}')}$. To complete our inductive argument, we only need to argue that $\eta'_{i,m+1}\leq \eta_{i,m+1}$; we will now show this.
\begin{equation}
    \begin{aligned}
    \eta'_{i,m+1}&\leq 5\frac{\gamma_{m,1}}{\gamma_{m,i}}(1+2\theta_{i,j})^2 + 56\frac{\gamma_{m,1}}{\gamma_{m,i}} + 2(1+2\theta_{i,j})\sqrt{\frac{\gamma_{m,1}}{\gamma_{m,i}}(\eta_{i,m}+44\frac{\gamma_{m,1}}{\gamma_{m,i}})}\\
    &\stackrel{(i)}{\leq} 75\frac{\gamma_{m,1}}{\gamma_{m,i}}(1+2\theta_{i,j})^2 + 2(1+2\theta_{i,j})\sqrt{\frac{\gamma_{m,1}}{\gamma_{m,i}}\eta_{i,m}}\\
    &\stackrel{(ii)}{=} 75\frac{\gamma_{m,1}}{\gamma_{m,i}}(1+2\theta_{i,j})^2 + 2(1+2\theta_{i,j})\sqrt{\frac{\gamma_{m,1}}{\gamma_{m,i}} 100(1+2\theta_{i,j})^2 \frac{\gamma_{m-1,1}}{\gamma_{m-1,i}} }\\
    &\stackrel{(iii)}{\leq} 95\frac{\gamma_{m,1}}{\gamma_{m,i}}(1+2\theta_{i,j})^2 \leq \eta_{i,m+1},
    \end{aligned}
\end{equation}
where (i) follows from $\sqrt{a+b}\leq \sqrt{a}+\sqrt{b}$ for any $a,b\geq 0$, (ii) follows from substituting $\eta_{i,m} = 100(1+2\theta_{i,j})^2(\gamma_{m-1,1}/\gamma_{m-1,i}) $, and (iii) follows from $\frac{\gamma_{m-1,1}}{\gamma_{m-1,i}}\leq \frac{\gamma_{m,1}}{\gamma_{m,i}}$. This completes the inductive argument.
\end{proof}

In the following lemma, we derive the self-correction property of our algorithm. That is, after a small number of epochs, we correct for effects of potential past under-exploration. This is evident in the factor of $\gamma_{m-1,1}/\gamma_{m-1,i}$ in the bound of \Cref{lem:policy-reg-bound-mhat}, which is improved to a factor of $\frac{\gamma_{m-1,h}}{\gamma_{m-1,i}}\bigg(\frac{\gamma_{\msafealg_{h-1},1}}{\gamma_{\msafealg_{h-1},i}}\bigg)^{1/2^{m-\msafealg_{h-1}-1}}$ in the bound of \Cref{lem:policy-reg-bound-intermediate-mhat}. Increasing epochs without increasing the candidate exploitation parameters helps reduce the term with the exponent, which converges to a constant within a small number of rounds.

\begin{lem}
\label{lem:policy-reg-bound-intermediate-mhat}
Suppose the event $\eventReg$ holds. Consider any two class indices $i,h\in[M']$ such that $h\leq i$. Let $\theta_{i,j}:=\frac{\gamma_{m+1,i}}{\gamma_{m+1,j}}$. For any policy $\pi\in\Tilde{\Pi}_{j}\cup \Pi_{0,m}$, model $f_m\in\{\hatf_m,\hatf_{m,i}\}$, and epoch $m\in [\msafealg_{h-1}+1, \msafealg_i+1]$, we have:
\begin{align*}
    \Reg_{j}(\pi) \leq 2\Reg_{f_{m}}(\pi) + \frac{\eta_{i,h,m} K}{\gamma_{m,i}},\\
    \Reg_{f_{m}}(\pi) \leq 2\Reg_{j}(\pi) + \frac{\eta_{i,h,m} K}{\gamma_{m,i}},
\end{align*}
where $\eta_{i,h,m}=100(1+2\theta_{i,j})^2\frac{\gamma_{m-1,h}}{\gamma_{m-1,i}}\bigg(\frac{\gamma_{\msafealg_{h-1},1}}{\gamma_{\msafealg_{h-1},i}}\bigg)^{1/2^{m-\msafealg_{h-1}-1}}$. Further for $m\in [\msafealg_{h-1}+\lceil\log_2(\log_2(\gamma_{\msafealg_{h-1},1}/\gamma_{\msafealg_{h-1},i}))\rceil, \msafealg_i+1]$, we have $\eta_{i,h,m}\leq 400(1+2\theta_{i,j})^2\frac{\gamma_{m-1,h}}{\gamma_{m-1,i}}$. 
\end{lem}
\begin{proof}
Consider any two class indices $i,h\in[M']$ such that $h\leq i$. We will prove the required bound by induction. The bound for the base case $m=\msafealg_{h-1}+1$ follows from \Cref{lem:policy-reg-bound-mhat}. Suppose the bound in \Cref{lem:policy-reg-bound-intermediate-mhat} holds for class indices $i,h$ and for some epoch $m\in [\msafealg_{h-1}+1, \msafealg_i]$. Let $f_{m+1}\in\{\hatf_{m+1},\hatf_{m+1,i}\}$. To complete our inductive argument, we will show the bound in \Cref{lem:policy-reg-bound-intermediate-mhat} holds for class indices $i,h$ and epoch $m+1$. Now, from \Cref{lem:new_mtest_implication} and \Cref{lem:policy-reg-bound-wrt-old-model-test-implication} we have \eqref{eq:test-implications-for-second-induction} holds.
\begin{equation}
\label{eq:test-implications-for-second-induction}
    \begin{aligned}
        &\Reg_{\hatf_{m+1}}(\pi_{\hat{f}_{m+1,i}}) \leq \frac{\Tilde{\eta} K}{\gamma_{m+1,i}},\\
        & \Reg_{\hatf_{m}}(\pi) \leq 4\Reg_{f_{m+1}}(\pi) + \frac{\Tilde{\eta}' K}{\gamma_{m,i}}, \quad \forall\pi\in \{\pi_{\hatf_{m+1}},\pi_{\hatf_{m+1},i},p_{m+1,1},\dots,p_{m+1,M'} \},\\
        & \Tilde{\eta} = 28\frac{\gamma_{m,h}}{\gamma_{m,i}}\bigg(\frac{\gamma_{\msafealg_{h-1},1}}{\gamma_{\msafealg_{h-1},i}}\bigg)^{1/2^{m-\msafealg_{h-1}}},\;\; \Tilde{\eta}' = 44\frac{\gamma_{m,h}}{\gamma_{m,i}}\bigg(\frac{\gamma_{\msafealg_{h-1},1}}{\gamma_{\msafealg_{h-1},i}}\bigg)^{1/2^{m-\msafealg_{h-1}-1}}.
    \end{aligned}
\end{equation}
Now from \eqref{eq:test-implications-for-second-induction}, inductive hypothesis (bounds in \Cref{lem:policy-reg-bound-intermediate-mhat} hold for epoch $m$) and \Cref{lem:inductive-step-0.1-mhat}, we have \eqref{eq:inductive-result-for-second-induction-with-tests} holds. \begin{equation}
\label{eq:inductive-result-for-second-induction-with-tests}
    \begin{aligned}
    \forall \pi\in\allPi_j\cup\Pi_{0,m+1},\;\; &\Reg_{j}(\pi) \leq 2\Reg_{f_{m+1}}(\pi) + \frac{\eta'_{i,h,m+1} K}{\gamma_{m+1,i}},\\
    &\Reg_{f_{m+1}}(\pi) \leq 2\Reg_{j}(\pi) + \frac{\eta'_{i,h,m+1} K}{\gamma_{m+1,i}},
    \end{aligned}
\end{equation}
where $\eta'_{i,h,m+1}= 5(\gamma_{m}/\gamma_{m,i})(1+2\theta_{i,j})^2 + 2\Tilde{\eta} + 2(1+2\theta_{i,j})\sqrt{(\gamma_{m}/\gamma_{m,i})(\eta_{i,h,m}+\Tilde{\eta}')}$. Note that since $m\geq\msafealg_{h-1}+1$, we have $\gamma_m\geq \gamma_{m,h}$. To complete our inductive argument, we only need to argue that $\eta'_{i,h,m+1}\leq \eta_{i,h,m+1}$; we will now show this.
\begin{equation}
    \begin{aligned}
    \eta'_{i,h,m+1} &\leq 5\frac{\gamma_{m,h}}{\gamma_{m,i}}(1+2\theta_{i,j})^2 + 56\frac{\gamma_{m,h}}{\gamma_{m,i}}\bigg(\frac{\gamma_{\msafealg_{h-1},1}}{\gamma_{\msafealg_{h-1},i}}\bigg)^{1/2^{m-\msafealg_{h-1}}} \\
    &+ 2(1+2\theta_{i,j})\sqrt{\frac{\gamma_{m,h}}{\gamma_{m,i}}\bigg(\eta_{i,h,m}+44\frac{\gamma_{m,h}}{\gamma_{m,i}}\bigg(\frac{\gamma_{\msafealg_{h-1},1}}{\gamma_{\msafealg_{h-1},i}}\bigg)^{1/2^{m-\msafealg_{h-1}-1}}\bigg)}\\
    &\stackrel{(i)}{\leq} 75(1+2\theta_{i,j})^2\frac{\gamma_{m,h}}{\gamma_{m,i}}\bigg(\frac{\gamma_{\msafealg_{h-1},1}}{\gamma_{\msafealg_{h-1},i}}\bigg)^{1/2^{m-\msafealg_{h-1}}} + 2(1+2\theta_{i,j})\sqrt{\frac{\gamma_{m,h}}{\gamma_{m,i}}\eta_{i,h,m}}\\
    &\stackrel{(ii)}{=} 75(1+2\theta_{i,j})^2\frac{\gamma_{m,h}}{\gamma_{m,i}}\bigg(\frac{\gamma_{\msafealg_{h-1},1}}{\gamma_{\msafealg_{h-1},i}}\bigg)^{1/2^{m-\msafealg_{h-1}}}\\ 
    &+ 2(1+2\theta_{i,j})\sqrt{\frac{\gamma_{m,h}}{\gamma_{m,i}} 100(1+2\theta_{i,j})^2\frac{\gamma_{m-1,h}}{\gamma_{m-1,i}}\bigg(\frac{\gamma_{\msafealg_{h-1},1}}{\gamma_{\msafealg_{h-1},i}}\bigg)^{1/2^{m-\msafealg_{h-1}-1}} }\\
    &\stackrel{(iii)}{\leq} 95(1+2\theta_{i,j})^2\frac{\gamma_{m,h}}{\gamma_{m,i}}\bigg(\frac{\gamma_{\msafealg_{h-1},1}}{\gamma_{\msafealg_{h-1},i}}\bigg)^{1/2^{m-\msafealg_{h-1}}} \leq \eta_{i,h,m+1},
    \end{aligned}
\end{equation}
where (i) follows from $\sqrt{a+b}\leq \sqrt{a}+\sqrt{b}$ for any $a,b\geq 0$, (ii) follows from substituting $\eta_{i,h,m}=100(1+2\theta_{i,j})^2\frac{\gamma_{m-1,h}}{\gamma_{m-1,i}}\bigg(\frac{\gamma_{\msafealg_{h-1},1}}{\gamma_{\msafealg_{h-1},i}}\bigg)^{1/2^{m-\msafealg_{h-1}-1}}$, and (iii) follows from $\frac{\gamma_{m-1,h}}{\gamma_{m-1,i}}\leq \frac{\gamma_{m,h}}{\gamma_{m,i}}$. This completes the inductive argument. Finally for $m\in [\msafealg_{h-1}+\lceil\log_2(\log_2(\gamma_{\msafealg_{h-1},1}/\gamma_{\msafealg_{h-1},i}))\rceil, \msafealg_i+1]$, we have \eqref{eq:self-correction-time}:
\begin{equation}
\label{eq:self-correction-time}
\begin{aligned}
    \eta_{i,h,m}&\leq 100(1+2\theta_{i,j})^2\frac{\gamma_{m-1,h}}{\gamma_{m-1,i}}\bigg(\frac{\gamma_{\msafealg_{h-1},1}}{\gamma_{\msafealg_{h-1},i}}\bigg)^{2/2^{\log_2(\log_2(\gamma_{\msafealg_{h-1},1}/\gamma_{\msafealg_{h-1},i}))}}\\ 
    &\leq 400(1+2\theta_{i,j})^2\frac{\gamma_{m-1,h}}{\gamma_{m-1,i}}.
\end{aligned}
\end{equation}

\end{proof}

\section{BOUNDING TIME TO DETECTION OF MISSPECIFICATION}
\label{app:bounding-time-to-misspecification}

In this section, we bound the number of rounds to determine whether class $\allF_i$ is misspecified. In particular, under \Cref{ass:gradual,ass:opt-value-jump}, we show that that misspecification for class $\allF_i$ is detected before the corresponding policy class bias dominates the corresponding variance. Unlike previous sections, the analysis in this section relies on \Cref{ass:gradual,ass:opt-value-jump}.

\textbf{Proof outline:} \Cref{lem:ass6implies5} first derives a minimum direct method evaluation error for models in $\allF_i$ in terms of $\Delta_i$. \Cref{lem:relate-policy-bias-with-opt-jump} allows us to re-write this error in terms of policy class bias $\beta_i$. The rest of our analysis bounds the number of rounds required to detect this error. \Cref{lem:IH-for-bounding-miss-time-based-on-opt-gap} bounds the length of the epoch where misspecification is detected for class $\allF_i$ in terms of the length of the epoch where misspecification is detected for class $\allF_{i-1}$. \Cref{lem:bound-epoch-terminal-in-terms-of-epoch-length} bounds the last round in an epoch in terms of its epoch length. Hence, \Cref{cor:IH-for-bounding-miss-time-based-on-opt-gap} uses these results to bound the time to detect misspecification for class $\allF_i$.

\begin{lem}
\label{lem:ass6implies5}
Suppose \Cref{ass:opt-value-jump} holds. Consider some $i< j\in[M']$ such that $B_{i}>0$ and $d_j\leq\omega d_{i}$. Further, consider any reward model $f\in \Tilde{\F}_i$. We then have:
\begin{equation}
\label{eq:ass6implies5}
    \exists \pi\in\{\pi_{f},\pi_j^*\},\quad  \Delta_i/2\leq|R(\pi)-R_{f}(\pi)|.
\end{equation}

\end{lem}
\begin{proof}
From \Cref{ass:opt-value-jump}, we have \eqref{eq:assump6} holds.
\begin{align}\label{eq:assump6}
\Delta_{i} := R(\pi_j^*) - R(\pi^*_{i}).
\end{align}
Suppose for contradiction, assume that \eqref{eq:contra} holds.
\begin{equation}
\label{eq:contra}
    \Delta_i/2>|R(\pi)-R_{f}(\pi)|\quad\forall\quad\pi\in\{\pi_{f},\pi_j^*\}.
\end{equation}
We can decompose \eqref{eq:contra} to obtain the following two relations:
\begin{equation}
\label{eq:contra-imp}
    \begin{aligned}
    &R(\pi_{f})+\Delta_i/2> R_{f}(\pi_{f})\\
    &R(\pi^*_j)-\Tilde{\Delta}_i/2< R_{f}(\pi^*_j).
    \end{aligned}
\end{equation}
By the definitions of $\pi_f$, we have $R_{f}(\pi_f)\geq R_{f}(\pi_j^*)$. 
Together with \eqref{eq:contra-imp}, this gives us:
\begin{equation}
\begin{aligned}
    &R(\pi_i^*)+\Delta_i/2 \geq R(\pi_{f})+\Delta_i/2> R_f(\pi_f)\geq R_f(\pi_j^*)> R(\pi_j^*)-\Delta_i/2 \\
    & \implies R(\pi_j^*) - R(\pi_i^*) < \Delta_i.
\end{aligned}
\end{equation}
This contradicts \eqref{eq:assump6}, hence \eqref{eq:contra} must be false. Therefore \eqref{eq:ass6implies5} holds.
\end{proof}

\begin{lem}
\label{lem:relate-policy-bias-with-opt-jump}
    Suppose \Cref{ass:realizability} and \Cref{ass:opt-value-jump} hold. Then for any $i<i^*$, we have $\Delta_i\geq \beta_i/(i^*-i) \geq \beta_i/\log_2(\alld_{i^*})$.
\end{lem}
\begin{proof}
The proof is fairly straightforward:
    \begin{equation}
        \begin{aligned}
            \beta_i=&R(\pi^*_{i^*}) - R(\pi_i) \\ 
            =&\sum_{j=i}^{i^*-1} \Big(R(\pi^*_{j+1}) - R(\pi_j^*) \Big) \leq (i^*-i) \Delta_i.
        \end{aligned}
    \end{equation}
The first equality follows from definition of $\beta_i$ and the last inequality follows from \Cref{ass:opt-value-jump}. Hence, we have $\Delta_i\geq \beta_i/(i^*-i)$. Now to complete the proof, note that $(i^*-i)\leq i^* \leq \log_2(\alld_{i^*})$
\end{proof}

\begin{lem}
\label{lem:IH-for-bounding-miss-time-based-on-opt-gap}
Suppose the events $\eventReg$ and $\eventExpEval$ hold. Suppose also that \Cref{ass:opt-value-jump} holds. Consider any $i\in[M']$ such that $B_i>0$. There exists a constant $C_2$ such that the following holds:
    \begin{equation}
    \label{eq:IH-for-bounding-miss-time-based-on-opt-gap}
    \tau_{\msafealg_i}-\tau_{\msafealg_{i}-1}\leq ( \tau_{\msafealg_{i-1}}-\tau_{\msafealg_{i-1}-1})+C_2\bigg(\frac{K}{\Delta_i^2}\bigg)^{1/\rho}\omega^2\alld_i\ln(6M^3T^2/\delta).%
\end{equation}
\end{lem}
\begin{proof}
We will prove \Cref{lem:IH-for-bounding-miss-time-based-on-opt-gap} via induction. Note that the base case is trivially satisfied by defining $\tau_{\msafealg_l}=0$ for any $l\leq 0$. For our inductive hypothesis, suppose the statement in \Cref{lem:IH-for-bounding-miss-time-based-on-opt-gap} holds for class index $i-1\in[M'-1]$. To complete our inductive argument, we will show that the statement in \Cref{lem:IH-for-bounding-miss-time-based-on-opt-gap} holds for class index $i\in[M']$. 

We split our analysis into two cases, a trivial case and a more involved case. The first case is $\msafealg_i\leq\msafealg_{i-1}+l$, where $l=\log_2\log_2(\gamma_{m,1}/\gamma_{m,i_{m+1}})$. In this case, the algorithm is still undergoing self-correction of the learning rates following detection of misspecification in model class $\msafealg_{i-1}$, and so the epoch lengths are not yet doubling. Therefore, we have that $\tau_{\msafealg_i}-\tau_{\msafealg_{i}-1}=  \tau_{\msafealg_{i-1}}-\tau_{\msafealg_{i-1}-1}$. Hence, the inequality we want to show \eqref{eq:IH-for-bounding-miss-time-based-on-opt-gap} is trivially satisfied. The second case is $\msafealg_i>\msafealg_{i-1}+l$. Let $m=\msafealg_i-1$, and let $j$ be the largest index such that $\Tilde{d}_j\leq \omega\Tilde{d}_i$. By \Cref{ass:opt-value-jump} and \Cref{lem:ass6implies5}, we know that the following \eqref{eq:ass6-implication-on-evaluation-error} holds.
\begin{equation}
\label{eq:ass6-implication-on-evaluation-error}
    \exists\quad\pi\in\{\pi_{\hat{f}_{m+1,i}},\pi_j^*\},\quad \Delta_i/2\leq|R(\pi)-R_{\hat{f}_{m+1,i}}(\pi)|.
\end{equation}
Let $\theta_{i,j}:=\frac{\gamma_{m+1,i}}{\gamma_{m+1,j}}$, and let $\alpha>0$ be a positive constant that we will fix later. Since $m\leq \msafealg_i$, from \Cref{lem:combinedprevious} and \eqref{eq:ass6-implication-on-evaluation-error}, we have \eqref{eq:combine-ass6-evaluation-error-with-evaluation-bound} holds.
\begin{equation}
\label{eq:combine-ass6-evaluation-error-with-evaluation-bound}
    \begin{aligned}
        \frac{\Delta_i}{2} &\leq \bigg( \frac{1+2\theta_{i,j}}{\alpha} + \frac{(1+2\theta_{i,j})\alpha}{16} + \frac{(2\theta_{i,j}^2+2(1+\theta_{i,j})^2/\alpha+2\theta_{i,j})\gamma_m}{\gamma_{m+1,i}} + 2\theta_{i,j} \bigg)\frac{K}{\gamma_{m+1,i}} \\
     &\quad + \frac{(1+2\theta_{i,j})\gamma_m}{\alpha \gamma_{m+1,i}}\max_{\pi\in\{\pi_{\hat{f}_{m+1,i}},\pi_j^*\}}\Reg_{\hatf_m}(\pi).
    \end{aligned}
\end{equation}
We will now bound $\max_{\pi\in\{\pi_{\hat{f}_{m+1,i}},\pi_j^*\}}\Reg_{\hatf_m}(\pi)$. From \Cref{lem:policy-reg-bound-intermediate-mhat}, the fact that $\Reg_j(\pi^*_j)=0=\Reg_{\hatf_{m+1,i}}(\pi_{\hatf_{m+1,i}})$, $m> \msafealg_{i-1}+l$, and $\gamma_{m+1,i}\geq \gamma_{m,i}$ -- we have \eqref{eq:policy-reg-implications-on-good-policies} holds.
\begin{equation}
\label{eq:policy-reg-implications-on-good-policies}
    \begin{aligned}
        &\Reg_{\hatf_m}(\pi^*_j) \leq %
        400(1+2\theta_{i,j})^2\frac{K}{\gamma_{m,i}},\\
        &\Reg_{\hatf_m}(\pi_{\hatf_{m+1,i}})\leq 2\Reg_j(\pi_{\hatf_{m+1,i}}) + 400(1+2\theta_{i,j})^2\frac{K}{\gamma_{m,i}} \leq 1200(1+2\theta_{i,j})^2\frac{K}{\gamma_{m,i}}.
    \end{aligned}
\end{equation}
Hence, combining \eqref{eq:combine-ass6-evaluation-error-with-evaluation-bound} and \eqref{eq:policy-reg-implications-on-good-policies}, we have \eqref{eq:opt-gap-bound} holds.
\begin{equation}
\label{eq:opt-gap-bound}
    \begin{aligned}
        \frac{\Delta_i}{2} &\leq \bigg( \frac{1+2\theta_{i,j}}{\alpha} + \frac{(1+2\theta_{i,j})\alpha}{16} + \frac{(2\theta_{i,j}^2+2(1+\theta_{i,j})^2/\alpha+2\theta_{i,j})\gamma_m}{\gamma_{m+1,i}} + 2\theta_{i,j} \bigg)\frac{K}{\gamma_{m+1,i}} \\
        &\quad + \frac{(1+2\theta_{i,j})\gamma_m}{\alpha \gamma_{m+1,i}}1200(1+2\theta_{i,j})^2\frac{K}{\gamma_{m,i}}.
    \end{aligned}
\end{equation}
We will now simplify \eqref{eq:opt-gap-bound}. Since $m\in(\msafealg_{i-1},\msafealg_i)$, we have $\gamma_m=\gamma_{m,i}$. Also note that $\gamma_{m,i}\leq \gamma_{m+1,i}$. Now by choosing $\alpha=128(1+2\theta_{i,j})$, we get the following simplification of \eqref{eq:opt-gap-bound}:
\begin{equation}
    \begin{aligned}
        &\frac{\Delta_i}{2} \leq \bigg( \frac{1}{128} + 8(1+2\theta_{i,j})^2 + 2\theta_{i,j}^2+(1+\theta_{i,j})/64+4\theta_{i,j} + 10(1+2\theta_{i,j})^2  \bigg)\frac{K}{\gamma_{m,i}} \\
        \implies & \Delta_i\leq 40(1+2\theta_{i,j})^2\frac{K}{\gamma_{m,i}} %
        \leq 360\;\theta_{i,j}^2\sqrt{8C_1 K\bigg(\frac{\alld_i\ln(6M^3T^2/\delta)}{(\tau_{m-1} - \tau_{m-2})/2}\bigg)^{\rho}}\\
        \implies & \tau_{\msafealg_i} - \tau_{\msafealg_{i}-1} \leq 4(\tau_{m-1}-\tau_{m-2}) \leq 8\cdot (8C_1 \cdot 360^2)^{1/\rho} \frac{K^{1/\rho}\theta_{i,j}^{4/\rho}\alld_i\ln(6M^3T^2/\delta)}{\Delta_i^{2/\rho}}.
    \end{aligned}
\end{equation}
By substituting $\theta_{i,j}=\frac{\gamma_{m+1,i}}{\gamma_{m+1,j}}=\sqrt{\frac{d_j^{\rho}}{d_i^{\rho}}}\leq \omega^{\rho/2}$, we completes the proof of the inductive step for the second case. Hence this completes the proof of the inductive argument.
\end{proof}

\begin{lem}
\label{lem:bound-epoch-terminal-in-terms-of-epoch-length}
    For any epoch $m$, we have $\tau_m\leq 2l(\tau_m-\tau_{m-1})$, where $l=\lceil \log_2\log_2(\gamma_{m,1}/\gamma_{m,i_{m}}) \rceil$.
\end{lem}
\begin{proof}
In $\CVBandit$, at any epoch, the length of the next epoch is either equal to the length of the current epoch (if misspecification was recently detected) or two times the length of the current epoch (if it has been a while since misspecification was detected). By definition, within the first $m$ epochs, misspecification is not detected for any class index in $[i_m,M']$. Hence for any epoch upto epoch $m$, at most $l$ consecutive epochs have the same length. Since doubling more frequently would enable larger epoch length, given a bound on the length of epoch $m$, $\tau_m$ is the largest when epoch lengths only double once every $l$ epochs. Hence $\tau_m\leq l\cdot 2(\tau_m-\tau_{m-1})$.

\end{proof}

\begin{cor}
\label{cor:IH-for-bounding-miss-time-based-on-opt-gap}
    There exists a constant $C_2$ such that the following holds with probability at least $1-\delta$. Suppose \Cref{ass:opt-value-jump} holds. Consider any $i\in[M']$ such that $B_i>0$. We then have that:
    \begin{equation}
    \label{eq:cor-IH-for-bounding-miss-time-based-on-opt-gap}
    \begin{aligned}
        \tau_{\msafealg_i} \leq 4lC_2\bigg(\frac{K}{\Delta_i^2}\bigg)^{1/\rho}\omega^2\alld_i\ln(6M^3T^2/\delta),
    \end{aligned}
\end{equation}
where $l=\lceil \log_2\log_2(\gamma_{\msafealg_i,1}/\gamma_{\msafealg_i,i}) \rceil$.
\end{cor}
\begin{proof}
    For any $i\in[M']$, we have \eqref{eq:telescoping-IH-for-bounding-miss-time-based-on-opt-gap} holds.
    \begin{equation}
    \label{eq:telescoping-IH-for-bounding-miss-time-based-on-opt-gap}
    \begin{aligned}
    & \tau_{\msafealg_i} - \tau_{\msafealg_{i}-1}\\
    \stackrel{(i)}{=}&\sum_{h=1}^i((\tau_{\msafealg_h}-\tau_{\msafealg_{h}-1})- ( \tau_{\msafealg_{h-1}}-\tau_{\msafealg_{h-1}-1}))\\
    \stackrel{(ii)}{\leq} &\sum_{h=1}^i%
    C_2\bigg(\frac{K}{\Delta_h^2}\bigg)^{1/\rho}\omega^2\alld_h\ln(6M^3T^2/\delta) \\
    \stackrel{(ii)}{\leq} & C_2\bigg(\frac{K}{\Delta_i^2}\bigg)^{1/\rho}\omega^2\alld_i\ln(6M^3T^2/\delta)\sum_{h=1}^i\frac{1}{2^{(i-h)}} \leq 2C_2\bigg(\frac{K}{\Delta_i^2}\bigg)^{1/\rho}\omega^2\alld_i\ln(6M^3T^2/\delta),
    \end{aligned}
\end{equation}
where (i) follows from the fact that $\tau_{\msafealg_h}=0$ for all $h\leq 0$, (ii) follows from \Cref{lem:IH-for-bounding-miss-time-based-on-opt-gap}, and (iii) follows from the fact that $\alld_{h+1}\geq 2 \alld_h$ by construction and that $\Delta_{h+1}\leq \Delta_{h}$ (\Cref{ass:opt-value-jump}). The result now follows by combining \eqref{eq:telescoping-IH-for-bounding-miss-time-based-on-opt-gap} and \Cref{lem:bound-epoch-terminal-in-terms-of-epoch-length}.

\end{proof}

\section{FINAL REGRET GUARANTEES}\label{app:final-regret-guarantees}
In this section, we derive our final regret bounds by utilizing the analysis in \Cref{lem:policy-reg-bound-intermediate-mhat} and \Cref{cor:IH-for-bounding-miss-time-based-on-opt-gap}. \Cref{lem:policy-reg-bound-intermediate-mhat} allows us to bound exploration regret until $\msafealg_i$ (the epoch where misspecification is detected with respect to class $\allF_i$). \Cref{cor:IH-for-bounding-miss-time-based-on-opt-gap} bounds the time to detect misspecfication for various classes.

\thmmain*

\begin{proof}
From \Cref{lem:high-prob-eventReg} and \Cref{lem:high-prob-eventPolicyVal}, we have that both $\eventReg$ and $\eventMT$ hold with probability at least $1-\delta$.  We now bound the expected cumulative regret up to round $T$ while assuming that this high-probability event holds. 

Let $\theta_{i,j}=(\alld_j/\alld_i)^{\rho/2}$ and let $m'\geq 1$ be the first epoch after detecting misspecification with respect to class $\allF_i$ when we are guaranteed to self-correct (see \Cref{lem:policy-reg-bound-intermediate-mhat}) for possibly under-exploring with respect to the class $i$. That is, if $i=1$ let $m'=1$ and if $i>1$ let $m'=\msafealg_{i-1}+\lceil\log_2(\log_2(\gamma_{\msafealg_{i-1},1}/\gamma_{\msafealg_{i-1},i}))\rceil)$. Note that for $i>1$, $\gamma_{\msafealg_{i-1},1}/\gamma_{\msafealg_{i-1},i}=(\alld_i/\alld_1)^{\rho/2}\leq \sqrt{\alld_i}$. Hence, $\tau_{m'}\leq \max(\tau_1,\tau_{\msafealg_{i-1}}\log_2(\alld_i))$.

Now, from \Cref{lem:policy-reg-bound-intermediate-mhat}, we have that:
\begin{equation}
\label{eq:bound-expected-regret}
\begin{aligned}
    &\CReg_T := \sum_{t=1}^T\Reg_{f^*}(p_{m(t)}) \stackrel{(i)}{=} \beta_jT + \sum_{t=1}^T\Reg_{j}(p_{m(t)}) \leq \beta_jT + \tau_{m'} + \sum_{t=\tau_{m'}+1}^T\Reg_{j}(p_{m(t)})\\
    &\stackrel{(ii)}{\leq} \beta_jT + \tau_{m'} + \sum_{t=\tau_{m'}+1}^T\bigg(2\Reg_{\hatf_{m(t)}}(p_{m(t)}) +\frac{400(1+2\theta_{i,j})^2 K}{\gamma_{m(t),i}} \bigg)\\
    &\stackrel{(iii)}{\leq} \beta_jT + \tau_{m'} + \sum_{t=\tau_{m'}+1}^T\bigg(\frac{2K}{\gamma_{m(t),j}} +\frac{3600\theta_{i,j}^2 K}{\gamma_{m(t),i}} \bigg) \stackrel{(iv)}{\leq} \beta_jT + \tau_{m'} + \sum_{t=\tau_{m'}+1}^T\frac{3602\theta_{i,j} K}{\gamma_{m(t),j}}\\
    &\stackrel{(v)}{\leq} \beta_jT + \tau_{m'} + \sum_{t=\tau_{m'}+1}^T 3602\theta_{i,j}\sqrt{8KC_1} \bigg(\frac{\alld_j\ln(6M^3T^2/\delta)}{(\tau_{m(t)-1}-\tau_{m(t)-2})/2}\bigg)^{\rho/2} \\
    &\stackrel{(vi)}{\leq} \beta_jT + \tau_1 + \tau_{\msafealg_{i-1}}\log_2(\alld_i) + (3602\cdot 2^{\rho}\sqrt{8 C_1})\cdot \theta_{i,j}\sqrt{K}(\alld_j\ln(6M^3T^2/\delta))^{\rho/2} \sum_{m=m'+1}^{m(T)}  \frac{\tau_{m(t)}-\tau_{m(t)-1}}{(\tau_{m(t)}-\tau_{m(t)-1})^{\rho/2}} \\
    &\stackrel{(vii)}{\leq} \beta_jT + \tau_{1} + 4\log_2\log_2(\alld_{i-1})  C_2\bigg(\frac{K}{\Delta_{i-1}^2}\bigg)^{1/\rho}\omega^2\alld_{i-1}\ln(6M^3T^2/\delta)\log_2(\alld_i)\\ 
    &\;\;\;+ (3602\cdot 2^{\rho}\sqrt{8 C_1})\cdot \theta_{i,j}\sqrt{K}(\alld_j\ln(6M^3T^2/\delta))^{\rho/2} T^{1-\rho/2} \log_2\log_2(\alld_{j}) \log_2T,
\end{aligned}
\end{equation}
where (i) follows from $\beta_j:=R(\pi_{f^*})-R(\pi^*_j)$; (ii) follows from \Cref{lem:policy-reg-bound-intermediate-mhat}; (iii) follows from \Cref{lem:QmRegEst}, the fact that $\gamma_m\geq \gamma_{m,j}$ (misspecification is not detected for class $\allF_j$), and $(1+2\theta_{i,j})^2\leq 9\theta_{i,j}^2$; (iv) follows from $\theta_{i,j}=\sqrt{\gamma_{m(t),i}/\gamma_{m(t),j}}$; (v) follows from our choice of $\gamma_{m(t),j}$; (vi) follows from $\tau_{m'}\leq \max(\tau_1,\tau_{\msafealg_{i-1}}\log_2(\alld_i))$ and length of epoch $m(t)$ is at most double the size of length for epoch $m(t)-1$; (vii) follows from the bound on $\tau_{\msafealg_{i-1}}$ from \Cref{cor:IH-for-bounding-miss-time-based-on-opt-gap}, the fact that $(\tau_{m(t)}-\tau_{m(t)-1})^{1-\rho/2}\leq T^{1-\rho/2}$ for any $t\leq T$, and the fact that $m(T)\leq \log_2\log_2(\alld_{j}) \log_2T$ (since misspecification is not detected for $\allF_j$ and hence fraction of non-doubling rounds is at most $\log_2\log_2(\alld_{j})$. Finally, the regret guarantee follows from additionally noting that $\Delta_{i-1}\geq \beta_{i-1}/\log_2(\alld_{i^*})$ (\Cref{lem:relate-policy-bias-with-opt-jump}). 

Also, since misspecification with respect to class $\allF_j$ is not detected until epoch $\msafe_j$ (see \Cref{lem:explicit-validation}, \Cref{lem:new_mtest}, and \Cref{lem:policy-reg-bound-wrt-old-model-test}), we know misspecification is not detected for at least $\Omega(\alld_j/B_j^{1/\rho})$ rounds.

\end{proof}

\section{ADDITIONAL DETAILS}
\label{app:additional-details}

\subsection{Constructing An Estimation Oracle}
\label{app:construct-estimation-oracle}

For completeness, we outline one of many approaches to construct an oracle that achieves the ``fast rates'' of \Cref{ass:model-selection}. Consider a sequence of classes $\F_1,\F_2,\dots,\F_i$ with VC subgraph dimensions of $d_1,d_2,\dots,d_i$ respectively. Consider a probability kernel $p$ and a natural number $n$. Consider $n$ independently and identically drawn samples from the distribution $D(p)$. Let $\hatf_j$ be an estimator in $\F_j$ that minimizes empirical squared error loss over the first $\lceil n/2\rceil$ samples. For any $\zeta\in(0,1)$, from fairly standard arguments based on local Rademacher complexities \citep[see Theorem 5.2 and example 3 in chapter 5 of ][]{koltchinskii2011oracle}, with probability $1-\zeta/(2i)$ we have:
\begin{equation}
\label{eq:model-estimation-rate}
\begin{aligned}
    \E_{x\sim D_{\Xscript}}\E_{a\sim p(\cdot|x)}[ (\hatf_j(x, a) - f^*(x,a))^2 ] \leq (1+\epsilon)b_j(p) + \ordO\bigg(\frac{d_j\ln(n i/\zeta)}{n}\bigg),
\end{aligned}
\end{equation}
where $\epsilon > 0$ is any fixed constant.\footnote{Note that $\epsilon$ is zero when $\F_j$'s are convex or well-specified.} Now let $\hatf$ be an estimator in the set $\{\hatf_1,\hatf_2,\dots,\hatf_i\}$ that minimizes empirical squared error loss over the remaining $\lfloor n/2\rfloor$ samples. Again from using the same arguments based on localization \citep[e.g.][]{mitchell2009general,koltchinskii2011oracle}, with probability $1-\zeta/2$ we have:
\begin{equation}
\label{eq:aggregation-rate}
\begin{aligned}
    &\E_{x\sim D_{\Xscript}}\E_{a\sim p(\cdot|x)}[ (\hatf(x, a) - f^*(x,a))^2 ]\\
    &\leq (1+\epsilon')\min_{j\in[i]} \E_{x\sim D_{\Xscript}}\E_{a\sim p(\cdot|x)} [(f_j(x,a)-f^*(x,a))^2] + \ordO\bigg(\frac{\ln(i/\zeta)}{n}\bigg),
\end{aligned}
\end{equation}
where $\epsilon' > 0$ is any fixed constant. By combining \eqref{eq:model-estimation-rate} and \eqref{eq:aggregation-rate}, with probability $1-\zeta$, we have:
\begin{equation}
\label{eq:combined-estimation-rate}
\begin{aligned}
    \E_{x\sim D_{\Xscript}}\E_{a\sim p(\cdot|x)}[ (\hatf(x, a) - f^*(x,a))^2 ] \leq (1+\epsilon)(1+\epsilon')b_j(p) + \ordO\bigg(\frac{d_j\ln(n i/\zeta)}{n}\bigg).
\end{aligned}
\end{equation}
This completes our outline for the construction of an oracle that satisfies \Cref{ass:model-selection}. The approach described here is based on using empirical risk minimization on training and validation sets. Other approaches one could use include aggregation algorithms \citep[see][ and references therein]{lecue2014optimal}, penalized regression \citep[see relevant chapters in ][]{koltchinskii2011oracle,wainwright2019high}, cross validation, etc.

\subsection{Constructing an Implementation of a Misspecification Test Oracle}
\label{sec:constructing-misspecification-oracle}

\Cref{def:MTest} describes a computational oracle to test/verify several inequalities. The test relies on several parameters, we can search over $\alpha>0$ via single variable optimization methods and search over $i,j\in[M']$ and $f\in\{\hatf_{m+1},\hatf_{m+1,i}\}$ via enumeration. The number of policies in $\Pi_{0,m+1,i}$ are few and the corresponding inequalities can be easily verified. Hence, we primarily need to argue that the inequalities corresponding to $\allPi_j$ in the ``policy-based misspecification test" can be verified computationally.

For any choice of $\alpha>0,i,j\in[M'],f\in\{\hatf_{m+1},\hatf_{m+1,i} \}$, we restate the ``policy-based misspecification test" that is used at the end of epoch $m$ and argue how this test can be verified via two calls to a cost sensitive classification solver. First, let us restate the test as a maximization problem for a given set of parameters (here $\lambda_{i,j,\alpha}$ serves as a short-hand for the policy independent terms):
\begin{equation}
\label{eq:misspecification-test-restated}
    \begin{aligned}
    &\max_{\pi\in\Pi_j}{|\hatR_{m+1,f}(\pi)-\hatR_{m+1}(\pi)|}  - \frac{(1+\theta_{i,j})\gamma_m}{\alpha \gamma_{m+1,i}}\hatReg_{m+1,\hatf_m}(\pi)\\
    &\quad\leq \bigg( \frac{1+\theta_{i,j}}{\alpha} + \frac{(1+\theta_{i,j})\alpha}{16} + \frac{(2\theta_{i,j}^2+(1+\theta_{i,j})^2/\alpha)\gamma_m}{\gamma_{m+1,i}} + \theta_{i,j} \bigg)\frac{K}{\gamma_{m+1,i}} =: \lambda_{i,j,\alpha}
    \end{aligned}
\end{equation}
We are interested in calculating the value of the maximization problem in \eqref{eq:misspecification-test-restated}. To calculate this maximum, we need to fix our estimators. Let $\hatR_{m+1,f}(\pi):=\frac{1}{|S_{m,\ho}|}\sum_{t\in S_{m,\ho}}f(x_t,\pi(x_t))=\frac{1}{|S_{m,\ho}|}\sum_{t\in S_{m,\ho}}\E_{a\sim \pi(\cdot|x_t)} f(x_t,a)$ for any policy $\pi$ and reward model $f$, which is the only obvious estimator we could think off for $R_f(\pi)$. Also let us use IPS estimaton for policy evaluation (the same argument works for DR), $\hatR_{m+1}(\pi):=\frac{1}{|S_{m,\ho}|}\sum_{t\in S_{m,\ho}}\frac{\pi(a_t|x_t)r_t(a_t)}{p_m(a_t|x_t)}$. \footnote{Up to constant factors, IPS estimators give us the best rates in \Cref{ass:policy-val-evaluation} with finite classes. These estimators are also used in several contextual bandit papers \citep[e.g.,][]{agarwal2014taming}.} Note that the value of the maximization problem in \eqref{eq:misspecification-test-restated} is equal to $\max(L_1,L_2)$, where $\{L_i|i\in[2]\}$ are defined as follows.
\begin{equation}
\label{eq:misspecification-test-decomposed}
    \begin{aligned}
             &L_1:=\max_{\pi\in\Pi_j}{\hatR_{m+1,f}(\pi)-\hatR_{m+1}(\pi)}  - \frac{(1+\theta_{i,j})\gamma_m}{\alpha \gamma_{m+1,i}}\hatReg_{m+1,\hatf_m}(\pi)\\ 
             &L_2:=\max_{\pi\in\Pi_j}{\hatR_{m+1}(\pi)-\hatR_{m+1,f}(\pi)}  - \frac{(1+\theta_{i,j})\gamma_m}{\alpha \gamma_{m+1,i}}\hatReg_{m+1,\hatf_m}(\pi).
    \end{aligned}
\end{equation}
Substituting value of these estimators for $L_1$ and $L_2$, we get.
\begin{equation}
\label{eq:misspecification-test-decomposed-substituted}
    \begin{aligned}
             &L_1=\max_{\pi\in\Pi_j} \sum_{t\in S_{m,\ho}}\frac{1}{|S_{m,\ho}|}\bigg(f(x_t,\pi(x_t))-\frac{\pi(a_t|x_t)r_t(a_t)}{p_m(a_t|x_t)} - \frac{(1+\theta_{i,j})\gamma_m}{\alpha \gamma_{m+1,i}}(\hatf_{m}(x_t,\pi_{\hatf_{m}}(x_t))-\hatf_{m}(x_t,\pi(x_t)))\bigg)\\ 
             &L_2=\max_{\pi\in\Pi_j} \sum_{t\in S_{m,\ho}}\frac{1}{|S_{m,\ho}|}\bigg(\frac{\pi(a_t|x_t)r_t(a_t)}{p_m(a_t|x_t)}-f(x_t,\pi(x_t)) - \frac{(1+\theta_{i,j})\gamma_m}{\alpha \gamma_{m+1,i}}(\hatf_{m}(x_t,\pi_{\hatf_{m}}(x_t))-\hatf_{m}(x_t,\pi(x_t)))\bigg)\\ 
    \end{aligned}
\end{equation}
Clearly, both $L_1$ and $L_2$ are cost-sensitive classification problems \citep[see][for problem definition]{krishnamurthy2017active}.\footnote{In both, we need to find a policy (classifier) that maps contexts to arms (classes), incurring a score (cost) for each decision such that the total score (cost) is maximized (minimized).} Hence we propose an approach to implement \Cref{def:MTest}.

\subsection{General Estimation Rates}
\label{app:general-estimation-rates}

Recall that $\CVBandit$ uses estimation rates $\xi_i$ defined in \Cref{ass:model-selection}. Apart from \Cref{app:bounding-time-to-misspecification} (bounding time to detect misspecification), our analysis allows for more flexible rates and does not rely on \Cref{ass:gradual,ass:opt-value-jump}. Hence, $\CVBandit$ can be used with more general rates and settings. In particular, we weaken the need for $\xi_i$'s to share the same rate in $n$. We now describe the rates that allows for the rest of our analysis to go through (except \Cref{app:bounding-time-to-misspecification}).

These more general rates can be described by two fairly benign conditions. First, we require $\xi_i$ to be a non-increasing function of $n$. In particular, we require:\footnote{We require the first condition to ensure that $\gamma_{m,i}$ is non-decreasing in $m$.}
\begin{equation}
\label{eq:xi-first-condition}
    \begin{aligned}
        \text{For all $i\in[M']$ and $\zeta\in(0,1)$, } \text{$\xi_i(n, \zeta)$ is non-increasing in $n$.}
    \end{aligned}
\end{equation}
The second condition helps us simplify notation. At a high-level, it requires larger classes indices to correspond to more complex classes and have slower estimation rates.\footnote{We use the second condition to ensure that $\gamma_{m,j}/\gamma_{m,i}$ is greater than or equal to one and is non-decreasing in $m$ for $j\leq i$. We only require this condition to simplify notation and our results can easily be generalized.}
\begin{equation}
\label{eq:xi-second-condition}
    \begin{aligned}
        \text{For all $i\in[M']$ and $\zeta\in(0,1)$, } \frac{\xi_i(n,\zeta)}{\xi_{i-1}(n,\zeta)} \text{ is non-increasing in $n$ and is $\geq 1$,} %
    \end{aligned}
\end{equation}
where we define $\xi_0(n,\zeta):=\ln(1/\zeta)/n$, which is the estimation rate for estimating the mean of a one-dimensional bounded random variable.\footnote{In general, estimation rates are never faster than $\xi_0$. So, this is not a strong condition to have and helps simplify notation when stating guarantees for some misspecification tests.}

\end{document}